\documentclass{article} %
\usepackage{iclr2027_conference,times}

\usepackage{amsmath,amssymb,amsthm,mathtools}
\usepackage{tikz}
\usetikzlibrary{fadings,patterns,shadows.blur,shapes,shapes.geometric,arrows.meta,positioning,fit,calc}
\usepackage{array,booktabs,multirow,tabularx}
\usepackage{yhmath,mathdots,extarrows}
\usepackage{siunitx,gensymb}
\usepackage{physics,cancel,color,xcolor,xspace}
\usepackage{mdframed,thmtools}
\usepackage{pifont}
\usepackage[normalem]{ulem} %
\usepackage{bbold}          %
\usepackage{enumitem}
\usepackage{wrapfig}
\usepackage{algorithm}      %
\usepackage{algpseudocode}  %

\usepackage{microtype}      %
\usepackage{needspace}      %
\usepackage{etoolbox}       %
\usepackage{caption}        %

\allowdisplaybreaks[2]

\preto\section{\needspace{4\baselineskip}}
\preto\subsection{\needspace{3\baselineskip}}
\preto\subsubsection{\needspace{3\baselineskip}}
\preto\paragraph{\needspace{2\baselineskip}}

\usepackage{hyperref}
\hypersetup{
    colorlinks=true,
    linkcolor=blue,
    citecolor=teal,
    urlcolor=magenta,
    pdfauthor={Konstantinos Emmanouilidis, Lachlan MacDonald, Salma Tarmoun, Ren\'e Vidal},
    pdftitle={Convergence of Large-Stepsize SGD for the Cross-entropy Loss via Stochastic Lyapunov Stability},
    pdfborder={0 0 0}
}
\usepackage{url}
\usepackage[sort&compress,capitalize,nameinlink]{cleveref}		%

\definecolor{darkgreen}{RGB}{110,192,19}
\definecolor{shadecolor}{gray}{0.9}

\definecolor{thmblue}{RGB}{220,235,252}         %
\definecolor{thmblueborder}{RGB}{37,90,180}      %
\definecolor{lemmint}{RGB}{218,248,234}           %
\definecolor{lemmintborder}{RGB}{25,140,90}       %
\definecolor{propamber}{RGB}{255,244,215}         %
\definecolor{propamberborder}{RGB}{185,115,10}    %
\definecolor{assumlav}{RGB}{234,230,250}          %
\definecolor{assumlavborder}{RGB}{90,75,165}      %
\definecolor{examplebg}{RGB}{240,250,240}         %
\definecolor{exampleborder}{RGB}{60,160,60}       %

\newtheoremstyle{remark}
  {3pt}{3pt}{\itshape}{}{\bfseries}{.}{ }
  {Remark~\thmnumber{#2}}
\theoremstyle{remark}

\declaretheoremstyle[
  headfont=\normalfont\bfseries\color{thmblueborder},
  notefont=\normalfont\mdseries\color{thmblueborder},
  notebraces={(}{)},
  bodyfont=\normalfont,
  postheadspace=0.5em,
  spaceabove=6pt, spacebelow=4pt,
  mdframed={
    skipabove=6pt, skipbelow=6pt,
    linewidth=1.5pt,
    linecolor=thmblueborder,
    backgroundcolor=thmblue,
    innerleftmargin=8pt, innerrightmargin=8pt,
    innertopmargin=5pt, innerbottommargin=5pt,
    roundcorner=4pt}
]{thmbox}

\declaretheoremstyle[
  headfont=\normalfont\bfseries\color{lemmintborder},
  notefont=\normalfont\mdseries\color{lemmintborder},
  notebraces={(}{)},
  bodyfont=\normalfont,
  postheadspace=0.5em,
  spaceabove=6pt, spacebelow=4pt,
  mdframed={
    skipabove=6pt, skipbelow=6pt,
    linewidth=1.5pt,
    linecolor=lemmintborder,
    backgroundcolor=lemmint,
    innerleftmargin=8pt, innerrightmargin=8pt,
    innertopmargin=5pt, innerbottommargin=5pt,
    roundcorner=4pt}
]{lembox}

\declaretheoremstyle[
  headfont=\normalfont\bfseries\color{propamberborder},
  notefont=\normalfont\mdseries\color{propamberborder},
  notebraces={(}{)},
  bodyfont=\normalfont,
  postheadspace=0.5em,
  spaceabove=6pt, spacebelow=4pt,
  mdframed={
    skipabove=6pt, skipbelow=6pt,
    linewidth=1.5pt,
    linecolor=propamberborder,
    backgroundcolor=propamber,
    innerleftmargin=8pt, innerrightmargin=8pt,
    innertopmargin=5pt, innerbottommargin=5pt,
    roundcorner=4pt}
]{propbox}

\declaretheoremstyle[
  headfont=\normalfont\bfseries\color{assumlavborder},
  notefont=\normalfont\mdseries\color{assumlavborder},
  notebraces={(}{)},
  bodyfont=\normalfont,
  postheadspace=0.5em,
  spaceabove=6pt, spacebelow=4pt,
  mdframed={
    skipabove=6pt, skipbelow=6pt,
    linewidth=1.5pt,
    linecolor=assumlavborder,
    backgroundcolor=assumlav,
    innerleftmargin=8pt, innerrightmargin=8pt,
    innertopmargin=5pt, innerbottommargin=5pt,
    roundcorner=4pt}
]{assumbox}

\declaretheorem[style=assumbox,within=section]{definition}
\declaretheorem[style=thmbox,sibling=definition]{theorem}

\declaretheorem[style=lembox,sibling=definition]{assumption}

\declaretheorem[style=lembox,sibling=definition]{lemma}

\newcommand{\debug}[1]{#1}

\newcommand{\newmacro}[2]{\newcommand{#1}{\debug{#2}}}
\newcommand{\newop}[2]{\DeclareMathOperator{#1}{\debug{#2}}}

\newcommand{\R}{\mathbb{R}}
\newcommand{\Var}{\operatorname{\textrm{Var}}}

\newcommand{\F}{\mathcal{F}}

\newop{\ex}{\mathbb{E}}     %
\newop{\prob}{\mathbb{P}}   %
\newop{\simplex}{\hull}

\newmacro{\noise}{U}
\newmacro{\filter}{\mathcal{F}}
\newmacro{\sgda}{\text{SGDA}}
\newmacro{\seg}{\text{SEG}}
\newcommand{\Expep}[1]{\ex\left[#1 \Big| \filter_t\right]}

\newmacro{\point}{x}
\newmacro{\pointalt}{\alt\point}
\newmacro{\pointaltalt}{\altalt\point}
\newmacro{\points}{\mathcal{X}}
\newmacro{\intpoints}{\relint\points}
\newmacro{\base}{p}
\newmacro{\basealt}{q}
\newmacro{\basealtalt}{u}
\newmacro{\open}{\mathcal{U}}
\newmacro{\closed}{\mathcal{C}}
\newmacro{\cpt}{\mathcal{K}}
\newmacro{\nhd}{\mathcal{U}}
\newmacro{\gmat}{g}
\newmacro{\gdist}{\dist_{\gmat}}
\newmacro{\mfld}{M}
\newmacro{\form}{\omega}
\newmacro{\tvec}{z}
\newmacro{\uvec}{u}
\newmacro{\basin}{\mathbb{B}}
\newmacro{\ball}{\basin}
\newmacro{\sphere}{\mathbb{S}}

\newmacro{\tstart}{0}		%
\newmacro{\timealt}{s}		%
\newmacro{\horizon}{T}		%

\newmacro{\traj}{x}		%
\newmacro{\trajalt}{y}		%
\newmacro{\trajaltalt}{z}		%

\newmacro{\flowmap}{\Theta}		%
\DeclarePairedDelimiterXPP{\flowof}[2]{\flowmap_{#1}}{(}{)}{}{#2}		%

\newop{\Opt}{Opt}		%
\newop{\Sol}{Sol}		%
\newop{\gap}{Gap}	
\newop{\dualitygap}{Duality-Gap}		%
\newop{\orcl}{Or}		%

\newmacro{\tfun}{f}		%
\newmacro{\obj}{f}		%
\newmacro{\objalt}{g}		%
\newmacro{\sobj}{F}		%

\newmacro{\gvec}{g}		%
\newmacro{\oper}{A}		%
\newmacro{\vecfield}{v}		%

\newcommand{\sol}[1][\point]{#1^{\ast}}		%
\newmacro{\solvec}{\vecfield(\sol)}		%
\newmacro{\solpay}{\eq[\payv]}		%

\newmacro{\signal}{V}		%
\newmacro{\step}{\gamma}		%
\newmacro{\learn}{\eta}		%

\newmacro{\vbound}{G}		%
\newmacro{\lips}{\ell}		%
\newmacro{\strong}{\mu}		%
\newmacro{\smooth}{\beta}		%

\newop{\tspace}{T}		%
\newop{\tcone}{TC}		%
\newop{\dcone}{\tcone^{\ast}}		%
\newop{\ncone}{NC}		%
\newop{\pcone}{PC}		%
\newop{\hull}{\Delta}		%

\newmacro{\cvx}{\mathcal{C}}		%
\newmacro{\subd}{\partial}		%

\newmacro{\minmax}{\mathcal{L}}		%

\newmacro{\minvar}{{\point_{1}}}		%
\newmacro{\minvaralt}{\alt\minvar}		%
\newmacro{\minvars}{\points_{1}}		%
\newmacro{\minsol}{\sol[\minvar]}		%

\newmacro{\maxvar}{\point_{2}}		%
\newmacro{\maxvaaltr}{\alt\maxvar}		%
\newmacro{\maxvars}{\points_{2}}		%
\newmacro{\maxsol}{\sol[\maxvar]}		%

\newop{\Eucl}{\Pi}		%
\newop{\logit}{\Lambda}		%
\newop{\dkl}{KL}		%

\newmacro{\hreg}{h}		%
\newmacro{\hconj}{\hreg^{\ast}}		%
\newmacro{\breg}{D}		%
\newmacro{\mprox}{P}		%
\newmacro{\mirror}{Q}		%
\newmacro{\fench}{F}		%
\newmacro{\depth}{H}		%
\newmacro{\hstr}{K}		%
\newmacro{\hker}{\theta}		%

\newmacro{\proxdom}{\points_{\hreg}}		%
\newmacro{\proxdomi}{\points_{\hreg_{\play}}}		%
\newmacro{\zone}{\mathbb{D}}		%

\DeclarePairedDelimiterXPP{\proxof}[2]{\mprox_{#1}}{(}{)}{}{#2}		%

\newmacro{\state}{x}
\newmacro{\statealt}{y}
\newmacro{\statealtalt}{z}
\newmacro{\solb}{\mathrm{R}}
\newmacro{\growth}{L}
\newmacro{\wqsmscale}{\mu}
\newmacro{\wqsmshift}{\lambda}

\newmacro{\stepalt}{\alpha}

\newmacro{\rvar}{g}
\newmacro{\kyrt}{\delta_{\text{\tiny{KYRT}}}}
\newmacro{\energy}{\mathcal{E}}

\newmacro{\irr}{\psi}

\newmacro{\set}{C}
\newmacro{\tf}{\phi}

\newmacro{\normal}{\mathcal{N}}
\newmacro{\thres}{\theta}
\newmacro{\borel}{\mathcal{B}}

\DeclarePairedDelimiterX{\setdef}[2]{\lbrace}{\rbrace}{#1:#2}
\DeclarePairedDelimiterX{\inner}[2]{\langle}{\rangle}{#1,#2}
\DeclarePairedDelimiterXPP{\exclude}[1]{\mathopen{}\setminus}{\lbrace}{\rbrace}{}{#1}

\DeclarePairedDelimiterXPP{\exof}[1]{\ex}{\Big[}{\Big]}{}{%
   #1}
\DeclarePairedDelimiterXPP{\probof}[1]{\prob}{(}{)}{}{%
   #1}
\DeclarePairedDelimiterXPP{\oneof}[1]{\one}{\lbrace}{\rbrace}{}{%
   #1}

\newcommand{\Expe}[1]{\exof{#1}}

\DeclareMathOperator{\dist}{dist}

\DeclareMathOperator{\one}{\mathbb{1}} %

\DeclareMathOperator{\relint}{ri}

\usepackage{xcolor}

\makeatletter
\newcommand{\onlyappendixintoc}{%
  \let\oldaddcontentsline\addcontentsline
  \renewcommand{\addcontentsline}[3]{}%
}
\newcommand{\restoreaddcontentsline}{%
  \let\addcontentsline\oldaddcontentsline
}
\makeatother

\title{Convergence of Large-Stepsize SGD for the \\Cross-entropy Loss via Stochastic Lyapunov\\ Stability
}

\author{Konstantinos Emmanouilidis$^{1,4}$, Lachlan MacDonald$^{1,4}$, Salma Tarmoun$^{2,4}$, Ren\'e Vidal$^{3,4}$ \\
{\small $^1$Department of Computer and Information Science \quad $^2$Applied Mathematics and Computational Science} \\
{\small $^3$Department of Electrical and Systems Engineering \quad $^4$IDEAS Center} \\
{\small University of Pennsylvania}
}

\iclrfinalcopy %
\begin{document}

\maketitle
\lhead{Preprint}

\begin{abstract}
Modern deep learning has been shown to operate 
using learning rates far larger than those justified by classical optimization theory. Most prior analyses of this
phenomenon focus on deterministic gradient descent, leaving the stochastic setting largely unexplored. In this work, we provide sharp convergence guarantees for Stochastic Gradient Descent (SGD) applied to the multiclass cross-entropy loss, for both linear classifiers and two-layer neural networks. We show that the stochasticity of SGD may cause the dynamics to alternate between an \textit{edge-of-stability regime} that is dominated by curvature-driven oscillations, and a \textit{stable regime} in which the expected loss decreases at a controlled rate. Despite that, we prove that SGD stabilizes the dynamics, ensuring that the iterates return to stability in a fixed number of iterations and allowing convergence in the best-iterate sense even with large learning rates. Experiments validate our theoretical findings and illustrate the benefits of SGD in the large-stepsize regime.
\end{abstract}

\section{Introduction}

Modern machine learning models are trained at unprecedented scales and with stepsizes that lie far outside the regime covered by classical optimization theory. Yet, stochastic gradient descent (SGD) and its variants {\citep{bottou2012stochastic, kingma2017adammethodstochasticoptimization, gould2024continuoustimeanalysisadaptiveoptimization}} remain the basic workhorses of deep learning. A striking empirical phenomenon, reported repeatedly across architectures, datasets, and training paradigms, is that state-of-the-art performance is typically achieved by using large learning rates, often far exceeding the stability thresholds predicted by smooth convex optimization. This phenomenon, commonly referred to as the Edge of Stability (EoS) {\citep{cohen2021gradient}}, has attracted considerable attention in recent years and has emerged as a unifying perspective on the behavior of gradient methods in modern deep networks \citep{damian2023self, cohen2024understandingoptimizationdeeplearning}.

\begin{figure}[h]
    \centering
    \includegraphics[width=.85\linewidth]{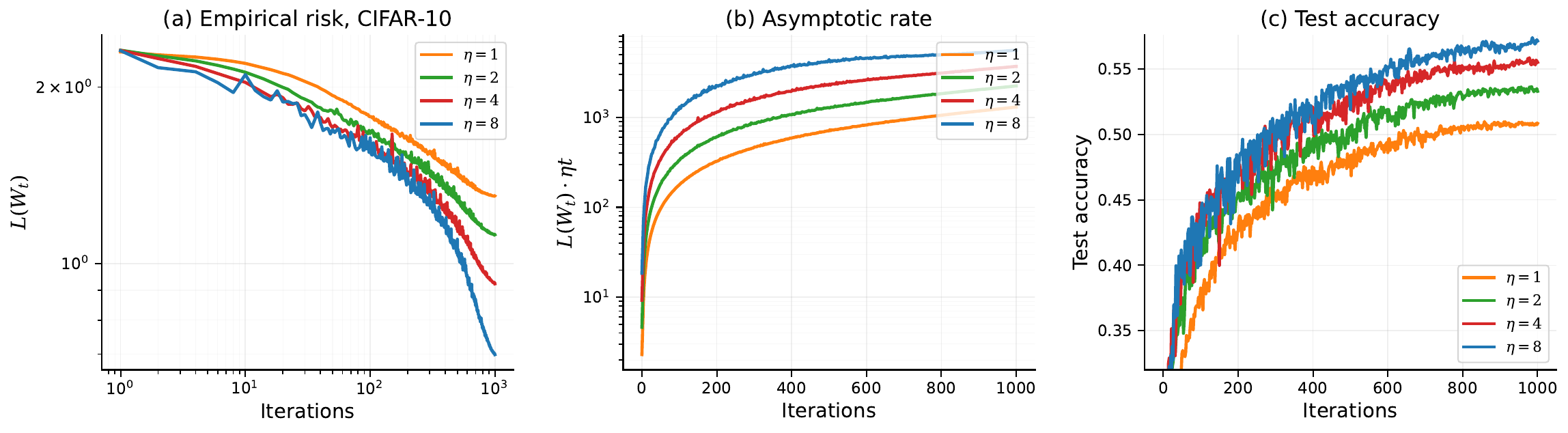}
    \caption{Experiment on CIFAR-10. An 8-layer NN with GELU activation function trained with SGD and large step sizes. The training loss decreases faster for large step sizes.}
    \label{fig: cifar-10}
\end{figure}

In the deterministic setting, the EoS phenomenon is now relatively well understood in linear logistic regression {\citep{wu2024large}}. For Gradient Descent (GD) applied to the logistic regression setup
\begin{eqnarray}
     L(W) = \frac{1}{n} \sum_{i=1}^n \ell(W x_i, y_i),
\end{eqnarray}
the classical theory dictates that convergence with a monotonic decrease of the loss requires a step size $\eta$ satisfying
$\lambda_{\max}(\nabla^2 L(W_t)) \leq \frac{2}{\eta},$
where  $\lambda_{\max}(\nabla^2 L(W_t))$ denotes the maximum eigenvalue of the Hessian of the loss at the current iterate $W_t$. However, prior empirical results have demonstrated that GD often converges even when this condition is violated, entering a regime characterized by non-monotonic decrease of the loss and oscillatory dynamics. Moreover, in the case of the square objective $\lambda_{\max}(\nabla^2 L(W_t))$ hovers around the critical value 
$\frac{2}{\eta}$ {\citep{cohen2021gradient, damian2023self}}.

The aforementioned behavior was analyzed in a sequence of works \citep{wu2024large, cai2024large, wu2025large} culminating in the analysis of \emph{large-stepsize GD} for logistic regression. %
A recurring conclusion is that the training dynamics can be divided in two regimes:
\begin{enumerate}[leftmargin=*]
    \item An early \emph{EoS phase}, where the iterates experience curvature-driven oscillations, the loss is not monotonically decreasing, and the Hessian’s top eigenvalue exceeds the stability boundary. 
    \item A late \emph{stable phase}, in which the algorithm self-stabilizes, the iterates enter a region of the landscape where 
$\lambda_{\max}(\nabla^2 L(W_t)) \leq \frac{2}{\eta}$, and the loss decreases at a predictable rate.%
\end{enumerate}\vspace{-0.2cm}

A key insight is that, although the dynamics during the EoS phase appear erratic, their structure is surprisingly rigid: the oscillations persist only for a finite number of iterations, after which the system transitions to the stable regime where standard descent analysis becomes applicable.

These advances have significantly sharpened our understanding of large-stepsize GD. Nevertheless, the deterministic viewpoint does not capture the behavior of the stochastic counterpart of the algorithm (SGD), that is commonly used in practice. Indeed, stochasticity introduces several new layers of complexity. First, the notion of a stability threshold becomes ill-defined. Prior work {\citep{wu2024large}} only provides theoretical guarantees for the average loss of SGD in logistic regression, without establishing what happens when the iterates become stable. Second, even if the trajectory reaches a region that would be stable under full-batch dynamics, a stochastic update may drive the iterates back outside this region. Thus, rather than a single transition from an EoS phase to a stable phase, SGD may exhibit a more intricate evolution in which stable and unstable regimes alternate before the dynamics eventually settle.
As a consequence, the deterministic EoS theory, based on precise control of the Hessian's largest eigenvalue does not extend in a straightforward or even meaningful way to the stochastic setting. This leaves a fundamental question open:

\begin{center}
    \textit{What is the appropriate notion of stability for SGD with the cross-entropy loss}\\
    \textit{and can we characterize the dynamics of SGD beyond the classical small learning rate regime?}
\end{center}

In this work, we aim to characterize the complex behaviour of the dynamics of SGD on the multi-class cross-entropy loss under arbitrarily large stepsizes and show that they exhibit three distinct regimes: an initial EoS regime, an intermediate regime in which stochastic fluctuations may repeatedly drive the trajectory in and out of stability, and a final with-high-probability stable regime. Our contributions can be described as follows:\vspace{-0.2cm}
\begin{enumerate}[leftmargin=*, label=$\star$]
    \item We propose a notion of stability for SGD in the cross-entropy loss setting, by leveraging the stochastic Lyapunov stability theory for dynamical systems. An important characteristic is that the stochastic stability criterion reduces to the deterministic one in the full-batch limit and captures the precise balance between curvature and noise that determines whether SGD remains stable.
    \item For the EoS setting, we provide convergence guarantees for SGD on the multi-class cross-entropy loss with arbitrarily large stepsizes $\eta > 0$, establishing the first theoretical guarantees for the multi-class setting.%
    \item For the intermediate regime, we prove that, despite the inherent stochasticity, SGD progressively stabilizes the trajectory, returning the dynamics back to stability after a fixed number of iterations and forcing the iterates to be stable with higher probability as the loss eventually decreases.
    \item For the stable regime, we provide a refined analysis of the dynamics, showing that SGD enters the stable regime in a fixed number of iterations with high probability and establishing a rate of $\Tilde{\mathcal{O}}\left(\frac{1}{t}\right)$ for the last iterate of convergence.
    \item We provide a fine-grained analysis for the two-layer neural network dynamics of SGD in the large learning rate setting, establishing distinct EoS and stable regimes and characterizing the behaviour of the dynamics in each regime.
    \item We provide extensive experimental validation of our theoretical results, showing the benefits of large step-sizes in practice.
\end{enumerate}

\vspace{-0.2cm}

\section{SGD for Cross-entropy Loss}

We consider the multi-class classification problem \eqref{eq:loss-def} with \(K\) classes, where \(x_i \in \mathbb{R}^d\) is the input, $y_i\in\{1,\ldots,K\}$ is the corresponding label, $f$ denotes the prediction function parameterized by weights \(W = [w_1,\ldots,w_K] \in \mathbb{R}^{d \times K}\), and $\ell\left(\cdot\right)$ is the cross-entropy loss function 
\begin{equation}
    L(W)
    = \frac{1}{n}\sum_{i=1}^n \ell\big(f(x_i;W), y_i\big),
    \qquad 
    \ell(f,y)
    = \log\!\bigg(\sum_{j=1}^K e^{f_j}\bigg)
      - f_y.
    \label{eq:loss-def}
    \vspace{-0.25cm}
\end{equation}

Our goal is to analyze the behavior of SGD in the large stepsize regime. For a mini-batch \(\mathcal{B}_t\subseteq[n]\) of size \(b \geq 1\), the SGD update at time \(t \geq 0\) is given by
\begin{equation}
    W_{t+1}
    = W_t - \eta\, g_t,
    \label{SGD} \tag{SGD}
\end{equation}
where $
    g_t
    = \frac{1}{b} \sum_{i\in\mathcal{B}_t}
    \nabla \ell(f(x_i;W_t), y_i)$ is the stochastic gradient oracle and $\eta>0$ is the stepsize. 

\subsection{Blanket Assumptions} 
We, next, introduce the assumptions needed for presenting our theoretical results. 
We assume without loss of generality that inputs satisfy $\|x_i\|_2 \le 1, \forall i \in [n]$. The next assumption regards the separability of the dataset, which is common in the literature \citep{wu2023implicit, cai2024large} for this setting.

\begin{assumption}[Separable data]\label{assumpt: separable}
There exists \(W_* \in \mathbb{R}^{d\times K}\) with \(\|W_*\|_F = 1\) and a margin \(\gamma>0\) such that for every \(i\in[n]\) and every class \(j\neq y_i\), it holds that
$(W_* x_i)_{y_i}
    \;-\;
    (W_* x_i)_j
    \ge
    \gamma.$
\end{assumption}\vspace{-0.2cm}

\subsection{Convergence Guarantees}

We are now in a position to present our main convergence guarantees for the cross-entropy loss setting. We emphasize that the dynamics of SGD with large stepsizes are inherently non-monotonic: the loss may increase due to stochasticity, while curvature-driven oscillations may pose another obstacle in achieving convergence. Our first main result shows that, despite the inherent impediments, SGD is able to converge even with \emph{arbitrarily large} stepsizes and in the presence of stochastic noise.

\begin{theorem}\label{thm:sgd_crossentropy}
Let Assumption~\ref{assumpt: separable} hold. The iterates of~\eqref{SGD} with any step size $\eta > 0$ and batch size $b \geq 1$ satisfy for any $t > 0$ that
\begin{eqnarray}
    \min_{0 \leq k \leq t-1} \Expe{L(W_k)} \leq\frac{K-1 + \ln^2(\gamma^2\eta t)+\eta^2 (1+\frac{1}{b})^2}{\gamma^2\eta t}.\label{eq:EoS-bound}
\end{eqnarray}
\end{theorem}

{In Theorem~\ref{thm:sgd_crossentropy} and the subsequent results for linear classifiers, we have assumed without loss of generality that $W_0 = \textbf{0}$. If the initialization is not zero, then a term depending on the norm of the initialization $W_0$ is added in the bounds.}

Theorem~\ref{thm:sgd_crossentropy} establishes an upper bound on the best iterate of SGD with any step size $\eta > 0$. The theorem allows us to compute the number of iterations after which the loss is below a specific value. Interestingly, the result indicates a $\Tilde{\mathcal{O}}\left(\frac{1}{t}\right)$ rate for the best iterate both for the case of small step sizes as well as even for the case of large step sizes. In the small learning rate regime, inequality \eqref{eq:EoS-bound} {matches, up to logarithmic factors,} the well known $\mathcal{O}\left(\frac{1}{t}\right)$ rate {of SGD}
\citep{nacson2019stochastic}, 
{here for the best} iterate of the algorithm. 

Our result serves as an extension from the binary logistic setting considered in {\citet{wu2024large}} to the multi-class cross-entropy loss, recovering the same rate of convergence in the case of $K=2$. We highlight that even though the rate in {\citet{wu2024large}} is established for the average loss, it can easily be algebraically converted to a best-iterate guarantee. Lastly, we note that our in expectation results can be converted to high probability results by leveraging Freedman's inequality. 

\subsection{Stability in Stochastic Optimization}\label{sec: stability_stoch_algo}
In the stochastic setting, non-monotonicity of the loss is no longer a sufficient signature of edge-of-stability behavior. Along the trajectory of SGD, the loss can increase for two different reasons. The first is simply the noise coming from the minibatch gradient: since SGD does not use the full gradient, the loss may occasionally increase even when the stepsize is small. The second reason is the use of a large stepsize: when the stepsize is too large compared to the local curvature of the loss, the iterates start to oscillate, which is the behavior typically associated with the edge of stability.

\begin{figure}[h]
    \centering
    \includegraphics[width=0.4\linewidth]{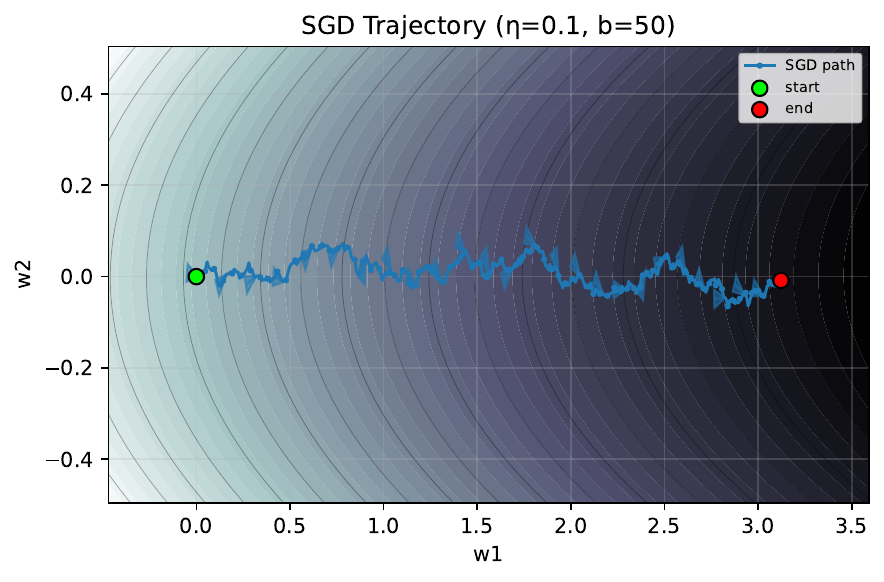}
    \includegraphics[width=0.4\linewidth]{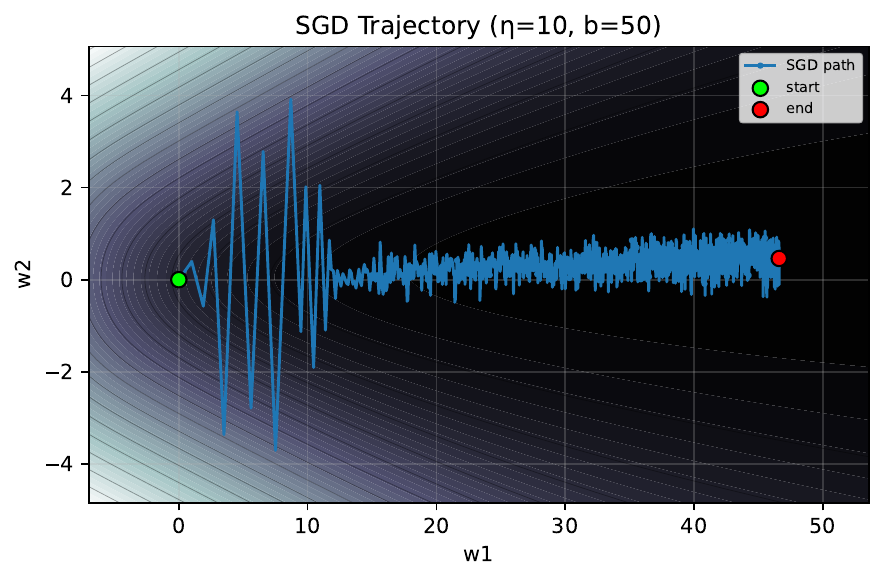}
    \caption{Two-dimensional trajectories of SGD with small and large stepsizes. Left: the trajectory of small stepsize SGD is affected by stochastic noise but does not oscillate. Right: in large stepsize SGD the stochastic noise is accompanied by genuine oscillations caused by the local curvature.}
    \label{fig:sgd-small-large-traj}
\end{figure}

This makes the definition of stability less direct than in the deterministic setting. One could try to use the largest eigenvalue of the full-batch Hessian, or a similar quantity for the minibatch Hessian. However, any such definition has to allow for the fact that SGD may temporarily leave any given stable region because of minibatch noise. Thus, the stability criterion should capture whether the loss decreases on average over the mini-batches, rather than whether it decreases at every single step.

For this reason, we use a stochastic Lyapunov viewpoint \citep{kushner1967stochastic, khasminskii2012stochastic, mao2007stochastic}. In deterministic optimization, a Lyapunov function $\mathcal{E}:\mathbb{R}^d \to \mathbb{R}_{\geq 0}$, which in our case is the loss, decreases along the trajectory. This is exactly the role played by the usual descent lemma. For SGD, such a pointwise decrease is too strong, since every update contains noise. The natural replacement is that the loss decreases in conditional expectation. Namely, for a Lyapunov function $\mathcal{E}:\mathbb{R}^d \to \mathbb{R}_{\geq 0}$, we say that the dynamics are Lyapunov stable whenever
\begin{eqnarray}
   \Expep{\mathcal{E}(W_{t+1}) - \mathcal{E}(W_t)} \leq 0 . \label{eq: stability_notion}
\end{eqnarray}

We refer the interested reader for a dedicated comparison of the proposed stability criterion with other notions of stochastic stability to Appendix~\ref{app: stochastic-stability}.

\subsection{The Stable Regime}

In this section, we provide a refined analysis of convergence by studying the stable regime of SGD. Using the Lyapunov stability viewpoint, we define the stable set as
\begin{eqnarray}
    \mathcal{S}
    =
    \left\{
    W \in \R^{d\times K}: L(W) \leq \tilde L,
    \tilde L = \frac{1}{\eta c}
    \right\},
    \nonumber
\end{eqnarray}
where $c = 2\max\left\{{8}\left(1+\frac{{ 2}n}{b\min\left\{\gamma^2, 1\right\}}\right), {\frac{n}{2\eta}}{, \frac{n}{\sqrt2}}\right\}$. %
This definition captures the region in which the \textit{expected loss} admits a one-step descent. In other words, inside $\mathcal{S}$ the stochastic updates may still fluctuate, but the loss decreases on average over the different mini-batches. The constant $c$ reflects the effect of minibatch noise through the batch size $b \geq1$, and in the full-batch case the criterion recovers the deterministic stability condition up to constants.

With this definition in place, we can describe the behavior of SGD more precisely. The next theorem shows that the iterates enter the stable set with high probability and, as long as they remain inside it, the expected loss of the last iterate decays at a rate $\mathcal{\tilde{O}}(1/t)$.

\begin{theorem}\label{thm: stable_regime}
Let Assumption~\ref{assumpt: separable} hold. The iterates of~\eqref{SGD} with any stepsize $\eta > 0$ and batch size $b \geq 1$ satisfy the following.
\begin{enumerate}[leftmargin=*]
    \item \textbf{Entrance to the stable regime.}
    For any $\delta \in (0,1)$, there exists $t_{in} \leq t_{max}(\delta)$ such that, with probability at least $1-\delta$, the SGD iterates enter the stable set $\mathcal{S}$ at time $t_{in}$.

    \item \textbf{Convergence inside the stable regime.}
    If the iterates $W_k \in \mathcal{S}, \forall k \in [t_{1}, t-1]$ for some $t_1 > 0$, then it holds that
    \begin{eqnarray}
     \Expe{L(W_{t})} &\leq& \frac{8 F(W_{t_{1}}) + 4\ln^2(\gamma^2\eta (t-{t_{1}}))}{7\gamma^2\eta (t-{t_{1}})},
    \label{eq:stable-bound}
    \end{eqnarray}
    where $F(W)
        =
        \frac{1}{n}
        \sum_{i=1}^n
        \sum_{j \ne y_i}
        e^{-\langle W(y_i)-W(j),x_i\rangle}.$
\end{enumerate}
\end{theorem}

Theorem~\ref{thm: stable_regime} gives a last-iterate guarantee for SGD after the dynamics have entered the stable regime. In the deterministic setting, GD first passes through an EoS phase and then enters a stable phase after some deterministic time bound $t_{in}^{GD}$. For SGD, however, the entrance time is necessarily probabilistic, because the trajectory depends on the sampled minibatches. Thus, the theorem shows that even in the presence of stochastic noise, the dynamics reach a region where the expected loss decreases at the standard $\mathcal{\tilde{O}}(1/t)$ rate.

However, there is an important difference from GD. In the deterministic case, once the iterates enter the stable regime, the dynamics remain there by construction. For SGD, this is no longer automatic: minibatch noise can temporarily push the iterates outside $\mathcal{S}$. This is the main additional difficulty in the stochastic setting and motivates the stochastic stabilization analysis developed in the next section.

\subsection{Stabilization of SGD}
\begin{wrapfigure}{r}{0.38\textwidth}
    \centering
    \vspace{-1em}
    \includegraphics[width=0.7\linewidth]{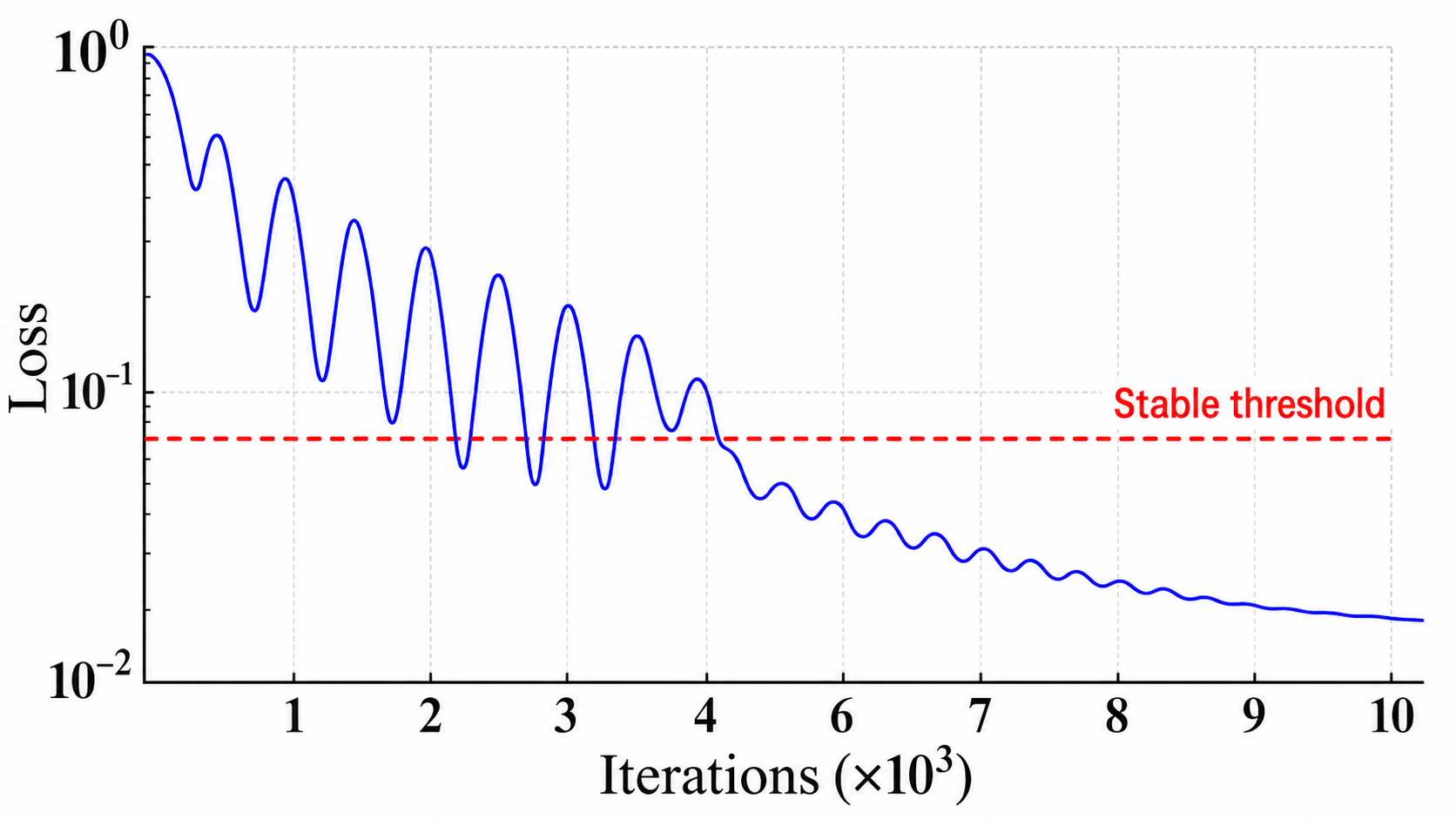}
    \vspace{-1em}
    \caption{Stabilization property of SGD.}
    \label{fig:self-stabilization}
    \vspace{-2em}
\end{wrapfigure}
We now show that SGD has an additional stabilizing property that is absent from the deterministic setup. Although mini-batch noise may occasionally push the iterates outside the stable set, this instability cannot persist indefinitely. Instead, the dynamics are pulled back toward the stable regime after a controlled number of iterations. In this sense, SGD is \emph{progressively stabilizing} the trajectory:
\begin{center}
    \textit{even if the iterates escape the stable regime,} \\
    \textit{they return to stability in a fixed number of iterations.}
\end{center}

This property is the key reason why SGD can converge despite the noise introduced by minibatch gradients. The iterates may not remain in the stable set forever, but they revisit it often enough for the loss to keep decreasing over time. We make this statement precise in the following theorem.

\begin{theorem}[Stabilization of SGD] \label{thm: self_stab}
    Let Assumption~\ref{assumpt: separable} hold. The iterates of~\eqref{SGD} with any stepsize $\eta > 0$ and batch size $b \geq 1$ satisfy the following.
    \begin{enumerate}[leftmargin=*]
        \item \textbf{Stability is maintained with high probability.}
        If $W_t \in \mathcal{S}$, then with probability at least $1-\delta'$ it holds that $W_{t+1} \in \mathcal{S}$, where $\delta'
        =
        \mathcal{O}\left(\exp\!\left(
        -\frac{b}{\eta L(W_t)^{3/2}}
        \right)\right).$

        \item \textbf{Return to stability.}
        Assume that the dynamics exit the stable regime at $t_{out}>0$, and fix $\delta\in(0,1)$. Then, with probability at least $1-\delta$, the iterates of~\eqref{SGD} return to the stable regime in at most {$t_{\mathrm{re}} = \Tilde{\mathcal{O}}\left(\frac{K + \eta^2\left(1+\frac{1}{b}\right)^2}{\gamma^2 \eta \delta \tilde{L}} \right)$} number of steps{, where $\Tilde{\mathcal{O}}$ hides logarithmic factors in $t_{out}$, $\gamma$, $\eta$, $\delta$ and $\tilde L$}.
    \end{enumerate}
\end{theorem}

Theorem~\ref{thm: self_stab} shows that the stable regime is robust under the stochastic perturbations of SGD. The first part states that, if the current iterate is already stable, then the next iterate remains stable with high probability. This probability improves as the loss becomes smaller and as the iterate moves farther away from the stability boundary. In particular, as $L(W_t)\to 0$, the probability of staying inside $\mathcal{S}$ tends to one.

The second part shows that even if SGD leaves the stable set, it returns after a bounded number of iterations with high probability. Thus, the stochastic noise may create short excursions outside $\mathcal{S}$, but it does not permanently destabilize the dynamics. The resulting picture is a repeated cycle: SGD enters the stable set, may occasionally exit because of minibatch noise, and then returns again. As training progresses and the loss decreases, the exits become less likely, and the dynamics spend more and more time in the stable regime.

This stochastic stabilization mechanism is one of the main differences between SGD and GD with large learning rates. In GD, once the iterates enter the stable phase, there is no stochastic noise that can push them out. In SGD, temporary exits are possible, but the geometry of the cross-entropy loss pulls the dynamics back inside the stable set. We emphasize that this mechanism is different from the self-stabilization phenomenon for the square loss studied by \citet{damian2023self}, where stabilization comes from a higher-order Taylor term in the deterministic GD dynamics. Here, stabilization is instead driven by the self-bounded structure of the cross-entropy loss, which makes both curvature and stochastic noise decrease as the loss becomes small.

\section{Edge of Stability for Two-Layer Neural Networks}

An important question is whether the previous theoretical guarantees and insights extend to the dynamics of SGD when applied in neural networks. While the curvature of deep neural networks is more intricate, many architectures exhibit structural properties, such as near-homogeneity or controlled activation derivatives, that allow for a thorough analysis to be proceeded. 

We consider two-layer neural networks in mean-field scaling. The predictor is given by
\begin{equation}
   \label{eq: 2nn}
    f(W; x)
    :=
    \frac{1}{m}\sum_{j=1}^m a_j \phi(x^\top W^{(j)}),
    \quad
    W^{(j)} \in \R^{d\times K},
    \quad
    j=1,\dots,m,
\end{equation}
where the coefficients $a_j \in \{\pm 1\}$ are fixed and
$W = (W^{(j)})_{j=1}^m \in \R^{md\times K}$ are the trainable parameters.
We train the network with SGD using the rescaled stepsize $\tilde\eta := m \eta$, namely
\begin{eqnarray}
    W_{t+1} = W_t - \tilde\eta g_t .
    \nonumber
\end{eqnarray}
Throughout the section, we assume $\sup_i \|x_i\|\leq 1$. This is only for simplicity of presentation and can always be enforced by normalizing the data.

We impose the following conditions on the activation function.

\begin{assumption}
[Activation function conditions]
\label{assump: activation}
Let $\phi:\R \to \R$ be continuously differentiable and satisfy the following
\begin{enumerate}[leftmargin=*]
    \item \label{assump: activation:grad}
    \textbf{Derivative condition.}
    There exists $\alpha \in \left(\frac{1}{1+\gamma},1\right)$ such that
    $\alpha \leq \phi'(z) \leq 1,$
    where $\gamma>0$ is the margin from Assumption~\ref{assumpt: separable}.

    \item \label{assump:activation:smooth}
    \textbf{Smoothness.}
    There exists $\tilde\beta>0$ such that, for all $x,y\in\R$,
    $|\phi'(x)-\phi'(y)| \leq \tilde\beta |x-y|$.

    \item \label{assump: activation:near-homogeneous}
    \textbf{Near-homogeneity.}
    There exists $\kappa>0$ such that, for every $z\in\R$,
    $|\phi(z)-\phi'(z)z| \leq \kappa$.
\end{enumerate}
\end{assumption}

Assumption~\ref{assump: activation} is satisfied by several standard activations after a mild leaky modification. The smoothness and near-homogeneity conditions are standard in analyses of optimization for two-layer networks, and are also used in prior work on large-stepsize gradient descent~\citep{cai2024large}. The derivative condition ensures that every neuron continues to receive a non-trivial gradient during training; in particular, it rules out the degenerate case where some neurons become effectively frozen.

The lower bound $\alpha > 1/(1+\gamma)$ is the multiclass analogue of the positive derivative lower bound used in the binary setting by \citet{cai2024large}. In binary classification, any positive lower bound on $\phi'$ is enough for the usual perceptron-type argument, because each sample contributes only one signed margin inequality. In the multiclass case, each sample must be separated from all $K-1$ incorrect classes, and the activation derivative may be evaluated at different values for different classes. This creates an additional error term of size $1-\alpha$, which is controlled precisely when
$    \alpha \gamma > 1-\alpha,
$ or equivalently $\alpha>1/(1+\gamma)$.

This leads to the effective margin
\begin{equation}
  \tilde\gamma
  :=
  \alpha\gamma - (1-\alpha)
  =
  \alpha(1+\gamma)-1.
  \label{eq: effective-margin}
\end{equation}
The condition $\alpha>1/(1+\gamma)$ is exactly what makes $\tilde\gamma$ positive. Thus, $\tilde\gamma$ plays the role of the margin in the two-layer network analysis. It can be interpreted as the original margin $\gamma$, discounted by the smallest possible activation derivative, and corrected by the variation of the activation derivatives across classes. In the limit $\alpha\to 1$, this correction disappears and $\tilde\gamma\to\gamma$, recovering the usual margin for linear logistic regression.%

The next lemma shows that this assumption is not restrictive: leaky versions of several common activations satisfy Assumption~\ref{assump: activation}.

\begin{lemma}[Leaky activation functions satisfying Assumption~\ref{assump: activation}]
\label{lem:leaky-activations}
For a margin $\gamma > 0$, let $c$ satisfy $1/(1+\gamma)<c<1$.
The following leaky variants of common activation functions satisfy Assumption~\ref{assump: activation}.
\begin{itemize}[leftmargin=*]
\item \textbf{GELU, Softplus, SiLU.}
Let $\phi$ be GELU, Softplus, or SiLU, and define
{
$\widetilde\phi(x)
    :=
    \frac{1+2c}{3}x+\frac{1-c}{6}\phi(x)$
}
Then $\widetilde\phi$ satisfies
Assumption~\ref{assump: activation} with
$\alpha=c$, $\widetilde\beta=4(1-c)$, and $\kappa=1-c$.

\item \textbf{Huberized ReLU.}
Let $\phi$ be the Huberized $\mathrm{ReLU}_h$. Then,
$\tilde\phi(x) := c x + (1-c)\phi(x)$
satisfies Assumption~\ref{assump: activation} with
$\alpha=c$, $\tilde\beta=(1-c)/h$, and $\kappa=(1-c)h/2$.

\item \textbf{Tanh, Sigmoid.}
The leaky tanh
$\tilde\phi(x) := c x + (1-c)\tanh(x)$
and the leaky sigmoid
$\tilde\phi(x) := c x + \frac{1-c}{1+e^{-x}}$
both satisfy Assumption~\ref{assump: activation} with
$\alpha=c$, $\tilde\beta=1-c$, and $\kappa=1-c$.
\end{itemize}
\end{lemma}

Having introduced the assumptions on the activation function, we now state the main convergence result for two-layer neural networks. The theorem shows that the large-stepsize convergence guarantees obtained in the linear case continue to hold despite the nonlinearity of the model.

\begin{theorem}\label{thm: sgd_2_layer_nn}
Let Assumptions~\ref{assumpt: separable} and~\ref{assump: activation} hold. The iterates of~\eqref{SGD} for the two-layer neural network in~\eqref{eq: 2nn}, with any stepsize $\tilde\eta := m\eta${, batch size $b\geq 1$ and any initialization $W_0 \in \R^{md\times K}$,} satisfy
\begin{eqnarray}
    \min_{0\leq t \leq T-1} \Expe{L(W_t)}
    &\leq&
    \frac{
    K - 1
    +
    {4}\ln^{2}(\tilde\gamma^{2}\eta T)
    +
    {16}\kappa^{2}
    +
    {2}\eta^{2}\left(1 + \frac{1}{b}\right)^{2}
    {+
    \frac{\tilde\gamma^{2}}{m}\|W_0\|_F^2}
    }{
    \tilde\gamma^{2}\eta T
    },
    \label{eq:EoS-bound_2layer}
\end{eqnarray}
where $\tilde\gamma = \alpha\gamma - (1-\alpha)$. %
\end{theorem}

Theorem~\ref{thm: sgd_2_layer_nn} shows that SGD converges for any stepsize $\eta>0$ for two-layer neural networks. More precisely, the best iterate achieves a $\widetilde{\mathcal{O}}(1/T)$ convergence rate. This guarantee holds without requiring the trajectory to be in the stable regime, as it remains valid even during the edge-of-stability phase, where the loss may be non-monotone and the iterates may oscillate. {Moreover, the guarantee holds for an arbitrary initialization, which only enters through the additive term $\|W_0\|_F^2/m$: in particular, it covers standard random initializations with $\|W_0\|_F^2 = \mathcal O(m)$, which break the symmetry between neurons. }

To the best of our knowledge, this is the first convergence guarantee for SGD on two-layer neural networks trained with the cross-entropy loss in the large-stepsize regime. The closest related result is the deterministic GD analysis of~\citet{cai2024large}, which studies two-layer networks under the binary logistic loss. Our result recovers the same asymptotic rate up to constants, while allowing stochastic gradients and multiclass cross-entropy loss.

\subsection{The Stable Regime}
We, next, provide a fine-grained analysis of the dynamics by introducing a stable regime for the iterates similarly to the cross-entropy loss case. 
As in the cross-entropy loss, the appropriate notion of stability is not based on pointwise monotonicity of the realized loss, since the SGD updates remain stochastic. Instead, we use the Lyapunov viewpoint and define the stable regime as the region where the expected loss admits a one-step descent.

For the two-layer network, this leads to the stable set (refer to Appendix~\ref{app:twolayer:stable} for the exact derivation)
\begin{eqnarray}
    \mathcal{S}_{NN}
    &=&
    \left\{
    W_t\in\R^{md \times K}: L(W_t) \leq \tilde{L}_{NN}
    \right\},
    \nonumber \\
    \widetilde L_{\mathrm{NN}}
    &=&
    \min\left\{
    \frac{1}{
    {16}\eta(1+\widetilde\beta)
    \left(
    1+\frac{2n}{
    b\min\{\widetilde\gamma^2,1\}}
    \right)},
    \frac{1}{\sqrt{2}n\eta},
    \frac{1}{2ne^{\kappa+2}}
    \right\}.
    \nonumber
\end{eqnarray}

We now show that SGD enters this stable regime with high probability and that, once inside it, the last iterate enjoys the same sublinear convergence rate.

\begin{theorem}\label{thm: sgd_2_layer_nn_stable}
Let Assumptions~\ref{assumpt: separable} and~\ref{assump: activation} hold. The iterates of~\eqref{SGD} for the two-layer neural network in~\eqref{eq: 2nn}, with any stepsize $\tilde\eta := m\eta$ and batch size {$b\geq 1$}, satisfy the following.
\begin{enumerate}[leftmargin=*]
    \item \textbf{Entrance to the stable regime.}
    For any $\delta\in(0,1)$, there exists $t_{in}\leq t_{max}(\delta)$ such that, with probability at least $1-\delta$, the SGD iterates enter the stable set $\mathcal{S}_{NN}$.

    \item \textbf{Convergence inside the stable regime.}
    If the iterates of \ref{SGD} satisfy $W_k \in \mathcal S_{NN}, \forall k \in [t_{1}, t-1]$, for some $t_1 > 0$, it holds that
    \begin{eqnarray}
     {\mathbb E[L(W_t)] \;\le\; \frac{8}{\tilde\gamma^{2}\, \eta\,(t - t_{1})}}
    \end{eqnarray}
\end{enumerate}
\end{theorem}

Theorem~\ref{thm: sgd_2_layer_nn_stable} shows that, after entering the stable regime, the last iterate of SGD satisfies a {$\mathcal{O}(1/t)$} convergence rate. This is the nonlinear analogue of the stable-regime guarantee obtained for the linear cross-entropy model. It also matches the rate known for deterministic GD with logistic regression in two-layer networks~\citep{cai2024large}, while allowing for stochastic gradients and multiclass cross-entropy loss.

\subsection{Stabilization in Two-layer Neural Networks}
Interestingly, the stabilization mechanism persists beyond the linear cross-entropy loss setting. Even in the nonlinear two-layer network setting, SGD has the same qualitative behavior: as the loss decreases, the iterates become increasingly likely to remain in the stable set $\mathcal{S}_{NN}$, and even if mini-batch noise temporarily forces them outside, the dynamics return to stability after a controlled number of iterations with high probability.

\begin{theorem}[Stabilization of SGD] \label{thm: self_stab_2_layer_nn}
    Let Assumptions~\ref{assumpt: separable} and~\ref{assump: activation} hold. The iterates of~\eqref{SGD} with any stepsize $\tilde\eta := m\eta$ and batch size $b \geq 1$ satisfy the following.
    \begin{enumerate}[leftmargin=*]
        \item \textbf{Maintaining stability.}
        \label{prob_remaining_2_layer}
        If $W_t \in \mathcal{S}_{NN}$, then with probability at least $1-\delta'$ it holds that $W_{t+1} \in \mathcal{S}_{NN}$, where $\delta'
        =
        \mathcal{O}\left(\exp\!\left(
        -\frac{b}{\eta L(W_t)^{3/2}}
        \right)\right).$

        \item \textbf{Return to stability.}
        \label{return_steps_2_layer}
        If the dynamics exit the stable regime $\mathcal{S}_{NN}$ at time $t_{out}>0$, then for any $\delta\in(0,1)$, with probability at least $1-\delta$, they return to $\mathcal{S}_{NN}$ in at most {$t_{\mathrm{re}} = \mathcal{\Tilde{O}}\left(\frac{K + \kappa^2 + \eta^2\left(1+\frac{1}{b}\right)^2}{\tilde{\gamma}^2 \eta \delta \tilde{L}_{NN}} \right)$} number of steps{, where $\Tilde{\mathcal{O}}$ hides logarithmic factors in $t_{out}$, $\tilde\gamma$, $\eta$, $\delta$ and $\tilde L_{NN}$}.
    \end{enumerate}
\end{theorem}

\vspace{-0.25cm}
\section{Experiments}
\label{sec:experiments}

We now complement our theoretical results with a series of experiments illustrating the behavior of SGD in the large step-size regime. 
We conduct three set of experiments to validate our theoretical results. In the first set of experiments, we run SGD with a linear model on a separable synthetic dataset and plot the loss for different step sizes. Each class of points contains $N = 50$ points from $\mathcal{N}(\mu_i, \sigma^2 I),$ where $\mu_i = (\frac{\pi i}{3}, 1), i \in [2], \sigma = 0.05$ are the parameters of the $i-$th class. 

In Figure~\ref{fig: cross-entropy-loss}, we observe that the loss incurs significant spikes at the start and non-monotonic decrease, indicating the distinct unstable and stable regimes. In addition, we validate the sublinear rate of decrease in the stable regime, validating the results of Theorem~\ref{thm: stable_regime}. Lastly, we observe that for larger batch sizes the loss incurs less noise and less oscillations. 
\vspace{-0.25cm}
\begin{figure}[h!]
    \centering
    \includegraphics[width=0.4\linewidth]{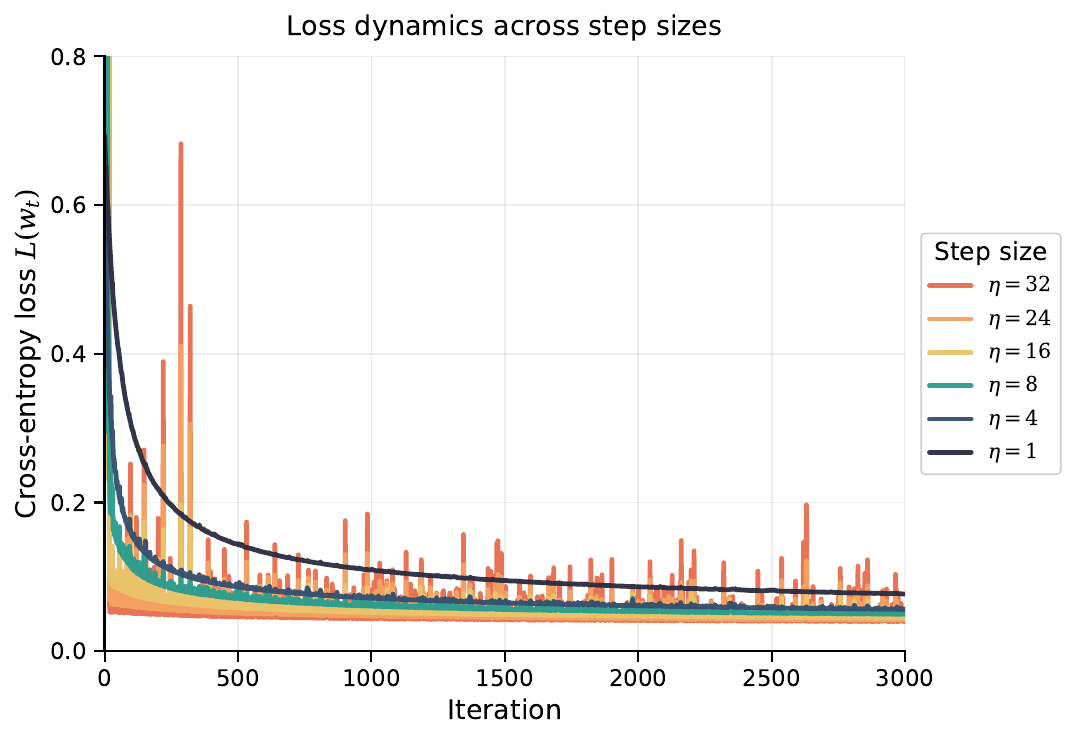}
    \includegraphics[width=0.45\linewidth]{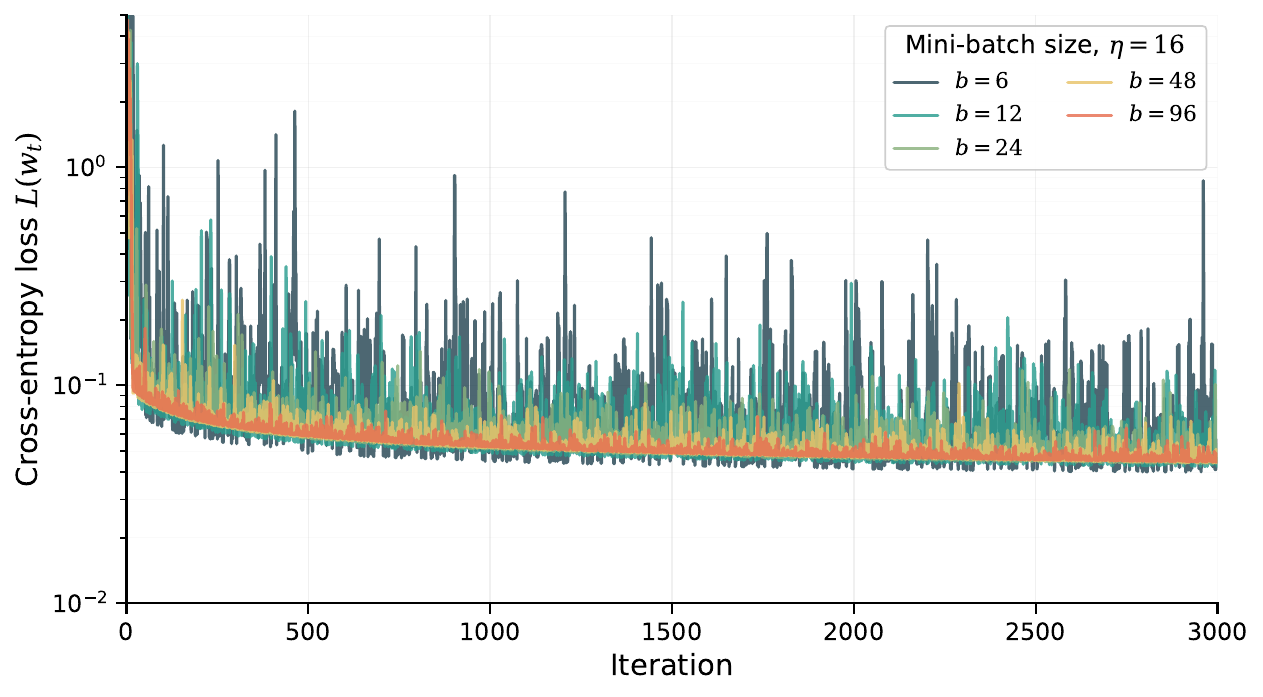}
    \caption{Cross-entropy loss with different stepsizes and batch sizes.}
    \label{fig: cross-entropy-loss}%
\end{figure}

The second set of experiments focuses on validating our theory for Neural Networks (NN). In particular, we run SGD on a two-layer NN with leaky softplus activation on the MNIST dataset. In Figure~\ref{fig: mnist}, we observe that the loss spikes at the initial unstable regime and then decreases non-monotonically. We, also, verify that the loss converges at an $\mathcal{O}\left(\frac{1}{t}\right)$ rate in the stable regime. Lastly, we observe that larger stepsizes provide faster convergence and, if not better, similar test accuracy. 
\begin{figure}[ht]
    \centering
    \includegraphics[width=\linewidth]{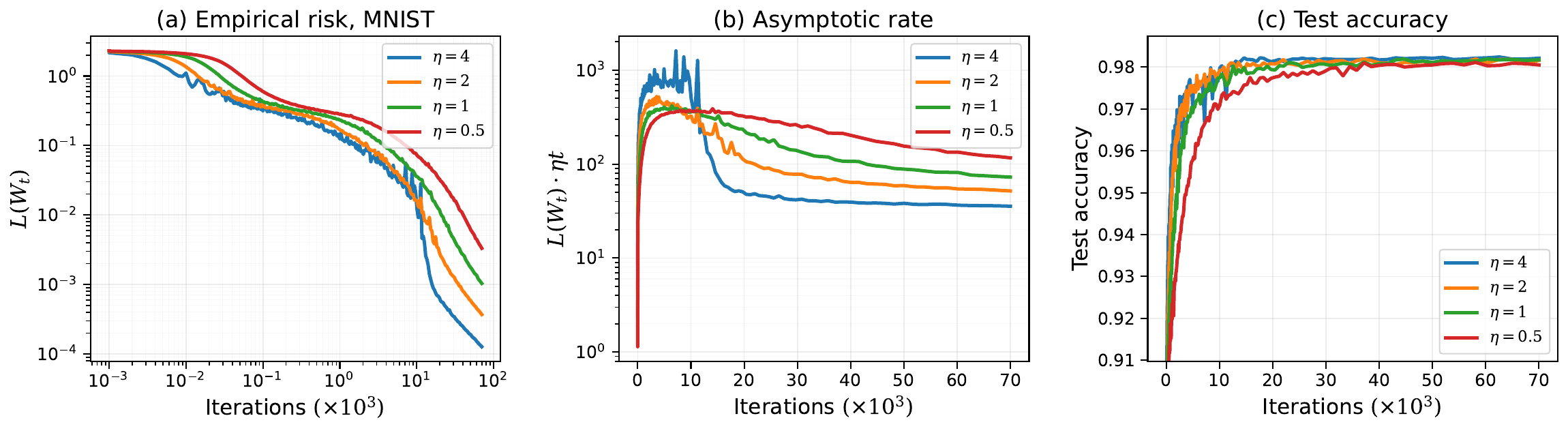}
    \caption{Experiment on MNIST dataset. The dynamics of SGD for training a two-layer NN with leaky softplus activation.}
    \label{fig: mnist}
\end{figure}

The third set of experiments focuses on the CIFAR-10 dataset, where we test whether our results hold even for larger networks than the two-layer NNs considered in the theoretical results. More specifically, we run SGD on an 8-layer NN with GELU activation for different stepsizes. In figure~\ref{fig: cifar-10}, we observe that the loss decreases non-monotonically at a $\mathcal{O}\left(\frac{1}{t}\right)$ rate. Large stepsizes provide faster decrease of the loss and seem to be synergetic in achieving higher test accuracy. 
We have not tried to test exhaustively the effect of the stepsize on the test accuracy, as this remains beyond the scope of this paper.

\section{Conclusion}

In this work, we developed a theoretical framework for analyzing SGD for the cross-entropy loss with large learning rates. Unlike deterministic GD, non-monotonicity of the loss in SGD is caused either due to mini-batch noise or due to curvature-driven oscillations, and therefore stability must be defined through a stochastic Lyapunov criterion. Using this viewpoint, we proved convergence guarantees for multiclass cross-entropy loss with arbitrary large step sizes and showed that the dynamics enter a stable regime where the last iterate converges at a sublinear rate. We, further, established the stabilization mechanism of SGD: even when stochastic noise temporarily forces the iterates outside the stable set, the dynamics return to stability in a controlled number of iterations. Finally, we extended this picture to two-layer neural networks and validated the established results experimentally, showing the benefits of SGD with large learning rates.

\newpage
\subsection*{AI use statement}

In this work, we used generative AI tools for helping in editing the manuscript, brainstorming, as well as searching for related literature. Any mathematical claim, in which generative AI has assisted, has been carefully inspected by the authors. We have reviewed all AI-assisted work and take responsibility for the final content of this work,
including text, claims or artifacts produced with the aid of generative AI.

\bibliographystyle{plainnat}
\bibliography{ref}
\newpage
\appendix

\section*{Supplemental Material}
\addcontentsline{toc}{section}{Appendix}
\tableofcontents
\newpage

\section{Stochastic Lyapunov Stability and the Stable Regime}
\label{app: stochastic-stability}

We provide additional context for the stability notion used in the main paper. 
The classical stability criterion for deterministic GD is based on monotonic decrease of a Lyapunov function, typically the objective itself. 
Indeed, if \(W_{t+1}=W_t-\eta \nabla L(W_t)\), then the standard descent lemma implies
\[
    L(W_{t+1})-L(W_t)
    \leq
    -\eta \|\nabla L(W_t)\|^2
    + \frac{\eta^2}{2}\lambda_{\max}(\nabla^2 L(W_t))\|\nabla L(W_t)\|^2 .
\]
Thus, a sufficient condition for one-step descent is
\[
    \eta \lambda_{\max}(\nabla^2 L(W_t)) \leq 2.
\]
This is the usual deterministic stability threshold. In deterministic edge-of-stability analyses, the violation of this inequality is directly related to curvature-driven oscillations and non-monotone behavior of the loss \citep{cohen2021gradient,damian2023self,wu2024large}. 

For SGD, however, pointwise monotonicity of the realized loss is too strong and is not the right stability requirement. The update
\[
    W_{t+1}=W_t-\eta g_t,
    \qquad
    \mathbb E[g_t\mid \mathcal F_t]=\nabla L(W_t),
\]
contains mini-batch noise. Therefore, even when the stepsize is small and the curvature is benign, the realized loss \(L(W_{t+1})\) may be larger than \(L(W_t)\) for a particular mini-batch. Consequently, non-monotonicity alone does not distinguish between two qualitatively different phenomena: harmless stochastic fluctuations and genuine curvature-driven instability. A stability notion for SGD must therefore average over the randomness of the next mini-batch while keeping the current iterate fixed.

The natural replacement for deterministic Lyapunov decrease is stochastic Lyapunov decrease. Given the filtration \(\mathcal F_t\) generated by the trajectory up to time \(t \geq 0\), a nonnegative function \(\mathcal{E}\) is a stochastic Lyapunov function if it satisfies a one-step drift condition of the form
\[
    \mathbb E\!\left[\mathcal{E}(W_{t+1})-\mathcal{E}(W_t)\mid \mathcal F_t\right]\leq 0.
\]
This conditional drift condition is the discrete-time analogue of the Lyapunov criteria used in stochastic stability theory \citep{kushner1967stochastic,khasminskii2012stochastic,mao2007stochastic}. Related Foster--Lyapunov drift conditions are also standard in the stability and recurrence theory of Markov chains \citep{meyn1993markov}. The condition says that \(\mathcal{E}\) is not required to decrease for every realization of the stochastic update, but it must decrease on average over the algorithmic randomness at the current point. In our setting, the most appropriate Lyapunov function is the cross-entropy loss itself, namely \(\mathcal{E}(W)=L(W)\). This choice is natural for two reasons. First, the goal of the dynamics is to drive the training loss to zero. Second, in the separable cross-entropy setting the parameters need not converge to a finite minimizer; instead, their norm may diverge while the loss goes to zero \citep{soudry2018implicit}. Hence, stability notions based on distance to a fixed equilibrium point are \textit{not} well aligned with the geometry of the problem.

This Lyapunov stability viewpoint motivates, as discussed in Section~\ref{sec: stability_stoch_algo}, the definition of the stable set
\[
    \mathcal S
    :=
    \left\{
        W\in \mathbb R^{d\times K}:
        L(W)\leq \frac{1}{\eta c}
    \right\}.
\]
Inside \(\mathcal S\), SGD may still exhibit random upward fluctuations of the realized loss, but the loss decreases in conditional expectation. This is precisely the stochastic analogue of deterministic Lyapunov stability.

It is useful to contrast this criterion with other standard notions of stochastic stability. 
Stability in probability requires that trajectories starting near an equilibrium remain near it with high probability; this is the classical stochastic analogue of Lyapunov stability and is commonly formulated for stochastic differential equations and Markov processes \citep{khasminskii2012stochastic,kushner1967stochastic}. 
Almost-sure stability strengthens this requirement by demanding pathwise stability with probability one, namely that the trajectory remains close to the equilibrium for almost every realization of the noise \citep{mao2007stochastic,khasminskii2012stochastic}. 
Moment stability requires decay, or at least uniform control, of quantities such as \(\mathbb E\|W_t-W_\star\|^p\), with the cases \(p=1\) and \(p=2\) corresponding to mean and mean-square stability, respectively \citep{mao2007stochastic,higham2002mean}. 
Exponential moment stability further asks for such moments to decay at an exponential rate, and is often used in the numerical analysis of stochastic differential equations \citep{higham2002mean}. 
Finally, Foster--Lyapunov or positive-recurrence criteria require a negative drift outside a suitable compact set and are used to prove recurrence, ergodicity, or existence of invariant distributions for Markov chains and Markov processes \citep{meyn1993markov,meyn1993stability}. 
These criteria are powerful in settings with a finite attracting equilibrium, a stationary distribution, or a recurrent compact set. 
However, they are less suited to separable cross-entropy minimization, where the relevant limiting behavior is not convergence of \(W_t\) to a finite point, but rather convergence of \(L(W_t)\) to zero while the parameters may diverge in norm.

For this reason, the conditional Lyapunov-drift criterion is the correct notion for our purposes. It is a pointwise criterion, since it evaluates stability at the current iterate; it is genuinely stochastic, since it takes expectation with respect to the next mini-batch; and it is expressed in terms of the loss, which is the relevant convergent quantity in separable cross-entropy minimization. 
Furthermore, when the mini-batch coincides with the full dataset, the stochastic gradient \(g_t\) reduces to the full gradient \(\nabla L(W_t)\), and the conditional drift condition recovers the standard deterministic descent criterion, up to constant factors. 
In this sense, the stable regime notion introduced above extends the deterministic edge-of-stability threshold while also accounting for the fluctuations caused by mini-batch sampling.

\section{Additional Related Work}
\label{app:additional-related-work}

In this appendix, we provide a more detailed discussion of related work. 
Our work is connected to several lines of research: edge-of-stability dynamics, large-stepsize optimization for logistic and cross-entropy losses as well as implicit bias in separable classification.

\paragraph{Edge of stability.}
The edge-of-stability phenomenon was systematically documented by \citet{cohen2021gradient}, who observed that gradient descent on neural networks often operates near the classical stability threshold predicted by the largest eigenvalue of the Hessian. 
Classical smooth optimization theory suggests that, for deterministic gradient descent, monotone descent of the loss is guaranteed when the stepsize is smaller than a curvature-dependent threshold. 
However, empirical evidence shows that modern neural network training often uses stepsizes that exceed this threshold, while the loss still decreases over longer time scales. 

\citet{damian2023self} proposed a self-stabilization mechanism for deterministic gradient descent at the edge of stability. 
Their analysis shows that, in certain settings, the dynamics themselves can reduce the effective instability and bring the trajectory back toward a stable region. 
Relatedly, \citet{cohen2024understandingoptimizationdeeplearning} studied edge-of-stability dynamics through the lens of central flows, further clarifying the structure of large-stepsize training. 
These works focus primarily on deterministic gradient descent. 
In contrast, our work studies stochastic gradient descent, where non-monotonicity of the loss has two possible sources: curvature-driven oscillations caused by large stepsizes, and random fluctuations caused by mini-batch sampling. 
This distinction makes the stochastic setting qualitatively different from the deterministic one and motivates a stability notion based on conditional Lyapunov drift rather than pointwise monotone descent.

 Another related although distinct mechanism in SGD is the stochastic sharpness gap {\citep{liao2026sgdedgestabilitystochastic}}, where the sharpness of SGD equilibrates at a level slightly smaller than $2/\eta$ depending on the batch size. The aforementioned phenomenon is observed primarily in the square loss and differs from the stochastic stabilization mechanism that drives the stochastic dynamics of the cross-entropy loss, making the curvature and noise decrease asymptotically to zero. The unification of all mechanisms into a principled framework remains an interesting open question left for future work.
{Closely related, \citet{andreyev2025edgestochasticstabilityrevisiting} revisit the edge of stability for mini-batch SGD, primarily for the square loss, and show that the quantity that stabilizes at $2/\eta$ is a mini-batch notion of sharpness rather than the sharpness of the full-batch loss, which instead settles below this threshold. Their work highlights that the instability of SGD is governed by the interaction between curvature and mini-batch noise, which is also the starting point of our analysis.

\textit{Loss-based versus sharpness-based stability.} We emphasize that our notion of stability differs from the sharpness-based characterization of the edge of stability in \citet{cohen2021gradient} and \citet{damian2023self}. In particular, our analysis does not rely on progressive sharpening, and we do not claim that the dynamics track an instability boundary of the sharpness. For separable data and the cross-entropy loss, the curvature is controlled by the loss itself, since $\|\nabla^2 L(W)\|_{\text{op}} \le 2L(W)$ (Lemma~\ref{lemma: S3}). Consequently, in the late phase of training the stability of the dynamics is determined by the level of the loss, which motivates a loss-based stability criterion, in line with the deterministic analyses of \citet{wu2024large} and \citet{cai2024large}. Our results should therefore be interpreted as a characterization of large-stepsize SGD through loss-based stochastic stability, rather than as a characterization of the sharpness dynamics at the edge of stability.}

\paragraph{Large-stepsize methods for logistic and cross-entropy losses.}
A recent line of work has provided rigorous guarantees for large-stepsize gradient descent in classification problems with logistic-type losses. 
\citet{wu2024large} showed that, for logistic regression, large stepsizes can accelerate optimization and that non-monotonicity of the loss can be beneficial rather than harmful. 
Their analysis demonstrates that deterministic gradient descent may pass through a non-monotone edge-of-stability phase before entering a stable regime where the loss decreases at a controlled rate. 
\citet{cai2024large} extended this perspective to non-homogeneous two-layer networks under binary logistic loss, showing that large stepsizes can improve margin growth and optimization speed. 
More recent work has continued to refine the theory of large-stepsize training for logistic-type objectives \citep{wu2025large}.

Our work is closest in spirit to this line of research, but differs in three important ways. 
First, we study SGD rather than full-batch GD. 
This requires separating curvature-driven instability from mini-batch noise, since the realized loss of SGD may be non-monotone even in regimes where the expected loss decreases. 
Second, we consider the multiclass cross-entropy loss rather than only the binary logistic loss. 
The multiclass setting introduces interactions between all incorrect classes through the softmax probabilities, and each example must be separated from every competing class. 
Third, our stability criterion is stochastic: the stable regime is defined as the region where the cross-entropy loss decreases in conditional expectation over the next mini-batch. 
Thus, our results can be viewed as a stochastic analogue of the deterministic large-stepsize theory for logistic and cross-entropy losses.

\paragraph{Large-stepsize methods for Convex Optimization}
A recent line of works studies how acceleration can be achieved by using merely large stepsizes in GD. More specifically, \citet{altschuler2025acceleration} show that GD can be accelerated in strongly convex problems by using the silver stepsize schedule. Another work \citep{grimmer2024provably} studies the effect of long steps in ensuring descent not per iteration of the optimization algorithm but along multiple iterations of the trajectory. For separable convex problems, random stepsizes are shown to accelerate the rate of convergence of GD \citep{altschuler2024acceleration}, while an extension \citep{bok2024accelerating} to proximal GD is also available for smooth convex optimization. 
\paragraph{Implicit bias in separable classification.}
Our analysis is also related to the literature on the implicit bias of gradient methods for separable classification. 
For logistic and cross-entropy losses on separable data, the empirical risk does not have a finite minimizer. 
Instead, the training loss converges to zero while the norm of the parameters diverges. 
At the same time, the direction of the parameters may converge to a maximum-margin solution. 
This phenomenon was established for linear logistic regression by \citet{soudry2018implicit}, and subsequent works refined the convergence rates and extended the analysis to broader settings \citep{ji2019implicit,nacson2019convergence}.

This behavior is important for the stability notion used in our work. 
Many classical notions of stability are formulated around convergence to, or recurrence near, a finite equilibrium point. 
Such notions are not directly aligned with separable cross-entropy minimization, where the relevant limiting object is not a finite parameter vector. 
Instead, the quantity that converges is the loss. 
For this reason, we use the cross-entropy loss itself as the Lyapunov function. 
The resulting stable regime is not a neighborhood of a finite minimizer; it is a region in parameter space where the conditional drift of the loss is non-positive. 
This loss-based viewpoint matches the geometry of separable classification more naturally than equilibrium-based stability criteria.

\paragraph{Stochastic optimization.}
For SGD, \citet{nacson2019stochastic} proved convergence to zero loss with a fixed non-vanishing learning rate for homogeneous linear classifiers trained on linearly separable data, under both sampling with and without replacement. It is an interesting question to study large step size SGD with random reshuffling \citep{emmanouilidis2024stochastic} or other without replacement sampling strategies. For neural networks trained with SGD and small step sizes, \citet{brutzkus2018sgd} showed that SGD can learn overparameterized two-layer networks on linearly separable data. In our case, we extend the two-layer neural network analysis for SGD in the setting of large step sizes.%

\newpage

\section{Proof for Cross-entropy Loss}
\subsection{Notation}
For each sample $(x_i, y_i)$, let
\begin{eqnarray}
    z_i &=& W^\top x_i, \nonumber \\
    p_i &=& \mathrm{softmax}(z_i), \nonumber
\end{eqnarray}
where $p_i(j) = \frac{e^{z_i(j)}}{\sum_{\ell=1}^K e^{z_i(\ell)}}$. The gradient and Hessian of the cross-entropy loss evaluated at the sample $(x_i, y_i)$ are given by
\begin{eqnarray}
    \nabla \ell_i(W) &=& (p_i - e_{y_i})x_i^\top, \label{eq:grad-softmax}\\
    \nabla^2 \ell_i(W) &=& (\mathrm{diag}(p_i) - p_i p_i^\top) \otimes (x_i x_i^\top). \label{eq:hess-softmax}
\end{eqnarray}
Define the potential functions $F, G: \R^{d \times K} \rightarrow \R_{\geq 0}$ with
\begin{eqnarray}
F(W) &=& \frac{1}{n} \sum_{i=1}^n \sum_{j \ne y_i} e^{-\langle W({y_i}) - W(j), x_i\rangle}, \label{eq:FW}\\
G(W) &=& \frac{1}{n} \sum_{i=1}^n (1 - p_i(y_i)). \label{eq:GW}
\end{eqnarray}
Since $\log(1 + u) \le u, \forall u \in (0, 1]$, it follows that
\begin{eqnarray}
G(W) \leq L(W) \le F(W). \label{eq: L_lower_upper_bound}
\end{eqnarray}

\subsection{Preparatory Lemmas}
We provide the analysis of SGD for the multi-class cross-entropy loss with any step size $\eta > 0$. The proof of Theorems~\ref{thm:sgd_crossentropy}, \ref{thm: stable_regime}, \ref{thm: self_stab} relies on a sequence of lemmas that establish key properties of the dynamics and the geometry of the loss landscape. We first introduce the aforementioned lemmas and then present the main theoretical results.

\begin{lemma}[Perceptron-type inequality]\label{lemma: S1}
Let Assumption~\eqref{assumpt: separable} hold. Then, for $W \in \R^{d \times K}$ it holds 
\begin{eqnarray}
    \langle \nabla L(W), W^* \rangle _F \le -\gamma\, G(W). \label{eq:S1}
\end{eqnarray}
\end{lemma}
\begin{proof}
    From~\eqref{eq:grad-softmax}, we have that 
    \begin{eqnarray}
    \langle \nabla L(W), W^* \rangle_F &\stackrel{\eqref{eq:grad-softmax}}{=}& \frac{1}{n}\sum_{i=1}^{n} \langle (p_i - e_{y_i})x_i^\top, W^* \rangle_F  = \frac{1}{n}\sum_{i=1}^n \langle p_i - e_{y_i}, W^{*\top} x_i \rangle_F. \nonumber
    \end{eqnarray}
    Letting $\Delta_i = W^{*\top}x_i$, we have that 
    \begin{eqnarray}
    \langle p_i - e_{y_i}, \Delta_i \rangle _F
    = \sum_{j \ne y_i} p_i(j)\big(\Delta_i(j) - \Delta_i(y_i)\big). \nonumber
    \end{eqnarray}
    By Assumption~\ref{assumpt: separable}, we have that $\Delta_i(y_i) - \Delta_i(j) \ge \gamma$, and thus we obtain for $i \in [n]$ that
    \begin{eqnarray}
    \langle p_i - e_{y_i}, \Delta_i \rangle \le -\gamma \sum_{j \ne y_i} p_i(j) 
    = -\gamma (1 - p_i(y_i)). \label{eq:G_per_i}
    \end{eqnarray}
    Summing for $i=1, ..., n$ and dividing by $n$, we get
    \begin{eqnarray}
        \frac{1}{n} \sum_{i=1}^{n} \langle p_i - e_{y_i}, \Delta_i \rangle &\leq& - \frac{\gamma}{n} \sum_{i=1}^{n} (1 - p_i(y_i)) \nonumber \\
        \iff \langle \nabla L(W), W^* \rangle _F &\leq& -\gamma\, G(W).
    \end{eqnarray}
\end{proof}

\begin{lemma}\label{lemma: gradient_bound}
For every \(W \in \R^{d \times K}\), it holds $\forall i\in[n]$ that
\begin{eqnarray}
\left\|\frac{1}{b} \sum_{i=1}^b \nabla \ell(W^\top x_i)\right\|_F
&\leq&
\frac{\sqrt{2}}{b} \sum_{i=1}^b\ell(W^\top x_i)\leq
{\sqrt{2}\,n\,L(W)},
\label{eq:individual_grad_bound}
\\
\|\nabla L(W)\|_F
&\leq&
\sqrt{2}\,G(W)
\leq
\sqrt{2}\,L(W).
\label{eq:gradnorm_goal_extended}
\end{eqnarray}
\end{lemma}
\begin{proof}
It holds that $\nabla \ell(W^\top x_i)
=
(p_i-e_{y_i})x_i^\top,
\nabla L(W)
=
\frac{1}{n}\sum_{i=1}^n (p_i-e_{y_i})x_i^\top$,
where \(p_i=\text{softmax}(W^\top x_i)\).\\
We, first, bound the individual gradients. Since \((p_i-e_{y_i})x_i^\top\) is a rank-one matrix, we have
\begin{eqnarray}
\|\nabla \ell(W^\top x_i)\|_F
&=&
\|(p_i-e_{y_i})x_i^\top\|_F
\nonumber\\
&=&
\|p_i-e_{y_i}\|_2\,\|x_i\|_2
\nonumber\\
&\leq&
\|p_i-e_{y_i}\|_2 .
\label{eq:indiv_step1}
\end{eqnarray}
Next, it holds that
\begin{eqnarray}
\|p_i-e_{y_i}\|_2^2
&=&
\sum_{j=1}^K \bigl(p_i(j)-\one_{\{j=y_i\}}\bigr)^2
\nonumber\\
&=&
\bigl(1-p_i(y_i)\bigr)^2+\sum_{j\neq y_i} p_i(j)^2
\nonumber\\
&\leq&
\bigl(1-p_i(y_i)\bigr)^2
+
\left(\sum_{j\neq y_i} p_i(j)\right)^2
\nonumber\\
&=&
\bigl(1-p_i(y_i)\bigr)^2+\bigl(1-p_i(y_i)\bigr)^2
\nonumber\\
&=&
2\bigl(1-p_i(y_i)\bigr)^2.
\label{eq:indiv_step2}
\end{eqnarray}
Therefore, we get
\begin{eqnarray}
\|p_i-e_{y_i}\|_2
\leq
\sqrt{2}\bigl(1-p_i(y_i)\bigr).
\label{eq:indiv_step3}
\end{eqnarray}
Combining \eqref{eq:indiv_step1} and \eqref{eq:indiv_step3}, we obtain
\begin{eqnarray}
\|\nabla \ell(W^\top x_i)\|_F
\leq
\sqrt{2}\bigl(1-p_i(y_i)\bigr),
\qquad \forall i\in[n].
\end{eqnarray}
Since $\ell(W^\top x_i)
=
-\log p_i(y_i)
\geq
1-p_i(y_i)$ and $L(W) = \frac{1}{n}\sum_{i=1}^{n} \ell(W^Tx_i)$, it holds for any $b \geq 1$ that
\begin{eqnarray}
\left\|\frac{1}{b} \sum_{i=1}^b \nabla \ell(W^\top x_i)\right\|_F
&\leq&
\frac{1}{b} \sum_{i=1}^b \|\nabla \ell(W^\top x_i)\|_F \nonumber \\
&\leq& \frac{\sqrt{2}}{b} \sum_{i=1}^b\bigl(1-p_i(y_i)\bigr)\nonumber \\
&\leq& \frac{\sqrt{2}}{b} \sum_{i=1}^b\ell(W^\top x_i)\nonumber\\
&\leq&
{\sqrt{2}\max_{j\in[n]}\ell(W^\top x_j)}\nonumber\\
&\leq&
{\sqrt{2}\,n\,L(W)}
\nonumber.
\end{eqnarray}

We now turn to the full gradient. From triangle inequality, we have that
\begin{eqnarray}
\|\nabla L(W)\|_F
&=&
\Bigl\|
\frac{1}{n}\sum_{i=1}^n (p_i-e_{y_i})x_i^\top
\Bigr\|_F
\nonumber\\
&\leq&
\frac{1}{n}\sum_{i=1}^n \|(p_i-e_{y_i})x_i^\top\|_F
\nonumber\\
&=&
\frac{1}{n}\sum_{i=1}^n \|\nabla \ell(W^\top x_i)\|_F
\nonumber\\
&\leq&
\frac{\sqrt{2}}{n}\sum_{i=1}^n \bigl(1-p_i(y_i)\bigr)
\nonumber\\
&=&
\sqrt{2}\,G(W).
\end{eqnarray}
Finally, since \(G(W)\leq L(W)\), we conclude that
\begin{eqnarray}
\|\nabla L(W)\|_F
\leq
\sqrt{2}\,G(W)
\leq
\sqrt{2}\,L(W).
\end{eqnarray}
This completes the proof.
\end{proof}

\begin{lemma}\label{lemma: S3}
For all $W \in \R^{d \times K}$, it holds that
\begin{eqnarray}
\|\nabla^2 L(W)\|_{\text{op}} \le 2 L(W). \label{eq:S3}
\end{eqnarray}
\end{lemma}
\begin{proof}
From~\eqref{eq:hess-softmax} and using the fact that the operator norm of a Kronecker product satisfies $\|A \otimes B\| = \|A\|\,\|B\|$, we obtain
\begin{eqnarray}
\|\nabla^2 \ell_i(W)\|_{\text{op}}
&=& \|\mathrm{diag}(p_i) - p_i p_i^\top\|_{\text{op}} \|x_i x_i^\top\|_{\text{op}} \nonumber\\
&\stackrel{\|x\|_2 \leq 1}{\leq}& \|\mathrm{diag}(p_i) - p_i p_i^\top\|_{\text{op}}. \label{eq:step2}
\end{eqnarray}
For the positive semidefinite matrix $A_i := \mathrm{diag}(p_i) - p_i p_i^\top$, it holds that
\begin{eqnarray}
  \|A_i\|_{\text{op}} &\leq& {\tr(A_i)} = 1 - {\|p_i\|_2^2} = 1 - p_i(y_i)^2 - \sum_{j\neq {y_i}} p_i(j)^2 \nonumber \\
  &\leq& 1 - p_i(y_i)^2 = (1 - p_i(y_i)) (1+p_i(y_i)) \nonumber \\
  &\leq& 2 (1 - p_i(y_i))\nonumber  
\end{eqnarray}
Using the inequality $-\log z \ge 1-z, \forall z > 0$ and the fact that $\ell_i(W) = -\log p_i(y_i)$ , we have that 
\begin{eqnarray}
    1 - p_i(y_i) \le \ell_i(W).
\end{eqnarray}
Hence, we obtain
\begin{eqnarray}
    \|A_i\|_{\text{op}} \le 2(1 - p_i(y_i)) \le 2\,\ell_i(W). \label{eq:step4}
\end{eqnarray}
Combining \eqref{eq:step2} and \eqref{eq:step4}, we obtain 
\begin{eqnarray}
\|\nabla^2 \ell_i(W)\|_{\text{op}} \;\le\; 2\,\ell_i(W).
\end{eqnarray}
Summing over $i=1,\ldots,n$ and dividing by $n$, we have that
\begin{eqnarray}
\|\nabla^2 L(W)\|_{\text{op}}
\;\le\; \frac{1}{n}\sum_{i=1}^n \|\nabla^2 \ell_i(W)\|
\;\le\; \frac{2}{n}\sum_{i=1}^n \ell_i(W)
\;=\; 2\,L(W).
\end{eqnarray}
\end{proof}

\begin{lemma} \label{lemma: average_G_bound}
    The iterates of \ref{SGD} with any step size $\eta > 0$ satisfy that
    \begin{equation}
        \frac{1}{t}\sum_{k=0}^{t-1} \Expe{G(W_k)} \leq \frac{{\sqrt{2(K-1)}}+2\ln(\gamma^2\eta t)+2\eta \left(1 + \frac{1}{b}\right)}{\eta \gamma^2 t} \nonumber
    \end{equation}
    where $G(W) = \frac{1}{n}\sum_{i=1}^n (1-p_i(y_i))$.
\end{lemma}
\begin{proof}
    From Lemma~\ref{lemma: S1}, the gradient of the softmax loss satisfies
\begin{eqnarray}
\langle \nabla L(W_t), W^*\rangle_F
\;\le\;
-\gamma\,G (W_t),
\label{eq:S1-reuses}
\end{eqnarray}
where $G (W_t)=\frac{1}{n}\sum_{i=1}^n (1-p_i(y_i))$.
Using the \eqref{SGD} update $W_{t+1}=W_t-\eta g_t$, we get
\begin{eqnarray}
\langle W_{t+1},W^*\rangle_F
&=&
\langle W_t,W^*\rangle_F
-\eta\,\langle g_t,W^*\rangle_F \nonumber
\end{eqnarray}
Taking expectation condition on the filtration $\filter_t$ and using that $\Expep{g_t} = \nabla L(W_t)$, we have that
\begin{eqnarray}
    \Expep{\langle W_{t+1},W^*\rangle_F} &=& \langle W_t,W^*\rangle_F
-\eta\,\langle \nabla L(W_t),W^*\rangle_F \nonumber \\
&\stackrel{\eqref{eq:S1-reuses}}{\geq}& \langle W_t, W^*\rangle_F + \eta \gamma G(W_t) \label{eq:perc-steps}
\end{eqnarray}
where at the last step we have used Lemma~\ref{lemma: S1}. Taking expectation again and using the tower law of expectation, we get 
\begin{eqnarray}
    \Expe{\langle W_{t+1},W^*\rangle_F - \langle W_t, W^*\rangle_F } &\geq& \eta \gamma \Expe{G(W_t)} \nonumber
\end{eqnarray}
Unrolling the recursion and multiplying by $\frac{1}{t}$, we obtain
\begin{eqnarray}
    \frac{1}{t} \Expe{\langle W_t,W^*\rangle_F}  &\geq& \frac{\eta \gamma}{t} \sum_{k=0}^{t-1} \Expe{G(W_k)}, \label{eq:perc-sums}
\end{eqnarray}
where we have used the fact that $W_0 = \textbf{0}$. Since $\|W^*\|_F=1$ and from applying the Cauchy-Schwarz inequality it holds $\|W_t\|_F\ge \langle W_t,W^*\rangle_F$, we obtain
\begin{eqnarray}
    \frac{1}{t}\sum_{k=0}^{t-1} \Expe{G(W_k)} \leq \frac{\Expe{\|W_t\|_F}}{\eta\gamma t} \label{eq:Gavgs}
\end{eqnarray}
We, next, bound the $\Expe{\|W_t\|_F}$. Applying the triangle inequality, we have that
\begin{eqnarray}
\|W_t\|_F \leq \|W_t-U\|_F+\|U\|_F \label{eq:Wt-bound}
\end{eqnarray}
Thus, it suffices to bound the terms $\|W_t-U\|_F, \|U\|_F$.
From {Lemma}~\ref{prop:split_{in}gd}, for any decomposition $U=U_1+U_2$ with $U_2 = \frac{\eta}{\gamma} \left(1 + \frac{1}{b}\right) W^*$, it holds that
\begin{eqnarray}
\frac{\Expe{\|W_t-U\|_F^2}}{2\eta t} +\frac{1}{t}\sum_{k=0}^{t-1} \Expe{L(W_k)} &\leq& L(U_1) +\frac{\|W_0-U\|_F^2}{2\eta t} \nonumber \\
\Rightarrow \Expe{\|W_t-U\|_F^2} &\leq& 2\eta t L(U_1) + \|U\|_F^2,
\label{eq:split-reuse}
\end{eqnarray}
where we have used the fact that $W_0 = \textbf{0}$ and the non-negativity of the loss $L$.
Selecting $U_1=\alpha W^*$ with $\alpha=\frac{\ln(\gamma^2\eta t)}{\gamma}$.  
Then, it holds
\begin{eqnarray}
L(U_1)\;\le\;(K-1)e^{-\alpha\gamma}
\;=\;\frac{K-1}{\gamma^2\eta t}
\label{eq:U1-bound}
\end{eqnarray}
Thus, substituting into \eqref{eq:split-reuse} and using Jensen inequality, we have that
\begin{eqnarray}
    \Expe{\|W_t-U\|_F} \leq \sqrt{\Expe{\|W_t-U\|_F^2}} \leq \sqrt{2\eta t L(U_1) + \|U\|_F^2} \leq \sqrt{\frac{{2}(K-1)}{\gamma^2} + \|U\|_F^2} \label{eq:Wt-U}
\end{eqnarray}
We, next, bound the norm $\|U\|_F$. It holds that 
\begin{eqnarray}
    {\|U\|_F = \left\|\alpha W^*+\frac{\eta}{\gamma} \left(1 + \frac{1}{b}\right) W^*\right\|_F = \frac{1}{\gamma}\left[\ln(\gamma^2\eta t)+\eta \left(1 + \frac{1}{b}\right)\right],} \label{eq:U-norm}
\end{eqnarray}
{where we have used $\|W^*\|_F = 1$ and $\ln(\gamma^2\eta t) \geq 0$. Taking expectation in \eqref{eq:Wt-bound} and using \eqref{eq:Wt-U}, we have that}
\begin{align}
\Expe{\|W_t\|_F} &\le \Expe{\|W_t-U\|_F}+\|U\|_F
\;\le\;
\frac{\sqrt{2(K-1)}}{\gamma} + 2\|U\|_F \nonumber\\
&=
\frac{\sqrt{2(K-1)}+2\ln(\gamma^2\eta t)+2\eta \left(1 + \frac{1}{b}\right)}{\gamma}\label{eq:Wt-bound1}
\end{align}
where we have used the fact that $\sqrt{a+b}\le \sqrt{a}+\sqrt{b}$ for $a,b\ge 0$ and the {expression} of $\|U\|_F$ from \eqref{eq:U-norm}.
Substituting \eqref{eq:Wt-bound1} into \eqref{eq:Gavgs} gives
\begin{eqnarray}
    \frac{1}{t}\sum_{k=0}^{t-1} \Expe{G(W_k)} \leq \frac{{\sqrt{2(K-1)}}+2\ln(\gamma^2\eta t)+2\eta \left(1 + \frac{1}{b}\right)}{\eta \gamma^2 t} \label{eq:meanG}
\end{eqnarray}
\end{proof}
\begin{lemma}\label{lemma: stochastic_grad_bound}
    The stochastic gradient oracle $g_t$ of the cross-entropy loss satisfies
    \begin{eqnarray}
        \Expep{\|g_t\|^2} \leq \left(1+\frac{{2}n}{b\gamma^2}\right) \|\nabla L(W_t)\|^2 \nonumber
    \end{eqnarray}
\end{lemma}
\begin{proof}
    From the bias-variance decomposition, it holds 
    \begin{eqnarray}
        \Expep{\|g_t\|^2} &=& \|\nabla L(W_t)\|_F^2+\Expep{\|g_t-\nabla L(W_t)\|_F^2} \nonumber\\
            &=& \|\nabla L(W_t)\|_F^2+\frac1b\,\Var(\nabla \ell_i(W_t)\mid \mathcal F_t) \nonumber\\
            &\leq& \|\nabla L(W_t)\|_F^2+\frac1b\,\Expep{\|\nabla \ell_i(W_t)\|_F^2}\nonumber
        \end{eqnarray}
        as $g_t$ is the average of $b$ i.i.d. samples (conditionally on $\mathcal F_t$).
        {We, next, compare the individual gradients with the full gradient, adapting the binary-classification argument of \citet{nacson2019stochastic} to the multi-class setting. From \eqref{eq:softmax_sq_bound} in the proof of Lemma~\ref{lemma:variance} and $\|x_i\|_2 \leq 1$, it holds that $\|\nabla \ell_i(W_t)\|_F^2 \leq 2\bigl(1-p_i(y_i)\bigr)^2$. Since $0 \leq 1-p_i(y_i) \leq 1$, we get
        \begin{eqnarray}
            \sum_{i=1}^n \|\nabla \ell_i(W_t)\|_F^2 \leq 2\sum_{i=1}^n \bigl(1-p_i(y_i)\bigr)^2 \leq 2\left(\sum_{i=1}^n \bigl(1-p_i(y_i)\bigr)\right)^2 = 2n^2 G(W_t)^2. \nonumber
        \end{eqnarray}
        Moreover, from Lemma~\ref{lemma: S1} and the Cauchy--Schwarz inequality with $\|W_*\|_F = 1$, we have $\gamma G(W_t) \leq -\langle \nabla L(W_t), W_*\rangle_F \leq \|\nabla L(W_t)\|_F$. Combining the two inequalities gives
        \begin{eqnarray}
            \sum_{i=1}^n \|\nabla \ell_i(W_t)\|_F^2 \leq \frac{2n^2}{\gamma^2}\|\nabla L(W_t)\|_F^2. \label{eq:lin-sq-grad-comparison}
        \end{eqnarray}
        Substituting the definition of $\Expep{\|\nabla \ell_i(W_t)\|_F^2}$ and using \eqref{eq:lin-sq-grad-comparison}, we obtain}
        \begin{eqnarray}
            \Expep{\|g_t\|^2} \leq \|\nabla L(W_t)\|_F^2+\frac{1}{nb}\sum_{i=1}^n \|\nabla \ell_i(W_t)\|_F^2  &\leq& \left(1+\frac{{2}n}{b\gamma^2}\right) \|\nabla L(W_t)\|_F^2 \nonumber
        \end{eqnarray}
\end{proof}
\vspace{-0.5cm}
\subsection{Variance Bound}
As a stochastic optimization algorithm, SGD is leveraging a stochastic mini-batch gradient at each iteration instead of the full-batch gradient. Controlling the noise induced by the stochastic oracles in the dynamics is of critical importance for analyzing the convergence. We can bound the variance of the stochastic mini-batches by utilizing the following lemma.
\begin{lemma}[Variance bound]\label{lemma:variance}
Let \(g(W) = \frac{1}{b}\sum_{i\in\mathcal{B}_t}\nabla\ell(f(x_i;W),y_i)\) denote the mini-batch gradient. Then, it holds that
\begin{equation}
    \Expe{\| g(W) - \nabla L(W)\|^2}\;\le\;\frac{2}{b}\, G(W),\label{eq:variance-bound_}
\end{equation}
where $G(W)
    := \frac{1}{n}
    \sum_{i=1}^n
    \Big[1 - \mathrm{softmax}(W x_i)_{y_i}\Big]$.
\end{lemma}

\begin{proof}
We begin by decomposing the variance of the stochastic oracle. 
For uniform sampling of minibatch $B_t$ of size $b \geq 1$, it holds that
\begin{eqnarray}
\mathbb{E}\!\left[\|g_t - \nabla L(W_t)\|_F^2 \mid \mathcal{F}_t\right]
\;=\;
\frac{1}{b}\,\mathrm{Var}_i\!\big(\nabla \ell_i(W_t)\big)
\;\le\;
\frac{1}{b}\,\mathbb{E}_i\!\left[\|\nabla \ell_i(W_t)\|_F^2\right],
\label{eq:oracle_var_start}
\end{eqnarray}
where we used that $\mathrm{Var}(Z) \le \mathbb{E}[\|Z\|^2]$ for any random variable $Z$.
For each $i\in[n]$, by the definition of $\nabla \ell_i(W_t)$ we have
\begin{eqnarray}
\|\nabla \ell_i(W_t)\|_F^2 
&=& \|(p_i - e_{y_i})x_i^\top\|_F^2
\;=\;
\|p_i - e_{y_i}\|_2^2 \,\|x_i\|_2^2
\;\le\;
\|p_i - e_{y_i}\|_2^2, \label{eq:gradnorm_frob}
\end{eqnarray}
since $\|x_i\|_2 \le 1$. We next expand the norm, obtaining
\begin{align}
\|p_i - e_{y_i}\|_2^2
&= (1 - p_i(y_i))^2 + \sum_{j\neq y_i} p_i(j)^2 
\le
(1 - p_i(y_i))^2 + \left(\sum_{j\neq y_i} p_i(j)\right)^2 \nonumber\\
&\le
2(1 - p_i(y_i))^2, \label{eq:softmax_sq_bound}
\end{align}
where we used that {$\sum_{j\neq y_i}p_i(j)=1-p_i(y_i)$}.
Thus, from \eqref{eq:gradnorm_frob}, \eqref{eq:softmax_sq_bound} we have
\begin{eqnarray}
\|\nabla \ell_i(W_t)\|_F^2 \le 2(1 - p_i(y_i))^2 \le 2(1 - p_i(y_i)).
\end{eqnarray}
Summing for $i\in[n]$ and dividing by $\frac{1}{n}$ gives
\begin{eqnarray}
\mathbb{E}_i[\|\nabla \ell_i(W_t)\|_F^2] \;\le\; 2\,G(W_t).
\label{eq:moment_bound}
\end{eqnarray}
Substituting \eqref{eq:moment_bound} into \eqref{eq:oracle_var_start}, we obtain the final 
\begin{eqnarray}
\mathbb{E}\!\left[\|g_t - \nabla L(W_t)\|_F^2 \mid \mathcal{F}_t\right] \leq \frac{2}{b} G(W_t).
\end{eqnarray}
\end{proof}

\subsection{Proofs for EoS Regime}
\begin{lemma}\label{prop:split_{in}gd}
Let $U=U_1+U_2$ with $U_1 \in \R^{d\times K}, U_2 = \frac{\eta}{\gamma} \left(1 + \frac{1}{b}\right) W^*$. Then, for all $t\geq 1$, it holds that
\begin{eqnarray}
    \frac{\Expe{\|W_t - U\|_F^2}}{2\eta t} + \frac{1}{t}\sum_{k=0}^{t-1} \Expe{L(W_k)} \leq L(U_1) + \frac{\|W_0 - U\|_F^2}{2\eta t} \label{eq:split_{in}gd1}
\end{eqnarray} 
\end{lemma}
\begin{proof}
   From the update rule \eqref{SGD}, we have that
\begin{eqnarray}
\|W_{t+1}-U\|_F^2 &=& \|W_t-U\|_F^2 + 2\eta \langle g_t, U-W_t\rangle_F + \eta^2 \|g_t\|_F^2 \nonumber \\
&=& \|W_t-U\|_F^2 + 2\eta \langle g_t, U_1-W_t\rangle_F + \eta^2 \left(\frac{2}{\eta} \langle g_t, U_2\rangle + \|g_t\|_F^2\right) \label{eq:telesc}
\end{eqnarray}
Taking expectation condition on the filtration $\filter_t$ and using the unbiasedness property of the stochastic oracles, we have
\begin{eqnarray}
    \Expep{\|W_{t+1}-U\|_F^2} &=& \|W_t-U\|_F^2 + 2\eta \langle \nabla L(W_t), U_1-W_t\rangle_F\nonumber \\
    && + \eta^2 \left(\frac{2}{\eta} \langle \nabla L(W_t), U_2\rangle + \Expep{\|g_t\|_F^2} \right) \label{eq: p_sgd_eq1a}
\end{eqnarray}
Using Lemmas~\ref{lemma: S1}, \ref{lemma:variance} with the selected $U_2 = \frac{\eta}{\gamma} \left(1 + \frac{1}{b}\right) W^*$, we have
\begin{eqnarray}
\frac{2}{\eta} \langle \nabla L(W_t), U_2\rangle + \Expep{\|g_t\|_F^2} &\stackrel{\eqref{eq:G_per_i}}{\leq}& - 2\left(1 + \frac{1}{b}\right) G(W_t) + \Expep{\|g_t\|_F^2}\nonumber \\
&=& - 2\left(1 + \frac{1}{b}\right) G(W_t) + \|\nabla L (W_t)\|_F^2 \nonumber \\
&& + \Expep{\|g_t - \nabla L(W_t)\|_F^2}\nonumber \\
&\stackrel{\text{Lemma}~\ref{lemma:variance}}{\leq}& - 2\left(1 + \frac{1}{b}\right) G(W_t) + \|\nabla L (W_t)\|_F^2 + \frac{2}{b} G(W_t)\nonumber \\
&\stackrel{\text{Lemma}~\ref{lemma: S1}}{\leq}& - 2\left(1 + \frac{1}{b}\right) G(W_t) + 2 \left(1 + \frac{1}{b}\right) G(W_t) \nonumber \\
&\leq& 0, \label{eq: p_lemma_eq2a}
\end{eqnarray}
Thus, substituting \eqref{eq: p_lemma_eq2a} into \eqref{eq: p_sgd_eq1a} gives
\begin{eqnarray}
   \Expep{\|W_{t+1}-U\|_F^2} &\leq& \|W_t-U\|_F^2 + 2\eta \langle\nabla L(W_t), U_1-W_t\rangle_F \nonumber
\end{eqnarray}
Using the fact that $L$ is convex, we have that
\begin{eqnarray}
    \Expep{\|W_{t+1}-U\|_F^2} &\leq& \|W_t-U\|_F^2 + 2\eta \left[L(U_1) - L(W_t)\right] \nonumber
\end{eqnarray}
Taking expectation again and using the tower law of expectation, we obtain
\begin{eqnarray}
    \Expe{\|W_{t+1}-U\|_F^2} &\leq& \Expe{\|W_t-U\|_F^2} + 2\eta \Expe{L(U_1) - L(W_t)}\nonumber
\end{eqnarray}
Summing for $k=0, ..., t-1$ and dividing by $2\eta t$, we have that
\begin{eqnarray}
    \frac{\Expe{\|W_t - U\|_F^2}}{2\eta t} + \frac{1}{t}\sum_{k=0}^{t-1} \Expe{L(W_k)} \leq L(U_1) + \frac{\|W_0 - U\|_F^2}{2\eta t}  \nonumber
\end{eqnarray}
\end{proof}

\subsubsection*{Proof of Theorem~\ref{thm:sgd_crossentropy}}
\label{app: thm:sgd_crossentropy}
\begin{proof}
    From Lemma~\ref{prop:split_{in}gd}, we have that for all $t\geq 1$ it holds
        \begin{eqnarray}
            \frac{\Expe{\|W_t - U\|_F^2}}{2\eta t} + \frac{1}{t}\sum_{k=0}^{t-1} \Expe{L(W_k)} \leq L(U_1) + \frac{\|U\|_F^2}{2\eta t}   \label{eq:split_{in}gd}
        \end{eqnarray}
    For $U_1=\alpha W^*$, $U_2 = \frac{\eta}{\gamma} \left(1 + \frac{1}{b}\right) W^*$, $\alpha=\frac{\ln(\gamma^2\eta t)}{\gamma}$, it holds that 
    \begin{eqnarray}
        L(U_1) &\leq& F(U_1) = (K-1)e^{-\alpha\gamma}=\frac{K-1}{\gamma^2\eta t}, \label{eq: p1_eq1} \\
        \|U\|^2_F &\leq& 2 \|U_1\|_F^2 + 2 \|U_2\|_F^2 \leq \frac{2\ln^2(\gamma^2\eta t)}{\gamma^2}+\frac{2\eta^2 (1+\frac{1}{b})^2}{\gamma^2}. \label{eq: p1_eq2}
    \end{eqnarray}
    Substituting \eqref{eq: p1_eq1}, \eqref{eq: p1_eq2} into~\eqref{eq:split_{in}gd} and rearranging the terms, we obtain
    \begin{eqnarray}
        \frac{1}{t}\sum_{k=0}^{t-1} \Expe{L(W_k)} \leq \frac{K-1 + \ln^2(\gamma^2\eta t)+\eta^2 (1+\frac{1}{b})^2}{\gamma^2\eta t} \nonumber
    \end{eqnarray}
\end{proof}
\subsection{Proofs for Stable Regime}
\label{app: proof_for_stable}
\begin{lemma}
\label{lem:ce_sgd_decrease}
If for some $t \geq 0$ the iterates of \eqref{SGD} satisfy
\[
L(W_t)\le {\min\left\{\frac{1}{8\eta\left(1+\frac{{2}n}{b\min\left\{\gamma^2, 1\right\}}\right)},\frac{1}{\sqrt2n\eta}\right\}},
\]
then
\begin{align}
\Expep{L(W_{t+1})-L(W_t)}\le 0,
\label{eq:sgd_upper}
\end{align}
{In particular, \eqref{eq:sgd_upper} holds for every $W_t \in \mathcal S$, since $\tilde L \le \frac{1}{\eta c}$.}
\end{lemma}
\begin{proof}
Recall that the minibatch gradient is $g_t=\frac1b\sum_{j\in \mathcal B_t}\nabla \ell_j(W_t)$.
Fix $i\in[n]$ and define the logit increment
\[
\Delta s_i:=(W_{t+1}-W_t)x_i=-\eta g_t x_i.
\]
Since $\|x_i\|_2\le 1$, we have
\begin{equation}
\label{eq:delta_si_bound}
\|\Delta s_i\|_2
\le \eta \|g_t\|_F \leq  \frac{\eta}{b} \sum_{j\in B_t}\|\nabla \ell_j(W_t)\|_2 \stackrel{\eqref{eq:individual_grad_bound}}{\leq} {\sqrt{2}\,\eta\, n\, L(W_t)}.
\end{equation}
where we have used \eqref{eq:individual_grad_bound} from Lemma~\ref{lemma: gradient_bound}.
Let $\phi_i(\alpha):=\ell_i(W_t-\alpha \eta g_t)$ for $\alpha\in[0,1]$.
From Taylor's theorem, there exists $\theta_i\in(0,1)$ such that
\begin{align}
\ell_i(W_{t+1})
&=
\ell_i(W_t)
-\eta \langle \nabla \ell_i(W_t), g_t\rangle
+\frac12 (\Delta s_i)^\top H_i(\theta_i)(\Delta s_i),
\label{eq:taylor_ce}
\end{align}
where $H_i(\theta_i)
=
\text{diag}(p_i^\theta)-p_i^\theta (p_i^\theta)^\top,
$ and $p_i^\theta:=\text{softmax}\!\big((W_t-\theta_i\eta g_t)x_i\big)$.
Since $H_i(\theta_i)$ is positive semidefinite, we have that
\[
(\Delta s_i)^\top H_i(\theta_i)(\Delta s_i)
\le \|H_i(\theta_i)\|_{\text{op}}\,\|\Delta s_i\|_2^2.
\]
Hence, from \eqref{eq:delta_si_bound} we obtain
\begin{align}
\ell_i(W_{t+1})
\le
\ell_i(W_t)
-\eta \langle \nabla \ell_i(W_t), g_t\rangle
+\frac{\eta^2}{2}\|H_i(\theta_i)\|_{\text{op}}\,\|g_t\|_F^2.
\label{eq:one_sample_descent}
\end{align}
For the softmax cross-entropy loss, if $u\in\mathbb R^K$ and $p=\text{softmax}(u)$ it holds that
\[
\text{diag}(p)-pp^\top \succeq 0
\qquad\text{and}\qquad
\|\text{diag}(p)-pp^\top\|_{\text{op}}\le \tr(\text{diag}(p)-pp^\top)=1-\|p\|_2^2.
\]
Therefore, we have that
\[
\|H_i(\theta_i)\|_{\text{op}}
\le 1-\|p_i^\theta\|_2^2
\le 2\bigl(1-p_i^\theta(y_i)\bigr).
\]
Since $-\ln a \ge 1-a,\forall a\in(0,1]$, we have that $1-p_i^\theta(y_i)\le \ell_i(W_t-\theta_i\eta g_t)$ and thus
\[
\|H_i(\theta_i)\|_{\text{op}}
\le 2\,\ell_i(W_t-\theta_i\eta g_t).
\]
In the stable-regime, the loss is small enough so that the logits move by at most a constant amount. In particular, {from \eqref{eq:delta_si_bound} and $L(W_t)\leq\frac{1}{\sqrt2n\eta}$,} it holds that
$\|\Delta s_i\|_\infty\le {\|\Delta s_i\|_2\le} 1$. Then, the softmax probabilities along the segment between $W_t$ and $W_{t+1}$ change by at most an absolute constant factor, and hence
\[
\ell_i(W_t-\theta_i\eta g_t)\le e^2 \ell_i(W_t).
\]
Therefore, we have that
\[
\|H_i(\theta_i)\|_{\text{op}}\le 2e^2 \ell_i(W_t)\le 16\,\ell_i(W_t),
\]
where we used $e^2<8$. Substituting into \eqref{eq:one_sample_descent}, we get
\begin{align}
\ell_i(W_{t+1})
\le
\ell_i(W_t)
-\eta \langle \nabla \ell_i(W_t), g_t\rangle
+8\eta^2 \ell_i(W_t)\|g_t\|_F^2.
\label{eq:sample_ce_final}
\end{align}
Summing \eqref{eq:sample_ce_final} over $i=1,\dots,n$ and dividing by $n$, we have
\begin{align}
L(W_{t+1})-L(W_t)
\le
-\eta \langle \nabla L(W_t), g_t\rangle
+8\eta^2 L(W_t)\|g_t\|_F^2.
\label{eq:loss_descent_batch}
\end{align}
Taking conditional expectation with respect to $\mathcal F_t$ and using
$\Expep{g_t}=\nabla L(W_t)$, we obtain
\begin{align}
\Expep{L(W_{t+1})-L(W_t)}
\le
-\eta \|\nabla L(W_t)\|_F^2
+8\eta^2 L(W_t)\Expep{\|g_t\|_F^2}.
\label{eq:cond_descent_main}
\end{align}
Using Lemma~\ref{lemma: stochastic_grad_bound}, we can bound the last term by $\Expep{\|g_t\|^2} \leq \left(1+\frac{{2}n}{b\gamma^2}\right) \|\nabla L(W_t)\|^2$ and thus have
\begin{align}
\Expep{L(W_{t+1})-L(W_t)}
&\le
-\eta
\left[
1-8\eta\left(1+\frac{{2}n}{b\gamma^2}\right)L(W_t)
\right]
\|\nabla L(W_t)\|_F^2.
\label{eq:final_conditional_descent}
\end{align}
For $L(W_t)\le \frac{1}{8\eta\left(1+\frac{{2}n}{b\gamma^2}\right)}$, we have that
\[
\Expep{L(W_{t+1})-L(W_t)}\le 0.
\]
Taking full expectation and applying the tower property of expectation, we get 
\[
\Expe{L(W_{t+1})-L(W_t)}\le 0.
\]
Combining the two requirements for the loss{, namely $L(W_t)\leq\frac{1}{\sqrt{2}n\eta}$ for the logit bound and $L(W_t)\le \frac{1}{8\eta\left(1+\frac{{2}n}{b\gamma^2}\right)}$ for the descent, and using $\gamma^2\geq\min\{\gamma^2,1\}$, it suffices to have $L(W_t) \leq \min\left\{\frac{1}{8\eta\left(1+\frac{{2}n}{b\min\left\{\gamma^2, 1\right\}}\right)},\frac{1}{\sqrt{2}n\eta}\right\}$.}
\end{proof}

\begin{lemma} \label{lemma: sgd_stable_phase}
If the iterates satisfy $W_k \in \mathcal{S}, \forall k \in[t_{1}, t)$, then it holds that
\begin{eqnarray}
    \Expe{L(W_{t})} &\leq& \frac{8 F(W_{t_{1}}) + 4\ln^2(\gamma^2\eta (t-{t_{1}}))}{7\gamma^2\eta (t-{t_{1}})}\label{eq:stable_phase-avg}
\end{eqnarray}
where $F(W)
        =
        \frac{1}{n}
        \sum_{i=1}^n
        \sum_{j \ne y_i}
        e^{-\langle W(y_i)-W(j),x_i\rangle}.$
\end{lemma}
\begin{proof}
From the update rule of \eqref{SGD} we have that
\begin{eqnarray}
\|W_{t+1}-U\|_F^2 &=& \|W_t-U\|_F^2 - 2\eta \langle g(W_t), W_t - U\rangle_F + \eta^2 \|g(W_t)\|_F^2. \nonumber
\end{eqnarray}
Taking expectation condition on the filtration $\filter_t$ and using the unbiased property $\Expep{g(W_{t})} = \nabla L(W_t)$, we have
\begin{eqnarray}
    \Expep{\|W_{t+1}-U\|_F^2} &=& \|W_t-U\|_F^2 - 2\eta \langle \nabla L(W_t), W_t - U\rangle_F + \eta^2 \Expep{\|g(W_t)\|_F^2} \label{eq:one-step}
\end{eqnarray}
Since $L$ is convex, we get 
\begin{eqnarray}
&& L(U) \geq L(W_t) + \langle \nabla L(W_t),\, U - W_t\rangle_F \nonumber \\
&\iff&  - \langle \nabla L(W_t), W_t - U\rangle_F \leq L(U) - L(W_t) \label{eq:convexity}
\end{eqnarray}
Substituting \eqref{eq:convexity} into \eqref{eq:one-step}, we obtain
\begin{eqnarray}
\|W_{t+1}-U\|_F^2
&\le& \|W_t-U\|_F^2 + 2\eta\big(L(U)-L(W_t)\big) + \eta^2 \Expep{\|g(W_t)\|_F^2} \nonumber
\end{eqnarray}
Taking expectation on both sides and using the tower law of expectation, we get
\begin{eqnarray}
    \Expe{\|W_{t+1}-U\|_F^2} &\leq& \Expe{\|W_t-U\|_F^2} + 2\eta\Expe{L(U)-L(W_t)} +\eta^2 \Expe{\|g(W_t)\|_F^2}. \label{eq:key}
\end{eqnarray}
Using the bias-variance decomposition and Lemma~\ref{lemma:variance}, we have that
\begin{eqnarray}
    \eta^2 \Expep{\|g(W_t)\|_F^2} &=& \eta^2 \|\nabla L(W_t)\|_F^2 + \eta^2 \Expep{\|g(W_t) - \nabla L(W_t)\|_F^2} \nonumber \\
    &= & \eta^2 \|\nabla L(W_t)\|_F^2+\frac{\eta^2}{b}\,\Var(\nabla \ell_i(W_t)\mid \mathcal F_t) \nonumber\\
    &\leq& \eta^2 \|\nabla L(W_t)\|_F^2+\frac{\eta^2}{b}\,\Expep{\|\nabla \ell_i(W_t)\|_F^2} \nonumber \\
    &\stackrel{}{\leq}& \eta^2\|\nabla L(W_t)\|_F^2 + \frac{\eta^2}{nb} \sum_{i=1}^n \|\nabla \ell_i(W_t)\|_F^2\nonumber
\end{eqnarray}
Using {\eqref{eq:lin-sq-grad-comparison}, namely} $\sum_{i=1}^n \|\nabla \ell_i(W_t)\|_F^2 \le \frac{{2}n^2}{\gamma^2}\|\nabla L(W_t)\|_F^2,$ we have
\begin{eqnarray}
    \eta^2 \Expep{\|g(W_t)\|_F^2} &\stackrel{}{\leq}&  \eta^2\left(1 + \frac{{2}n}{b\gamma^2}\right)  \|\nabla L(W_t)\|_F^2\nonumber \\
    &\stackrel{\text{Lemma}~\ref{lemma: gradient_bound}}{\leq}&  2\eta^2\left(1 + \frac{{2}n}{b\gamma^2}\right) L(W_t)^2\nonumber \\
    &\stackrel{}{\leq}& \frac{\eta}{4} L(W_t). \label{eq:stable-grad}
\end{eqnarray}
where at the last step we have used the fact that {$L(W_t) \leq \tilde L \leq \frac{1}{8\eta\left(1 + \frac{2n}{b\gamma^2}\right)}$, since $W_t \in \mathcal S$}.
Thus, \eqref{eq:key} becomes
\begin{eqnarray}
    \Expep{\|W_{t+1}-U\|_F^2} &\leq& \|W_t-U\|_F^2 + 2\eta [L(U)-L(W_t)] + \frac{\eta}{4} L(W_t) \nonumber \\
    &\stackrel{}{\leq}& \|W_t-U\|_F^2 +2 \eta L(U)-\frac{7\eta}{4} L(W_t) \label{eq:simplified}
\end{eqnarray}
Taking expectation again, using the tower law of expectation and rearranging the terms, we obtain
\begin{eqnarray}
    \Expe{L(W_t)} \leq \frac{8}{7} L(U) + \frac{4 \Expe{\|W_t-U\|_F^2 - \|W_{t+1}-U\|_F^2}}{7\eta}. \label{eq:one-line-usable}
\end{eqnarray}
Summing over $k=t_{1},\ldots,t-1$ and dividing by $(t-t_{1})$, we have that:
\begin{eqnarray}
\frac{1}{t-t_{1}}\sum_{k=t_{1}}^{t-1} \Expe{L(W_k)}
&\leq& \frac{8}{7}L(U) + \frac{4}{7\eta\,(t-t_{1})}\sum_{k=t_{1}}^{t-1}\big(\|W_k-U\|_F^2 - \|W_{k+1}-U\|_F^2\big) \nonumber\\
&=& \frac{8}{7}L(U) + \frac{4\left(\|W_{t_{1}}-U\|_F^2 - \|W_t-U\|_F^2\right)}{7\eta\,(t-t_{1})} \label{eq:avg-pre-split}
\end{eqnarray}
Letting $U = W_{t_{1}} + U_1$ with $U_1 =\alpha W^*$ and $\alpha=\frac{\ln(\gamma^2\eta (t-{t_{1}}))}{\gamma}$ we have that
\begin{eqnarray}
    \|W_{t_{1}}-U\|_F^2 = \|U_1\|_F^2 = \frac{\ln^2(\gamma^2\eta (t-{t_{1}}))}{\gamma^2} 
    \label{eq_proof_stable_1}
\end{eqnarray}
and 
\begin{eqnarray}
    L(U) &\leq& F(U) = \frac{1}{n} \sum_{i=1}^n e^{-\langle U, x\rangle_F} \nonumber \\
    &=& \frac{1}{n} \sum_{i=1}^n e^{-\langle U_1, X\rangle_F} e^{-\langle W_s, X\rangle_F} \nonumber \\
    &=& \frac{1}{n} \sum_{i=1}^n  \frac{e^{-\langle W_s, X\rangle_F}}{\gamma^2\eta (t-{t_{1}})} \nonumber \\
    &=& \frac{F(W_s)}{\gamma^2\eta (t-{t_{1}})} \label{eq_proof_stable_2}
\end{eqnarray}
Substituting \eqref{eq_proof_stable_1}, \eqref{eq_proof_stable_2} into \eqref{eq:avg-pre-split}, we obtain
\begin{eqnarray}
    \min_{t_{1} \leq k \leq t-1} \Expe{L(W_k)} &\leq& \frac{8 F(W_{t_{1}}) + 4\ln^2(\gamma^2\eta (t-{t_{1}}))}{7\gamma^2\eta (t-{t_{1}})} \label{eq:proof_stable_2}
\end{eqnarray}
Applying Lemma~\ref{lem:ce_sgd_decrease} for the stable region $\mathcal{S}_{[{t_{1}}, t-1]}$ we have that 
\begin{eqnarray}
    \Expe{L(W_t)} \leq \Expe{L(W_{t-1})} \leq ... \leq \Expe{L(W_{t_{1}})} \nonumber 
\end{eqnarray}
and thus from \eqref{eq:proof_stable_2} we have that
\begin{eqnarray}
    \Expe{L(W_{t})} &\leq& \min_{t_{1} \leq k \leq t-1} \Expe{L(W_k)} \leq \frac{8 F(W_{t_{1}}) + 4\ln^2(\gamma^2\eta (t-{t_{1}}))}{7\gamma^2\eta (t-{t_{1}})} \nonumber
\end{eqnarray}
\end{proof}

\begin{lemma}\label{lemma:phase-transition}
Let $W_0 = \mathbf 0$. There exists $t_{\mathrm{in}} \le t_{\max}(\delta)$ such
that with probability at least $1 - \delta$, we have that
$L(W_{t_{\mathrm{in}}}) \le \tilde L$, where
\[
  t_{\max}(\delta) \;\ge\; \frac{1}{\gamma^{2}}\,\max\!\left\{
   \frac{8\bigl(K - 1 + 2\eta(2 + \tfrac{1}{b})\bigr)}{\eta\,\delta\,\tilde L},
   \;\;
   \frac{32}{\eta\,\delta\,\tilde L}\,\ln\!\Bigl(\frac{32}{\eta\,\delta\,\tilde L}\Bigr)
   \right\}.
\]
\end{lemma}
\begin{proof}
From Lemma~\ref{lemma: average_G_bound} applied at $t = t_{\max}(\delta)$, we have that
\begin{equation}
  \frac{1}{t_{\max}}\sum_{k=0}^{t_{\max} - 1}\mathbb E[G(W_k)]
   \;\le\; \frac{2\bigl(K - 1 + 2\ln(\gamma^{2}\eta t_{\max}) + 2\eta(1 + \tfrac{1}{b})\bigr)}{\eta\,\gamma^{2}\, t_{\max}}. \label{eq:lin-mc-Gbar-0}
\end{equation}
Using inequality $\ln(\gamma^{2}\eta t_{\max}) = \ln(\gamma^{2} t_{\max}) + \ln(\eta) \le \ln(\gamma^{2} t_{\max}) + \eta$,
we obtain
\begin{equation}
  \frac{1}{t_{\max}}\sum_{k=0}^{t_{\max} - 1}\mathbb E[G(W_k)]
   \;\le\; \frac{2\bigl(K - 1 + 2\ln(\gamma^{2} t_{\max}) + 2\eta(2 + \tfrac{1}{b})\bigr)}{\eta\,\gamma^{2}\, t_{\max}}. \label{eq:lin-mc-Gbar-1}
\end{equation}
For $\delta \in (0, 1)$, we select $t_{\max}$ so that the right-hand side of~\eqref{eq:lin-mc-Gbar-1} is at most $\delta\,\tilde L/2$.
We, next, verify the formula for $t_{\max}$:
\begin{itemize}
\item For $t \geq \dfrac{8\bigl(K - 1 + 2\eta(2 + \tfrac{1}{b})\bigr)}{\eta\,\delta\,\tilde L\,\gamma^{2}}$,
it holds that
$\dfrac{2\bigl(K - 1 + 2\eta(2 + \tfrac{1}{b})\bigr)}{\eta\,\gamma^{2}\, t}
\;\le\; \dfrac{\delta\,\tilde L}{4}$.
\item For $t \geq \dfrac{32}{\eta\,\delta\,\tilde L\,\gamma^{2}}\,\ln\!\Bigl(\dfrac{32}{\eta\,\delta\,\tilde L}\Bigr)$,
it holds that
$\gamma^{2}\, t \;\ge\; \dfrac{32}{\eta\,\delta\,\tilde L}\,\ln\!\Bigl(\dfrac{32}{\eta\,\delta\,\tilde L}\Bigr)$,
which is a sufficient condition (see Lemma~G.5 in \citet{cai2024large}) for
$\dfrac{4\,\ln(\gamma^{2}\, t)}{\eta\,\gamma^{2}\, t}
\;\le\; \dfrac{\delta\,\tilde L}{4}$.
\end{itemize}
Thus, it suffices to select $t_{\max}(\delta)$ such that
\begin{equation}
  t_{\max}(\delta) \;\ge\; \frac{1}{\gamma^{2}}\,\max\!\left\{
   \frac{8\bigl(K - 1 + 2\eta(2 + \tfrac{1}{b})\bigr)}{\eta\,\delta\,\tilde L},
   \;\;
   \frac{32}{\eta\,\delta\,\tilde L}\,\ln\!\Bigl(\frac{32}{\eta\,\delta\,\tilde L}\Bigr)
   \right\}. \label{eq:lin-mc-tmax-condition}
\end{equation}
Thus, we have that
\begin{equation}
  \frac{1}{t_{\max}}\sum_{k=0}^{t_{\max} - 1}\mathbb E[G(W_k)]
   \;\le\; \frac{\delta\,\tilde L}{2}. \label{eq:lin-mc-Gbar-final}
\end{equation}
Hence, there exists $t_{\mathrm{in}} \le t_{\max}(\delta)$ such that
\begin{equation}
  \mathbb E[G(W_{t_{\mathrm{in}}})] \;\le\; \frac{\delta\,\tilde L}{2}. \label{eq:lin-mc-Gtin}
\end{equation}
Let the event
$\mathcal E_{\mathrm{in}} := \{G(W_{t_{\mathrm{in}}}) \le \tilde L/2\}$.
Since $G(W_{t_{\mathrm{in}}}) \ge 0$, using Markov's inequality, we get
\begin{equation}
  \mathbb P(\mathcal E_{\mathrm{in}}^{c})
   \;=\; \mathbb P\bigl(G(W_{t_{\mathrm{in}}}) > \tilde L/2\bigr)
   \;\le\; \frac{2\,\mathbb E[G(W_{t_{\mathrm{in}}})]}{\tilde L}
   \;\le\; \delta, \label{eq:lin-mc-Markov}
\end{equation}
and thus $\mathbb P(\mathcal E_{\mathrm{in}}) \ge 1 - \delta$. Conditioning on the event $\mathcal E_{\mathrm{in}}$ and using
$G(W_{t_{\mathrm{in}}}) \le \frac{\tilde L}{2} \le {\frac{1}{2n}}${, where we used $\tilde L \le {\frac{1}{n}}$ by the definition of $\tilde L$,} every
term in the sum
$G(W_{t_{\mathrm{in}}}) = \tfrac{1}{n}\sum_{i=1}^{n}(1 - p_i(W_{t_{\mathrm{in}}})_{y_i})$
satisfies
$1 - p_i(W_{t_{\mathrm{in}}})_{y_i} \le 1/2$ and thus $p_i(W_{t_{\mathrm{in}}})_{y_i} \ge 1/2$.
For every $j \neq y_i$, thus, it holds that
\begin{equation}
  z_i(W_{t_{\mathrm{in}}})_{y_i} - z_i(W_{t_{\mathrm{in}}})_j
   \;=\; \ln\!\frac{p_i(W_{t_{\mathrm{in}}})_{y_i}}{p_i(W_{t_{\mathrm{in}}})_j} \;\ge\; 0, \label{eq:lin-mc-positive-margin}
\end{equation}
which gives
\begin{eqnarray}
  F(W_{t_{\mathrm{in}}})
  &=& \frac{1}{n}\sum_{i=1}^{n}\sum_{j \neq y_i}\frac{p_i(W_{t_{\mathrm{in}}})_j}{p_i(W_{t_{\mathrm{in}}})_{y_i}} \nonumber\\
  &\le& \frac{1}{n}\sum_{i=1}^{n}2\bigl(1 - p_i(W_{t_{\mathrm{in}}})_{y_i}\bigr) \nonumber\\
  &=& 2\, G(W_{t_{\mathrm{in}}}) \nonumber\\
  &\le& \tilde L, \nonumber
\end{eqnarray}
where we used $p_i(W)_{y_i} \ge 1/2$ to bound $1/p_i(W)_{y_i} \le 2$ and
$\sum_{j \neq y_i}p_i(W)_j = 1 - p_i(W)_{y_i}$. Since $L(W) \le F(W)$ for every
$W \in \mathbb R^{d \times K}$, on the event $\mathcal E_{\mathrm{in}}$, we have that
\begin{equation}
  L(W_{t_{\mathrm{in}}}) \;\le\; F(W_{t_{\mathrm{in}}}) \;\le\; \tilde L. \label{eq:lin-mc-L-bound}
\end{equation}
Combining~\eqref{eq:lin-mc-L-bound} with $\mathbb P(\mathcal E_{\mathrm{in}}) \ge 1 - \delta$,
we conclude that there exists $t_{\mathrm{in}} \le t_{\max}(\delta)$ such that with probability at
least $1 - \delta$ it holds that $L(W_{t_{\mathrm{in}}}) \le \tilde L$, and the
dynamics enter the stable regime $\mathcal S$ at iteration $t_{\mathrm{in}}$.
\end{proof}

\subsubsection*{Proof of Theorem~\ref{thm: stable_regime}}
\label{app: thm: stable_regime}
\begin{proof}
    The theorem is proved by combining Lemma~\ref{lemma: sgd_stable_phase} and Lemma~\ref{lemma:phase-transition}. 
\end{proof}
\newpage
\subsection{Proofs for Stochastic Stabilization Mechanism}
\begin{lemma}
\label{lemma:exit-prob-loss}
Let $\mathcal{E}_t=\left\{L(W_t)\leq \tilde L,\; L(W_{t+1})>\tilde L\right\}$. If at iteration $t > 0$ the dynamics are in the stable regime, then the exit probability satisfies
\begin{equation}\label{eq:exit-prob-loss-only}
\Pr(\mathcal{E}_t\mid \F_t)\;\le\; 2\exp\!\left(
-\frac{b D_t}{4\eta L(W_t)^{3/2}}\right),
\end{equation}
{where $D_t = \frac{3R_t^2}
{8\left(48\eta L(W_t)^{1/2}+(1+n)R_t\right)}$, $R_t=\sqrt{L(W_t)+16\Delta_t}-\sqrt{L(W_t)}$ and $\Delta_t=\tilde L-L(W_t)\geq 0$.}
\end{lemma}
\begin{proof}
From inequality~\eqref{eq:loss_descent_batch}, we have that
\begin{eqnarray}
    L(W_{t+1})
    &\leq&
    L(W_t)-\eta\langle \nabla L(W_t),g_t\rangle+8\eta^2L(W_t)\|g_t\|^2.
\end{eqnarray}
Let $\xi_t:=g_t-\nabla L(W_t)$ denote the noise at iteration $t$. Substituting
$g_t=\nabla L(W_t)+\xi_t$ and expanding, we get
\begin{eqnarray}
    L(W_{t+1})
    &\leq&
    L(W_t)-\eta\bigl(1-8\eta L(W_t)\bigr)\|\nabla L(W_t)\|^2 \nonumber\\
    &&-\eta\bigl(1-16\eta L(W_t)\bigr)\langle \nabla L(W_t),\xi_t\rangle
    +8\eta^2L(W_t)\|\xi_t\|^2. \label{eq:ce-exit-start}
\end{eqnarray}

On the event
\[
\mathcal{E}_t=\left\{L(W_t)\leq \tilde L,\; L(W_{t+1})>\tilde L\right\},
\]
it holds that
\begin{eqnarray}
    \tilde L-L(W_t)+\eta\bigl(1-8\eta L(W_t)\bigr)\|\nabla L(W_t)\|^2
    &<&
    -\eta\bigl(1-16\eta L(W_t)\bigr)\langle \nabla L(W_t),\xi_t\rangle \nonumber\\
    &&+8\eta^2L(W_t)\|\xi_t\|^2. \label{eq:ce-exit-event}
\end{eqnarray}
{Since we seek an upper bound on $\Pr(\mathcal E_t\mid\mathcal F_t)$, we need a condition on the noise $\xi_t$ that is \emph{necessary} for the event $\mathcal E_t$. To this end, we lower bound the left-hand side and upper bound the right-hand side of~\eqref{eq:ce-exit-event}.

In the stable regime, we have that $L(W_t)\leq\tilde L=\frac{1}{\eta c}$ with $c\geq 8$, and hence $\eta L(W_t)\leq \frac18$. Thus, it holds that
\begin{eqnarray}
1-8\eta L(W_t)\;\geq\; 0
\qquad\text{and}\qquad
\bigl|1-16\eta L(W_t)\bigr|\;\leq\; 1. \label{eq:ce-coeff-bounds}
\end{eqnarray}
From the first inequality of~\eqref{eq:ce-coeff-bounds}, the left-hand side of~\eqref{eq:ce-exit-event} is lower bounded by
\begin{eqnarray}
\tilde L-L(W_t)+\eta\bigl(1-8\eta L(W_t)\bigr)\|\nabla L(W_t)\|^2
\;\geq\;
\tilde L-L(W_t)
\;=\;\Delta_t. \label{eq:ce-lhs-lower}
\end{eqnarray}
For the right-hand side of~\eqref{eq:ce-exit-event}, using the Cauchy--Schwarz inequality, the second inequality of~\eqref{eq:ce-coeff-bounds} and Lemma~\ref{lemma: gradient_bound}, namely $\|\nabla L(W_t)\|\leq\sqrt2L(W_t)$, we have that
\begin{eqnarray}
&&-\eta\bigl(1-16\eta L(W_t)\bigr)\langle \nabla L(W_t),\xi_t\rangle
+8\eta^2L(W_t)\|\xi_t\|^2 \nonumber\\
&\leq&
\eta\bigl|1-16\eta L(W_t)\bigr|\,\|\nabla L(W_t)\|\,\|\xi_t\|
+8\eta^2L(W_t)\|\xi_t\|^2 \nonumber\\
&\leq&
\sqrt2\,\eta L(W_t)\|\xi_t\|
+8\eta^2L(W_t)\|\xi_t\|^2. \label{eq:ce-rhs-upper}
\end{eqnarray}
Combining~\eqref{eq:ce-exit-event}, \eqref{eq:ce-lhs-lower} and~\eqref{eq:ce-rhs-upper}, on the event $\mathcal E_t$ it necessarily holds that
\begin{eqnarray}
\Delta_t
\;<\;
\sqrt2\,\eta L(W_t)\|\xi_t\|
+8\eta^2L(W_t)\|\xi_t\|^2. \label{eq:ce-necessary}
\end{eqnarray}
The right-hand side of~\eqref{eq:ce-necessary} is increasing in $\|\xi_t\|\geq 0$. Hence, solving the quadratic inequality, \eqref{eq:ce-necessary} is equivalent to
\begin{eqnarray}
\|\xi_t\|
\;>\;
r_t
:=
\frac{\sqrt{2\eta^2L(W_t)^2+32\eta^2L(W_t)\Delta_t}-\sqrt2\,\eta L(W_t)}{16\eta^2L(W_t)}
=
\frac{R_t}{8\sqrt2\,\eta L(W_t)^{1/2}}, \label{eq:ce-radius}
\end{eqnarray}
where $R_t=\sqrt{L(W_t)+16\Delta_t}-\sqrt{L(W_t)}\geq 0$. Therefore, we have that $\mathcal E_t\subseteq\{\|\xi_t\|>r_t\}$ and thus
\begin{eqnarray}
\Pr(\mathcal E_t\mid \mathcal F_t)
\;\leq\;
\Pr\left(\|\xi_t\|>r_t\mid \mathcal F_t\right). \label{eq:ce-exit-noise}
\end{eqnarray}}

Conditionally on the filtration $\mathcal F_t$, the noise
\begin{eqnarray}
\xi_t=\frac{1}{b}\sum_{i=1}^b \zeta_{t,i},
\qquad
\zeta_{t,i}:=\nabla \ell_i(W_t)-\nabla L(W_t),
\end{eqnarray}
is the average of $b$ centered i.i.d.\ vectors. From Lemma~\ref{lemma: gradient_bound} and Lemma~\ref{lemma:variance}, each term in the sum satisfies the bound
\begin{eqnarray}
\|\zeta_{t,i}\|
&\leq&
\|\nabla \ell_i(W_t)\|+\|\nabla L(W_t)\|
\leq
\sqrt{2}(1+n)L(W_t), \label{eq:ce-zeta-bound}
\end{eqnarray}
and the variance bound
\begin{eqnarray}
\Expe{\|\zeta_{t,i}\|^2\mid \mathcal F_t}
&\leq&
2L(W_t). \label{eq:ce-zeta-var}
\end{eqnarray}
{By the vector Bernstein inequality applied to $\xi_t$ with radius $r_t$, we have
\begin{eqnarray}
\Pr(\mathcal E_t\mid \mathcal F_t)
&\leq&
2\exp\!\left(
-\frac{b\,r_t^2}
{4L(W_t)+\frac{2\sqrt{2}(1+n)}{3}L(W_t)\,r_t}
\right) \nonumber\\
&=&
2\exp\!\left(
-\frac{b\,R_t^2/\bigl(128\eta^2L(W_t)\bigr)}
{4L(W_t)+\frac{(1+n)}{12\eta}L(W_t)^{1/2}R_t}
\right) \nonumber\\
&=&
2\exp\!\left(
-\frac{b}{4\eta L(W_t)^{3/2}}
\cdot
\frac{3R_t^2}
{8\left(48\eta L(W_t)^{1/2}+(1+n)R_t\right)}
\right), \nonumber
\end{eqnarray}
which proves~\eqref{eq:exit-prob-loss-only}. Finally, note that as $L(W_t)\to0$ we have $R_t\to4\sqrt{\tilde L}$ and $D_t\to\frac{3\sqrt{\tilde L}}{2(1+n)}>0$, so the exit probability is of order $\exp\bigl(-\Omega\bigl(b/(\eta L(W_t)^{3/2})\bigr)\bigr)$.}
\end{proof}

\begin{lemma}\label{lemma:return}
Assume that the dynamics exit the stable regime at $t_{\mathrm{out}} > 0$ and fix $\delta \in (0, 1)$. Then, with probability at least $1 - \delta$ the iterates of \eqref{SGD} return inside the stable regime in at most
\begin{equation*}
t_{\mathrm{re}} = \left\lceil\frac{4 }{\gamma^2 \eta \delta \tilde{L}} \max\!\left\{ A,\; {8} \ln\!\left( \frac{{32}}{\gamma^2 \eta \delta \tilde{L}} \right) \right\} \right\rceil
\end{equation*}
number of steps, where $A = {5(K - 1) + 8 \ln^2(\gamma^2 \eta t_{\mathrm{out}}) + 10 \eta^2 \big(1 + \tfrac{1}{b}\big)^2}$.
\end{lemma}
\begin{proof}
Let $U = U_1 + U_2$ with $U_1 = \alpha W_*$, $U_2 = \frac{\eta}{\gamma}\big(1 + \tfrac{1}{b}\big) W_*$ and $\alpha = \frac{\ln(\gamma^2 \eta (t - t_{\mathrm{out}}))}{\gamma}$. From the update rule of \eqref{SGD}, we have that
\begin{align*}
\|W_{t+1} - U\|_F^2 &= \|W_t - U\|_F^2 + 2\eta \langle g_t, U - W_t \rangle_F + \eta^2 \|g_t\|_F^2 \\
&= \|W_t - U\|_F^2 + 2\eta \langle g_t, U_1 - W_t \rangle_F \\
&\quad + \eta^2 \!\left( \frac{2}{\eta} \langle g_t, U_2 \rangle_F + \|g_t\|_F^2 \right).
\end{align*}
Taking expectation conditional on the filtration $\mathcal{F}_t$ and using the unbiasedness property of the stochastic oracles, we get
\begin{equation}\label{eq:return-1}
\begin{aligned}
\mathbb{E}\!\left[ \|W_{t+1} - U\|_F^2 \,\big|\, \mathcal{F}_t \right] &= \|W_t - U\|_F^2 + 2\eta \langle \nabla L(W_t), U_1 - W_t \rangle_F \\
&\quad + \eta^2 \!\left( \frac{2}{\eta} \langle \nabla L(W_t), U_2 \rangle_F + \mathbb{E}\!\left[ \|g_t\|_F^2 \,\big|\, \mathcal{F}_t \right] \right).
\end{aligned}
\end{equation}
Using Lemmas~\ref{lemma: S1}, \ref{lemma: gradient_bound}, \ref{lemma:variance} with the selected $U_2 = \frac{\eta}{\gamma}\big(1 + \tfrac{1}{b}\big) W_*$, we have that
\begin{equation}
\begin{aligned}
\frac{2}{\eta} \langle \nabla L(W_t), U_2 \rangle_F + \mathbb{E}\!\left[ \|g_t\|_F^2 \,\big|\, \mathcal{F}_t \right] 
&\stackrel{\text{Lemma~\ref{lemma: S1}}}{\leq} -2 \Big(1 + \tfrac{1}{b}\Big) G(W_t) + \mathbb{E}\!\left[ \|g_t\|_F^2 \,\big|\, \mathcal{F}_t \right] \\
&= -2 \Big(1 + \tfrac{1}{b}\Big) G(W_t) + \|\nabla L(W_t)\|_F^2 \\
&\quad + \mathbb{E}\!\left[ \|g_t - \nabla L(W_t)\|_F^2 \,\big|\, \mathcal{F}_t \right] \\
&\stackrel{\text{Lemma~\ref{lemma:variance}}}{\leq} -2 \Big(1 + \tfrac{1}{b}\Big) G(W_t) + \|\nabla L(W_t)\|_F^2 + \frac{2}{b} G(W_t) \\
&\stackrel{\text{Lemma~\ref{lemma: gradient_bound}}}{\leq} -2 \Big(1 + \tfrac{1}{b}\Big) G(W_t) + 2 \Big(1 + \tfrac{1}{b}\Big) G(W_t) \\
&\leq 0,\label{eq:return-2}
\end{aligned}
\end{equation}
where we have used that $G(W_t) \in [0,1]$ and hence $G(W_t)^2 \leq G(W_t)$. Substituting \eqref{eq:return-2} into \eqref{eq:return-1} gives
\begin{equation*}
\mathbb{E}\!\left[ \|W_{t+1} - U\|_F^2 \,\big|\, \mathcal{F}_t \right] \leq \|W_t - U\|_F^2 + 2\eta \langle \nabla L(W_t), U_1 - W_t \rangle_F.
\end{equation*}
Using the fact that $L$ is convex, we have that
\begin{equation*}
\mathbb{E}\!\left[ \|W_{t+1} - U\|_F^2 \,\big|\, \mathcal{F}_t \right] \leq \|W_t - U\|_F^2 + 2\eta \left[ L(U_1) - L(W_t) \right].
\end{equation*}
Taking expectation again and using the tower law of expectation, we obtain
\begin{equation*}
\mathbb{E}\!\left[ \|W_{t+1} - U\|_F^2 \right] \leq \mathbb{E}\!\left[ \|W_t - U\|_F^2 \right] + 2\eta\, \mathbb{E}\!\left[ L(U_1) - L(W_t) \right].
\end{equation*}
Summing for $k = t_{\mathrm{out}}, \ldots, t - 1$ and dividing by $2\eta(t - t_{\mathrm{out}})$, we have that
\begin{equation*}
\frac{\mathbb{E}\!\left[ \|W_t - U\|_F^2 \right]}{2\eta (t - t_{\mathrm{out}})} + \frac{1}{t - t_{\mathrm{out}}} \sum_{k = t_{\mathrm{out}}}^{t-1} \mathbb{E}\!\left[ L(W_k) \right] \leq L(U_1) + \frac{\mathbb{E}\!\left[\|W_{t_{\mathrm{out}}} - U\|_F^2\right]}{2\eta(t - t_{\mathrm{out}})}.
\end{equation*}
Using the non-negativity of $\|W_t - U\|_F^2$, we get
\begin{equation}\label{eq:return-avg}
\frac{1}{t - t_{\mathrm{out}}} \sum_{k = t_{\mathrm{out}}}^{t-1} \mathbb{E}\!\left[ L(W_k) \right] \leq L(U_1) + \frac{\mathbb{E}\!\left[\|W_{t_{\mathrm{out}}} - U\|_F^2\right]}{2\eta(t - t_{\mathrm{out}})}.
\end{equation}
For $U_1 = \alpha W_*$ with $\alpha = \frac{\ln(\gamma^2 \eta (t - t_{\mathrm{out}}))}{\gamma}$, we have that
\begin{equation}\label{eq:return-LU1}
L(U_1) \leq F(U_1) = (K-1)\, e^{-\alpha \gamma} = \frac{K - 1}{\gamma^2 \eta (t - t_{\mathrm{out}})}.
\end{equation}
We, next, bound the term $\mathbb{E}\!\left[\|W_{t_{\mathrm{out}}} - U\|_F^2\right]$. Applying $(a + b)^2 \leq 2a^2 + 2b^2$, we obtain
\begin{equation}\label{eq:return-split}
\mathbb{E}\!\left[\|W_{t_{\mathrm{out}}} - U\|_F^2\right] \leq 2\, \mathbb{E}\!\left[\|W_{t_{\mathrm{out}}}\|_F^2\right] + 2 \|U\|_F^2.
\end{equation}
For the term $\|U\|_F^2$, we have that
\begin{equation}\label{eq:return-U}
\|U\|_F^2 \leq 2 \|U_1\|_F^2 + 2 \|U_2\|_F^2 \leq \frac{2 \ln^2(\gamma^2 \eta(t - t_{\mathrm{out}}))}{\gamma^2} + \frac{2 \eta^2 \big(1 + \tfrac{1}{b}\big)^2}{\gamma^2}.
\end{equation}
For the term $\mathbb{E}\!\left[\|W_{t_{\mathrm{out}}}\|_F^2\right]$, we apply Lemma~\ref{prop:split_{in}gd} from the original time origin with comparator $U^{(0)} = U_1^{(0)} + U_2$ where $U_1^{(0)} = \alpha_0 W_*$ and $\alpha_0 = \frac{\ln(\gamma^2 \eta t_{\mathrm{out}})}{\gamma}$, which gives
\begin{equation}\label{eq:return-W-tout-diff}
\mathbb{E}\!\left[\|W_{t_{\mathrm{out}}} - U^{(0)}\|_F^2\right] \leq 2 \eta\, t_{\mathrm{out}}\, L(U_1^{(0)}) + \|U^{(0)}\|_F^2 \leq \frac{2(K-1)}{\gamma^2} + \frac{2 \ln^2(\gamma^2 \eta t_{\mathrm{out}}) + 2 \eta^2 \big(1 + \tfrac{1}{b}\big)^2}{\gamma^2},
\end{equation}
where we have used $W_0 = \textbf{0}$ and the bound \eqref{eq:return-LU1} applied at time $t_{\mathrm{out}}$. Applying $(a + b)^2 \leq 2a^2 + 2b^2$ again, we obtain
\begin{equation}\label{eq:return-W-tout}
\mathbb{E}\!\left[\|W_{t_{\mathrm{out}}}\|_F^2\right] \leq 2\, \mathbb{E}\!\left[\|W_{t_{\mathrm{out}}} - U^{(0)}\|_F^2\right] + 2 \|U^{(0)}\|_F^2 \leq \frac{4(K-1) + 8 \ln^2(\gamma^2 \eta t_{\mathrm{out}}) + 8 \eta^2 \big(1 + \tfrac{1}{b}\big)^2}{\gamma^2}.
\end{equation}
{Substituting \eqref{eq:return-U} and \eqref{eq:return-W-tout} into \eqref{eq:return-split}, we get
\begin{equation*}
\mathbb{E}\!\left[\|W_{t_{\mathrm{out}}} - U\|_F^2\right] \leq \frac{8(K-1) + 16 \ln^2(\gamma^2 \eta t_{\mathrm{out}}) + 20 \eta^2 \big(1 + \tfrac{1}{b}\big)^2 + 4 \ln^2(\gamma^2 \eta(t - t_{\mathrm{out}}))}{\gamma^2}.
\end{equation*}
Substituting this bound and \eqref{eq:return-LU1} into \eqref{eq:return-avg},} we obtain
\begin{equation}\label{eq:return-master}
\frac{1}{t - t_{\mathrm{out}}} \sum_{k = t_{\mathrm{out}}}^{t-1} \mathbb{E}\!\left[ L(W_k) \right] \leq \frac{C(t, t_{\mathrm{out}})}{\gamma^2 \eta (t - t_{\mathrm{out}})},
\end{equation}
where
\begin{equation*}
C(t, t_{\mathrm{out}}) :=  {5(K - 1) + 8 \ln^2(\gamma^2 \eta t_{\mathrm{out}}) + 10 \eta^2 \big(1 + \tfrac{1}{b}\big)^2} +  {2}\ln^2(\gamma^2 \eta(t - t_{\mathrm{out}})).
\end{equation*}
We now split the bound \eqref{eq:return-master} into two terms. Let $A := {5(K - 1) + 8 \ln^2(\gamma^2 \eta t_{\mathrm{out}}) + 10 \eta^2 \big(1 + \tfrac{1}{b}\big)^2}$.
Then $C(t, t_{\mathrm{out}}) = A + {2}\ln^2(\gamma^2 \eta (t - t_{\mathrm{out}}))$, and \eqref{eq:return-master} becomes
\begin{equation}\label{eq:return-master-split}
\frac{1}{t - t_{\mathrm{out}}} \sum_{k = t_{\mathrm{out}}}^{t-1} \mathbb{E}\!\left[ L(W_k) \right] \leq \frac{A}{\gamma^2 \eta (t - t_{\mathrm{out}})} + \frac{{2}\ln^2(\gamma^2 \eta (t - t_{\mathrm{out}}))}{\gamma^2 \eta (t - t_{\mathrm{out}})}.
\end{equation}
We select $t_{\mathrm{re}}$ so that each of the two terms on the right-hand side of \eqref{eq:return-master-split} is at most $\delta \tilde{L}/4$. Specifically, we next verify the formula for $t_{\mathrm{re}}$:
\begin{itemize}
    \item For $t - t_{\mathrm{out}} \geq \frac{4 A}{\gamma^2 \eta \delta \tilde{L}}$, it holds that
    \begin{equation*}
    \frac{A}{\gamma^2 \eta (t - t_{\mathrm{out}})} \leq \frac{\delta \tilde{L}}{4}.
    \end{equation*}
    
    \item For $t - t_{\mathrm{out}} \geq \frac{{32} }{\gamma^2 \eta \delta \tilde{L}} \ln\!\left( \frac{{32} }{\gamma^2 \eta \delta \tilde{L}} \right)$, it holds that $\gamma^2 \eta (t - t_{\mathrm{out}}) \geq \frac{{32} }{\delta \tilde{L}} \ln\!\left( \frac{{32} }{\gamma^2 \eta \delta \tilde{L}} \right)$, which is a sufficient condition (see Lemma G.5 in~\citet{cai2024large}) for
    \begin{equation*}
    \frac{ {2}\ln^2(\gamma^2 \eta (t - t_{\mathrm{out}}))}{\gamma^2 \eta (t - t_{\mathrm{out}})} \leq \frac{\delta \tilde{L}}{4}.
    \end{equation*}
\end{itemize}
Thus, it suffices to select $t_{\mathrm{re}}$ such that
\begin{equation}\label{eq:return-tre}
t_{\mathrm{re}} = \left\lceil\frac{4 }{\gamma^2 \eta \delta \tilde{L}} \max\!\left\{ A,\; {8} \ln\!\left( \frac{{32}}{\gamma^2 \eta \delta \tilde{L}} \right) \right\} \right\rceil.
\end{equation}
For this choice of $t_{\mathrm{re}}$ and selecting $t = t_{\mathrm{out}} + t_{\mathrm{re}}$, we have that
\begin{equation}\label{eq:return-final}
\frac{1}{t_{\mathrm{re}}} \sum_{k = t_{\mathrm{out}}}^{t_{\mathrm{out}} + t_{\mathrm{re}} - 1} \mathbb{E}\!\left[ L(W_k) \right] \leq \frac{\delta \tilde{L}}{2}.
\end{equation}
Hence, there exists $t_{\mathrm{in}}^{[2]} \in [t_{\mathrm{out}}, t_{\mathrm{out}} + t_{\mathrm{re}}]$ such that
\begin{equation*}
\mathbb{E}\!\left[ L(W_{t_{\mathrm{in}}^{[2]}}) \right] \leq \frac{\delta \tilde{L}}{2}.
\end{equation*}
From Markov's inequality, we have that
\begin{equation*}
\Pr\!\left( L(W_{t_{\mathrm{in}}^{[2]}}) \geq \tilde{L} \right) \leq \frac{\mathbb{E}\!\left[ L(W_{t_{\mathrm{in}}^{[2]}}) \right]}{\tilde{L}} \leq \frac{\delta}{2} \leq \delta,
\end{equation*}
which implies that
\begin{equation*}
\Pr\!\left( L(W_{t_{\mathrm{in}}^{[2]}}) \leq \tilde{L} \right) \geq 1 - \delta.
\end{equation*}
Thus, there exists $t_{\mathrm{in}}^{[2]} \leq t_{\mathrm{out}} + t_{\mathrm{re}}$ such that for $\delta \in (0, 1)$ with probability at least $1 - \delta$ it holds that $L(W_{t_{\mathrm{in}}^{[2]}}) \leq \tilde{L}$ and the dynamics enter the stable set.
\end{proof}

\subsubsection*{Proof of Theorem~\ref{thm: self_stab}}
\label{app: thm: self_stab}
\begin{proof}
    The theorem is proved by combining Lemma~\ref{lemma:exit-prob-loss} and Lemma~\ref{lemma:return}.
\end{proof}

\newpage

\section{Proofs for Two-Layer Neural Networks}\label{app:twolayer}

In this appendix, we provide the analysis of SGD applied to the two-layer
neural network~\eqref{eq: 2nn} under the multi-class cross-entropy
loss with any step size $\tilde\eta = m \eta > 0$. The proof of
convergence relies on a sequence of lemmas that establish
key properties of the dynamics and the geometry of the loss landscape.

\subsection{Notation}\label{app:twolayer:notation}

For each sample $(x_i, y_i)$, let
\[
  z_i(W) \;=\; \frac{1}{m}\sum_{r=1}^{m} a_r\, \phi(x_i^{\!\top} W^{(r)}) \in \mathbb{R}^{K},
  \qquad
  p_i(W) \;=\; \mathrm{softmax}\bigl(z_i(W)\bigr).
\]
The gradient and Hessian of the cross-entropy loss evaluated at the sample
$(x_i, y_i)$ are given by
\begin{align}
  \nabla \ell_i(W) &\;=\; J_i(W)^{\!\top}\bigl(p_i(W) - e_{y_i}\bigr), \label{eq:nn-grad}\\
  \nabla^{2}\ell_i(W) &\;=\; J_i(W)^{\!\top}\bigl(\mathrm{diag}(p_i(W)) - p_i(W)p_i(W)^{\!\top}\bigr) J_i(W)\nonumber \\
  &\;+\; \sum_{k=1}^{K}\bigl(p_i(W)_k - \mathbb 1_{\{k=y_i\}}\bigr)\,\nabla^{2} z_i(W)_k, \label{eq:nn-hess}
\end{align}
where $J_i(W)$ is the Jacobian of $W \mapsto z_i(W)$. We define the lifted
comparator
\[
  \overline{W}_{*} \;=\; \bigl(a_1 W_*, \ldots, a_m W_*\bigr) \in \mathbb{R}^{md \times K}
\]
with $\|\overline{W}_{*}\|_F \;=\; \sqrt{m}$
and let $c_\phi := \frac{1}{m^2}\sum_{r=1}^{m} a_r^{2}=\frac{1}{m}$, since $a_r \in \{\pm 1\}$. 
We define the potential functions
$F, G : \mathbb{R}^{md \times K} \to \mathbb{R}_{\ge 0}$ as follows
\begin{align}
  F(W) &\;=\; \frac{1}{n}\sum_{i=1}^{n}\sum_{j \neq y_i}
              \exp\!\bigl(-(z_i(W)_{y_i} - z_i(W)_j)\bigr), \label{eq:nn-F}\\
  G(W) &\;=\; \frac{1}{n}\sum_{i=1}^{n}\bigl(1 - p_i(W)_{y_i}\bigr). \label{eq:nn-G}
\end{align}
Since $\log(1+u) \le u$ for $u \in (0, 1]$, it follows that
\begin{equation}
  G(W) \;\le\; L(W) \;\le\; F(W). \label{eq:nn-GLF}
\end{equation}

\subsection{Preparatory Lemmas}\label{app:twolayer:prep}

The proof of Theorems~\ref{thm: sgd_2_layer_nn},~\ref{thm: sgd_2_layer_nn_stable},~\ref{thm: self_stab_2_layer_nn} relies on a sequence of lemmas
that establish key properties of the dynamics and the geometry of the loss
landscape. We first introduce the aforementioned lemmas and then present the
main theoretical result.

\begin{lemma}[Jacobian of the logits]\label{lem:nn-jacobian}
For every sample $i \in [n]$, class $k \in [K]$, and hidden unit $r \in [m]$, it holds that
\[
  \nabla_{W^{(r)}} z_i(W)_k \;=\; \frac{a_r}{m}\,\phi'\!\bigl(x_i^{\!\top} W^{(r)}_k\bigr)\, x_i e_k^{\!\top}.
\]
Moreover, the Jacobian $J_i(W)$ of the map $W \mapsto z_i(W)$ satisfies
\[
  \|J_i(W)\|_{\mathrm{op}} \;\le\; \frac{1}{\sqrt{m}}.
\]
\end{lemma}

\begin{proof}
Differentiating $z_i(W)_k = \frac{1}{m}\sum_{s=1}^{m} a_s\, \phi(x_i^{\!\top} W^{(s)}_k)$
with respect to the block $W^{(r)}$ gives the stated expression. For
$H = (H^{(1)}, \ldots, H^{(m)}) \in \mathbb R^{md \times K}$, we have that for every $k \in [K]$ it holds
\[
  \bigl(J_i(W) H\bigr)_k \;=\; \frac{1}{m}\sum_{r=1}^{m} a_r\, \phi'(x_i^{\!\top} W^{(r)}_k)\, x_i^{\!\top} H^{(r)}_k.
\]
Using $|a_r|=1$, $|\phi'|\le 1$, $\|x_i\|_2 \le 1$, and Cauchy--Schwarz across $r$, we have that
\[
  \bigl|\bigl(J_i(W) H\bigr)_k\bigr|
   \;\le\; \frac{1}{m}\sum_{r=1}^{m}\|H^{(r)}_k\|_2
   \;\le\; \frac{1}{\sqrt{m}}\Bigl(\sum_{r=1}^{m}\|H^{(r)}_k\|_2^{2}\Bigr)^{\!1/2}.
\]
Taking the square and summing over $k=1, ..., K$, we obtain
\[
  \|J_i(W) H\|_2^{2} \;\le\; \frac{1}{m}\sum_{k=1}^{K}\sum_{r=1}^{m}\|H^{(r)}_k\|_2^{2}
   \;=\; \frac{1}{m}\,\|H\|_F^{2}.
\]
Taking the supremum over $H$ satiafying $\|H\|_F = 1$, we have
\begin{eqnarray}
    \|J_i(W)\|^2_{\mathrm{op}} \;\le\; \frac{1}{m} \nonumber \\
    \Rightarrow \|J_i(W)\|_{\mathrm{op}} \;\le\; \frac{1}{\sqrt{m}}.
\end{eqnarray}
\end{proof}

\begin{lemma}[Two-layer perceptron-type inequality]\label{lem:nn-perceptron}
Let Assumptions~\ref{assumpt: separable}, \ref{assump: activation} hold.
Then, for every $W \in \mathbb R^{md \times K}$ it holds that
\begin{equation}
  \langle \nabla L(W), \overline{W}_{*}\rangle_F \;\le\; -\tilde\gamma G(W), \label{eq:nn-perceptron-bdd}
\end{equation}
where $\tilde \gamma := \alpha \gamma - (1 - \alpha)$.
\end{lemma}
\begin{proof}
By the chain rule and~\eqref{eq:nn-grad}, for each block $r \in [m]$, we have that
\begin{equation}
  \nabla_{W^{(r)}} L(W)
   \;=\; \frac{1}{n}\sum_{i=1}^{n}\frac{a_r}{m}\, x_i\bigl[(p_i(W) - e_{y_i}) \odot \phi'(x_i^{\!\top} W^{(r)})\bigr]^{\!\top}, \label{eq:nn-bdd-blockgrad}
\end{equation}
where $\odot$ denotes component-wise multiplication and $\phi'(x_i^{\!\top} W^{(r)}) \in \mathbb R^K$
is the vector with $k$-th component $\phi'(x_i^{\!\top} W^{(r)}_k)$. Since
$\overline{W}_{*} = (a_1 W_{*}, \ldots, a_m W_{*})$ and $a_r^{2} = 1$, we have
\begin{align}
  \langle \nabla L(W), \overline{W}_{*}\rangle_F
  &\;=\; \sum_{r=1}^{m}\bigl\langle \nabla_{W^{(r)}} L(W),\, a_r W_{*}\bigr\rangle_F \notag\\
  &\;=\; \frac{1}{n}\sum_{i=1}^{n}\sum_{r=1}^{m}\frac{a_r^{2}}{m}\,
       \Bigl\langle (p_i(W) - e_{y_i}) \odot \phi'(x_i^{\!\top} W^{(r)}),\, W_{*}^{\!\top} x_i\Bigr\rangle \nonumber\\
       &\;=\; \frac{1}{n}\sum_{i=1}^{n}\sum_{r=1}^{m}\frac{a_r^{2}}{m}\,
       \Bigl\langle (p_i(W) - e_{y_i}) \odot \phi'(x_i^{\!\top} W^{(r)}),\, \Delta_i\Bigr\rangle. \label{eq:nn-bdd-step1}
\end{align}
where we have defined $\Delta_i := W_{*}^{\!\top} x_i \in \mathbb R^K$. Expanding the inner product over classes
and using $p_i(W)_{y_i} - 1 = -\sum_{j \neq y_i} p_i(W)_j$, we have that
\begin{equation}
  \langle \nabla L(W), \overline{W}_{*}\rangle_F \;=\; \frac{1}{n}\sum_{i=1}^{n}\sum_{r=1}^{m}\sum_{j \neq y_i}\frac{a_r^{2}}{m}\, p_i(W)_j\,\bigl[\phi'(x_i^{\!\top} W^{(r)}_j)\,\Delta_i(j) - \phi'(x_i^{\!\top} W^{(r)}_{y_i})\,\Delta_i(y_i)\bigr]. \label{eq:nn-bdd-step2}
\end{equation}
We, next, bound the term $A := \phi'(x_i^{\!\top} W^{(r)}_j)\,\Delta_i(j) - \phi'(x_i^{\!\top} W^{(r)}_{y_i})\,\Delta_i(y_i)$ appearing on the right-hand side of \eqref{eq:nn-bdd-step2}. We have
\begin{eqnarray}
  A
  &=& \phi'(x_i^{\!\top} W^{(r)}_{y_i})\,\bigl(\Delta_i(j) - \Delta_i(y_i)\bigr) + \bigl(\phi'(x_i^{\!\top} W^{(r)}_j) - \phi'(x_i^{\!\top} W^{(r)}_{y_i})\bigr)\,\Delta_i(j) \nonumber \\
  &\stackrel{\text{Assumption}~\ref{assumpt: separable}}{\leq}& -\gamma \phi'(x_i^{\!\top} W^{(r)}_{y_i}) + \bigl(\phi'(x_i^{\!\top} W^{(r)}_j) - \phi'(x_i^{\!\top} W^{(r)}_{y_i})\bigr)\,\Delta_i(j) \nonumber \\
  &\stackrel{\alpha \leq \phi^{'}(\cdot)}{\leq}& -\alpha \gamma + \bigl(\phi'(x_i^{\!\top} W^{(r)}_j) - \phi'(x_i^{\!\top} W^{(r)}_{y_i})\bigr)\,\Delta_i(j). \label{eq:nn-bdd-decomp}
\end{eqnarray}
Since $\phi'(x_i^{\!\top} W^{(r)}_j),\, \phi'(x_i^{\!\top} W^{(r)}_{y_i}) \in (\alpha, 1)$, we have $\bigl|\phi'(x_i^{\!\top} W^{(r)}_j) - \phi'(x_i^{\!\top} W^{(r)}_{y_i})\bigr| \le 1 - \alpha$ and thus substituting into \eqref{eq:nn-bdd-decomp} we get
\begin{eqnarray}
    A &\leq&-\alpha \gamma + (1 - \alpha)\,\Delta_i(j) \nonumber
\end{eqnarray}
Applying Cauchy--Schwarz inequality and using that $\|W_k^*\|_2 \leq \|W^*\|_F = 1, \forall k\in[K]$ from Assumption~\ref{assumpt: separable}, we have that
\begin{equation}
  |\Delta_i(j)| \;=\; |(W_{*})_j^{\!\top} x_i| \;\le\; \|(W_{*})_j\|_2\,\|x_i\|_2 \;\le\; \|x_i\|_2 \stackrel{\|x_i\|_2 \leq 1}{\leq} 1, \label{eq:nn-bdd-Deltaj}
\end{equation}
and thus we get
\begin{eqnarray}
  A&\leq& -\alpha \gamma + (1 - \alpha) = -\tilde\gamma, \label{eq:nn-bdd-cross}
\end{eqnarray}
where $\tilde \gamma := \alpha \gamma - (1 - \alpha)$.
Substituting inequality~\eqref{eq:nn-bdd-cross} into~\eqref{eq:nn-bdd-step2}, we obtain
\begin{eqnarray}
    \langle \nabla L(W), \overline{W}_{*}\rangle_F &\leq& -\tilde\gamma \frac{1}{n}\sum_{i=1}^{n}\sum_{r=1}^{m}\sum_{j \neq y_i}\frac{a_r^{2}}{m}\, p_i(W)_j \nonumber \\
    &=& -\tilde\gamma \frac{1}{n}\sum_{i=1}^{n}\sum_{r=1}^{m}\frac{a_r^{2}}{m}\, \bigl(1 - p_i(W)_{y_i}\bigr) \nonumber \\
    &\stackrel{a_r^2 =1}{=}& -\tilde\gamma \frac{1}{n}\sum_{i=1}^{n} \bigl(1 - p_i(W)_{y_i}\bigr)\nonumber \\
    &=& -\tilde\gamma G(W). \nonumber
\end{eqnarray}
\end{proof}

\begin{lemma}[Two-layer gradient norm bounds]\label{lem:nn-grad-bound}
For every $W \in \mathbb R^{md \times K}$, it holds for every $i \in [n]$ that
\begin{align}
  \|\nabla \ell_i(W)\|_F &\;\le\; \sqrt{\frac{2}{m}}\,\bigl(1 - p_i(W)_{y_i}\bigr) \;\le\; \sqrt{\frac{2}{m}}\,\ell_i(W) \;\le\; \sqrt{\frac{2}{m}}\, n\, L(W), \label{eq:nn-grad-i}\\
  \|\nabla L(W)\|_F &\;\le\; \sqrt{\frac{2}{m}}\, G(W) \;\le\; \sqrt{\frac{2}{m}}\, L(W). \label{eq:nn-grad-full}
\end{align}
Moreover, it holds that
\begin{equation}
  \sqrt{m}\,\|\nabla L(W)\|_F \;\ge\; \tilde \gamma\, G(W). \label{eq:nn-grad-lower}
\end{equation}
\end{lemma}

\begin{proof}
By Lemma~\ref{lem:nn-jacobian}, we have that
\[
  \|\nabla \ell_i(W)\|_F \;=\; \|J_i(W)^{\!\top} (p_i(W) - e_{y_i})\|_F
   \;\le\; \|J_i(W)\|_{\mathrm{op}}\,\|p_i(W) - e_{y_i}\|_2
   \;\le\; \frac{1}{\sqrt m}\,\|p_i(W) - e_{y_i}\|_2.
\]
It holds that
\begin{align*}
  \|p_i(W) - e_{y_i}\|_2^{2}
   &\;=\; \bigl(1 - p_i(W)_{y_i}\bigr)^{2} + \sum_{j \neq y_i} p_i(W)_j^{2}\\
   &\;\le\; \bigl(1 - p_i(W)_{y_i}\bigr)^{2} + \biggl(\sum_{j \neq y_i} p_i(W)_j\biggr)^{\!2}\\
   &\;=\; 2\bigl(1 - p_i(W)_{y_i}\bigr)^{2},
\end{align*}
where we used $\sum a_j^{2} \le (\sum a_j)^{2}$ for non-negative $a_j$ and $\sum_{j \neq y_i} p_i(W)_j = 1 - p_i(W)_{y_i}$. Therefore, $\|p_i(W) - e_{y_i}\|_2 \le \sqrt 2\,(1 - p_i(W)_{y_i})$, and combining with $1 - p_i(W)_{y_i} \le \ell_i(W) \le n L(W)$ proves the inequality~\eqref{eq:nn-grad-i}.
By the triangle inequality, we have that
\[
  \|\nabla L(W)\|_F \;\le\; \frac{1}{n}\sum_{i=1}^{n}\|\nabla \ell_i(W)\|_F
   \;\le\; \sqrt{\frac{2}{m}}\cdot\frac{1}{n}\sum_{i=1}^{n}\bigl(1 - p_i(W)_{y_i}\bigr)
   \;=\; \sqrt{\frac{2}{m}}\, G(W) \leq \sqrt{\frac{2}{m}} L(W),
\]
since it holds that $G(W) \le L(W)$. Applying the Cauchy--Schwarz inequality and using the fact that $\|\overline{W}_*\|_F = \sqrt m$, we have that
\begin{eqnarray}
    \langle - \nabla L(W), \overline{W}_{*}\rangle_F \le \sqrt m\,\|\nabla L(W)\|_F
\end{eqnarray}
Using Lemma~\ref{lem:nn-perceptron}, we obtain 
\begin{eqnarray}
    \sqrt{m}\,\|\nabla L(W)\|_F \;\ge\; \tilde \gamma\, G(W) \nonumber
\end{eqnarray}
\end{proof}

\begin{lemma}[Two-layer Hessian bound]\label{lem:nn-hessian}
Under Assumption~\ref{assump: activation}, for any $W \in \mathbb R^{md \times K}$ and $i \in [n]$ it holds that
\begin{equation}
  \|\nabla^{2}\ell_i(W)\|_{\mathrm{op}} \;\le\; \frac{2(1+\tilde\beta)}{m}\,\ell_i(W),
  \qquad
  \|\nabla^{2} L(W)\|_{\mathrm{op}} \;\le\; \frac{2(1+\tilde\beta)}{m}\, L(W). \label{eq:nn-hessian}
\end{equation}
\end{lemma}

\begin{proof}
From the decomposition~\eqref{eq:nn-hess}, we have that
\begin{eqnarray}
    \nabla^{2}\ell_i(W) &\;=\;& J_i(W)^{\!\top}\bigl(\mathrm{diag}(p_i(W)) - p_i(W)p_i(W)^{\!\top}\bigr) J_i(W)\nonumber\\
  &\;+\;& \sum_{k=1}^{K}\bigl(p_i(W)_k - \mathbb 1_{\{k=y_i\}}\bigr)\,\nabla^{2} z_i(W)_k \label{eq: op_norm}
\end{eqnarray}
By the sub-additivity of the operator norm, we get
\begin{eqnarray}
    \left\|\nabla^{2}\ell_i(W)\right\|_{\mathrm{op}} &\leq& \left\|
  \|J_i(W)\|_{\mathrm{op}}^{2}\,\|\mathrm{diag}(p_i(W)) - p_i(W) p_i(W)^{\!\top}\right\|_{\mathrm{op}} \nonumber \\
  && + \left\|\sum_{k=1}^{K}\bigl(p_i(W)_k - \mathbb 1_{\{k=y_i\}}\bigr)\,\nabla^{2} z_i(W)_k\right\|_{\text{op}} \quad \quad \label{eq_pre2}
\end{eqnarray}
We, next, bound the two terms on the right-hand side of \eqref{eq_pre2}. The operator norm of the first term in \eqref{eq_pre2} is bounded by
\begin{eqnarray}
    \left\|
  \|J_i(W)\|_{\mathrm{op}}^{2}\,\|\mathrm{diag}(p_i(W)) - p_i(W) p_i(W)^{\!\top}\right\|_{\mathrm{op}}
   &\;\le\;& \frac{1}{m}\,\bigl(1 - \|p_i(W)\|_2^{2}\bigr)\nonumber \\
   &\;\le\;& \frac{2}{m}\bigl(1 - p_i(W)_{y_i}\bigr)\nonumber \\
   &\;\le\;& \frac{2}{m}\,\ell_i(W) \label{eq: op_b1}
\end{eqnarray}
where we used Lemma~\ref{lem:nn-jacobian}, $\|\mathrm{diag}(p) - p p^{\!\top}\|_{\mathrm{op}} \le 1 - \|p\|_2^{2} \le 2(1 - p(y_i))$, and $1 - p_i(W)_{y_i} \le \ell_i(W)$.
For the second term in \eqref{eq_pre2}, it holds that $\|\nabla^{2} z_i(W)_k\|_{\mathrm{op}} \le \frac{\tilde\beta}{m}$, since $\|x_i\|_2 \le 1$ and $|\phi''(\cdot)| \le \tilde\beta$. We, also, have that
\[
  \sum_{k=1}^{K}\bigl|p_i(W)_k - \mathbb 1_{\{k=y_i\}}\bigr|
   \;=\; \bigl(1 - p_i(W)_{y_i}\bigr) + \sum_{j \neq y_i} p_i(W)_j
   \;=\; 2\bigl(1 - p_i(W)_{y_i}\bigr)
   \;\le\; 2\,\ell_i(W).
\]
Thus, we get that
\begin{eqnarray}
    \left\|\sum_{k=1}^{K}\bigl(p_i(W)_k - \mathbb 1_{\{k=y_i\}}\bigr)\,\nabla^{2} z_i(W)_k\right\|_{\mathrm{op}} &\leq& \left\|\sum_{k=1}^{K}\bigl(p_i(W)_k - \mathbb 1_{\{k=y_i\}}\bigr)\right\|_{\mathrm{op}} \left\|\nabla^{2} z_i(W)_k\right\|_{\mathrm{op}} \nonumber \\
    &\leq& \frac{2\tilde\beta}{m}\,\ell_i(W)\label{eq: op_b2}.
\end{eqnarray}
Substituting inequalities \eqref{eq: op_b1}, \eqref{eq: op_b2} into \eqref{eq_pre2}, we get
\begin{eqnarray}
    \|\nabla^{2}\ell_i(W)\|_{\mathrm{op}} \;\le\; \frac{2(1+\tilde\beta)}{m}\,\ell_i(W) \nonumber
\end{eqnarray}
Summing over $i=1, ..., n$ and multiplying by $\frac{1}{n}$, we have that
\begin{eqnarray}
    \|\nabla^{2} L(W)\|_{\mathrm{op}} \;\le\; \frac{2(1+\tilde\beta)}{m}\, L(W). \nonumber
\end{eqnarray}
\end{proof}

\begin{lemma}\label{lem:nn-G-cesaro}
The iterates of \ref{SGD} for the two-layer network~\eqref{eq: 2nn} with any step size
$\tilde\eta = m\eta > 0$ and batch size $b \ge 1$ satisfy
\begin{eqnarray}
    \frac{1}{t}\sum_{k=0}^{t-1}\mathbb E[G(W_k)] &\leq& \frac{2\bigl(K - 1 + 2\ln(\tilde\gamma^{2}\eta t) + 4\kappa + 2\eta(1 + \frac{1}{b})\bigr)}{\eta\tilde\gamma^2 t} \nonumber
\end{eqnarray}
where $G(W) = \frac{1}{n}\sum_{i=1}^{n}(1 - p_i(W)_{y_i})$.
\end{lemma}

\begin{proof}
From Lemma~\ref{lem:nn-perceptron}, we have that the gradient of the cross-entropy loss satisfies $\langle \nabla L(W_t), \overline{W}_{*}\rangle_F \le -\tilde \gamma\, G(W_t)$. Using the \ref{SGD} update $W_{t+1} = W_t - \tilde\eta g_t$, we get
\[
  \langle W_{t+1}, \overline{W}_{*}\rangle_F \;=\; \langle W_t, \overline{W}_{*}\rangle_F - \tilde\eta\,\langle g_t, \overline{W}_{*}\rangle_F.
\]
Taking expectation conditional on the filtration $\mathcal F_t$ and using $\mathbb E[g_t \mid \mathcal F_t] = \nabla L(W_t)$, we have that
\[
  \mathbb E\bigl[\langle W_{t+1}, \overline{W}_{*}\rangle_F \mid \mathcal F_t\bigr]
  \;\ge\; \langle W_t, \overline{W}_{*}\rangle_F + \tilde\eta\tilde \gamma\, G(W_t).
\]
Taking expectation again and using the tower law of expectation, it holds that
\begin{eqnarray}
    \mathbb E\bigl[\langle W_{t+1}, \overline{W}_{*}\rangle_F \bigr]
  \;\ge\; \Expe{\langle W_t, \overline{W}_{*}\rangle_F} + \tilde\eta\tilde\gamma\, \Expe{G(W_t)}.
\end{eqnarray}
Unrolling the recursion and multiplying by $\frac{1}{t}$, we obtain
\[
  \frac{1}{t}\,\mathbb E\bigl[\langle W_t, \overline{W}_{*}\rangle_F\bigr]
  \;\ge\; \frac{\tilde\eta\,\tilde \gamma}{t}\sum_{k=0}^{t-1}\mathbb E[G(W_k)],
\]
where we have used the fact that $W_0 = \textbf{0}$. Since $\|\overline{W}_{*}\|_F = \sqrt m$, by Cauchy--Schwarz it holds that $\sqrt m\,\|W_t\|_F \ge \langle W_t, \overline{W}_{*}\rangle_F$, and thus we obtain
\begin{equation}
  \frac{1}{t}\sum_{k=0}^{t-1}\mathbb E[G(W_k)]
  \;\le\; \frac{\sqrt m\,\mathbb E[\|W_t\|_F]}{\tilde\eta\,\tilde \gamma\, t}. \label{eq:nn-G-cesaro-step}
\end{equation}
We, next, bound the term $\mathbb E[\|W_t\|_F]$. Applying the triangle inequality, we have that
\begin{equation}
  \|W_t\|_F \;\le\; \|W_t - U\|_F + \|U\|_F. \label{eq:nn-Wt-triangle}
\end{equation}
Thus, it suffices to bound the terms $\|W_t - U\|_F$ and $\|U\|_F$. From Lemma~\ref{lem:nn-distance-comparator}, for the decomposition $U = U_1 + U_2$ with
\[
  U_1 \;=\; \frac{\ln(\tilde\gamma^{2}\eta t)+2\kappa}{\tilde \gamma}\,\overline{W}_{*},
  \qquad
  U_2 \;=\; \frac{\eta(1 + \frac{1}{b})}{\tilde \gamma}\,\overline{W}_{*},
\]
it holds that
\[
  \frac{\mathbb E\bigl[\|W_t - U\|_F^{2}\bigr]}{2\tilde\eta\, t}
  + \frac{1}{t}\sum_{k=0}^{t-1}\mathbb E[L(W_k)]
  \;\le\; {\widehat L_W(U_1)} + \frac{\|U\|_F^{2}}{2\tilde\eta\, t},
\]
Since $L(\cdot) \geq 0$, it holds that
\[
  \mathbb E\bigl[\|W_t - U\|_F^{2}\bigr] \;\le\; 2\tilde\eta\, t\, {\widehat L_W(U_1)} + \|U\|_F^{2} \label{eq:pre4}.
\]
From {Lemma~\ref{lemma: conv_proxy_NN}}, we have that
\begin{eqnarray}
    {\widehat L_W(U_1)} \le \frac{K-1}{\tilde \gamma^{2}\eta t} \label{eq:pre5}
\end{eqnarray}
It holds that
\[
  \|U\|_F^{2} \;\le\; 2\|U_1\|_F^{2} + 2\|U_2\|_F^{2}
   \;=\; \frac{2 m\bigl(\ln(\tilde\gamma^{2}\eta t) + 2\kappa\bigr)^{2}}{\tilde\gamma^{2}}
        + \frac{2 m\eta^{2}(1 + \frac{1}{b})^{2}}{\tilde\gamma^{2}}.
\]
Using $(a + b)^{2} \le 2 a^{2} + 2 b^{2}$ for the first term, we obtain
\begin{equation}
  \|U\|_F^{2} \;\le\; \frac{4 m\,\ln^{2}(\tilde\gamma^{2}\eta t) + 16 m\,\kappa^{2} + 2 m\eta^{2}(1 + \frac{1}{b})^{2}}{\tilde\gamma^{2}}. \label{eq:nn-U-norm-sq}
\end{equation}
Applying the inequality $\|a + b\|^{2} \le 2\|a\|^{2} + 2\|b\|^{2}$, we obtain
\begin{align*}
  \mathbb E\bigl[\|W_t\|_F^{2}\bigr]
  &\;\le\; 2\,\mathbb E\bigl[\|W_t - U\|_F^{2}\bigr] + 2\,\|U\|_F^{2}\\
  &\;\le\; 4\tilde\eta\, t\, {\widehat L_W(U_1)} + 4\|U\|_F^{2}\\
  &\;\le\; \frac{4m(K-1)}{\tilde\gamma^{2}} + \frac{16 m\,\ln^{2}(\tilde\gamma^{2}\eta t) + 64 m\,\kappa^{2} + 8 m\eta^{2}(1 + \frac{1}{b})^{2}}{\tilde\gamma^{2}} \\
    &\;\le\; \frac{4 m (K-1) + 16 m\,\ln^{2}(\tilde\gamma^{2}\eta t) + 64 m\,\kappa^{2} + 8 m\eta^{2}(1 + \frac{1}{b})^{2}}{\tilde\gamma^{2}}\label{eq:pre5}.
\end{align*}
Using Jensen's inequality and $\sqrt{a + b + c + d} \le \sqrt a + \sqrt b + \sqrt c + \sqrt d$ for $a, b, c, d \geq 0$, we have that
\begin{eqnarray}
  \mathbb E\bigl[\|W_t\|_F\bigr] &\leq& \sqrt{\Expe{\|W_t\|_F^2}} \nonumber \\
  &\leq& \frac{2\sqrt m\bigl(\sqrt{K - 1} + 2\ln(\tilde\gamma^{2}\eta t) + 4\kappa + 2\eta(1 + \frac{1}{b})\bigr)}{\tilde\gamma}. \label{eq:nn-Wt-norm-bound}
\end{eqnarray}
Substituting~\eqref{eq:nn-Wt-norm-bound} into~\eqref{eq:nn-G-cesaro-step} and using the fact that $\tilde\eta = m\eta$ and $\sqrt{K-1} \le K - 1$ (for $K \ge 2$), we obtain
\begin{eqnarray}
    \frac{1}{t}\sum_{k=0}^{t-1}\mathbb E[G(W_k)] &\leq& \frac{2\bigl(K - 1 + 2\ln(\tilde\gamma^{2}\eta t) + 4\kappa + 2\eta(1 + \frac{1}{b})\bigr)}{\eta\tilde\gamma^2 t} \nonumber
\end{eqnarray}
\end{proof}

\begin{lemma}\label{lem:nn-LU1-bound}
Let Assumptions~\ref{assumpt: separable}, \ref{assump: activation} hold. For $U_1 = \frac{\ln(\tilde\gamma^{2}\eta t)+2\kappa}{\tilde\gamma}\,\overline{W}_{*}$ and any $t \ge 1$, we have that
\begin{equation}
  L(U_1) \;\le\; F(U_1) \;\le\; \frac{K - 1}{\tilde\gamma^{2}\eta t}. \label{eq:nn-LU1}
\end{equation}
\end{lemma}

\begin{proof}
Fix \(i\in[n]\) and \(j\neq y_i\). For every hidden unit \(r\in[m]\) it holds that $U_1^{(r)}=a_r\frac{\ln\!\big(\tilde\gamma^2\eta t\big)+2\kappa}{\tilde\gamma} W^\ast$. Letting $D = \frac{\ln\!\big(\tilde\gamma^2\eta t\big)+2\kappa}{\tilde\gamma}$ for brevity, we have
\begin{eqnarray}
(x_i^\top U_1^{(r)})_{y_i}
&=&
D\,a_r\,x_i^\top W^\ast_{y_i},
\\
(x_i^\top U_1^{(r)})_{j}
&=&
D\,a_r\,x_i^\top W^\ast_{j}.
\end{eqnarray}
Thus, we have that
\begin{eqnarray}
z_i(U)_{y_i}-z_i(U)_j
&=&
\frac1m\sum_{r=1}^m a_r
\Big[
\phi\!\big(D\,a_r\,x_i^\top W^\ast_{y_i}\big)
-
\phi\!\big(D\,a_r\,x_i^\top W^\ast_{j}\big)
\Big].
\label{eq:two-layer-margin-U}
\end{eqnarray}
We, now, use the near-homogeneity condition
\[
|\phi(z)-\phi'(z)z|\leq \kappa.
\]
For any \(u\in\R\), it holds that
\begin{eqnarray}
\phi(u)
&\geq&
\phi'(u)\,u-\kappa, \label{eq:nn_eq3}
\\
\phi(u)
&\leq&
\phi'(u)\,u+\kappa\label{eq:nn_eq4}.
\end{eqnarray}
{Applying \eqref{eq:nn_eq3}, \eqref{eq:nn_eq4} to the two terms on the right-hand side of \eqref{eq:two-layer-margin-U}, and using $|a_r| = 1$, we get for $u_r := D\,a_r\,x_i^\top W^\ast_{y_i}$ and $v_r := D\,a_r\,x_i^\top W^\ast_{j}$ that
\begin{eqnarray}
a_r\big[\phi(u_r)-\phi(v_r)\big]
&\geq&
a_r\big[\phi'(u_r)\,u_r-\phi'(v_r)\,v_r\big]
-\big|\phi(u_r)-\phi'(u_r)u_r\big|
-\big|\phi(v_r)-\phi'(v_r)v_r\big|
\nonumber\\
&\geq&
D\Big[
\phi'(u_r)\,x_i^\top W^\ast_{y_i}
-
\phi'(v_r)\,x_i^\top W^\ast_{j}
\Big]
-2\kappa,
\label{eq:two-layer-margin-lb-1}
\end{eqnarray}
where we used $a_r^2 = 1$. Let $A_i := x_i^\top W^\ast_{y_i}$ and $B_{ij} := x_i^\top W^\ast_{j}$. By separability, $A_i - B_{ij} \geq \gamma$, and since $\|x_i\|_2 \leq 1$ and $\|W^\ast\|_F = 1$, it holds that $|B_{ij}| \leq 1$. Using the fact that \(\phi'(\cdot)\in[\alpha,1]\) from Assumption~\ref{assump: activation}, it follows that
\begin{eqnarray}
\phi'(u_r)\,A_i - \phi'(v_r)\,B_{ij}
&=&
\phi'(u_r)\,(A_i - B_{ij}) + \big(\phi'(u_r) - \phi'(v_r)\big)\,B_{ij}
\nonumber\\
&\geq&
\alpha\gamma - (1-\alpha)
\;=\;
\tilde\gamma,
\label{eq:two-layer-margin-lb-2}
\end{eqnarray}
since $|\phi'(u_r) - \phi'(v_r)| \leq 1-\alpha$. Substituting \eqref{eq:two-layer-margin-lb-2} into \eqref{eq:two-layer-margin-lb-1} and averaging over $r \in [m]$ in \eqref{eq:two-layer-margin-U}, we obtain
\begin{eqnarray}
z_i(U)_{y_i}-z_i(U)_j
&\geq&
\tilde\gamma D
-2\kappa
\;=\;
\ln\!\big(\tilde\gamma^2\eta t\big)
\label{eq:two-layer-margin-lb-final}
\end{eqnarray}}
Thus, we get
\begin{eqnarray}
\exp\!\Big(-(z_i(U)_{y_i}-z_i(U)_j)\Big)
\leq
{\exp\!\big(-\tilde\gamma D+2\kappa\big)
=
\frac{1}{\tilde\gamma^2\,\eta t}}.
\end{eqnarray}
{Summing over all \(j\neq y_i\) and averaging over \(i\in[n]\),} it follows that
\begin{eqnarray}
F(U)
&\leq&
\frac{K-1}{\tilde\gamma^2\,\eta t}.
\end{eqnarray}
Since \(L(U)\leq F(U)\), we obtain that
\begin{eqnarray}
L(U)
&\leq&
F(U)
\leq
\frac{K-1}{\tilde\gamma^2\,\eta t}.
\end{eqnarray}
\end{proof}
\begin{lemma}
\label{lem:two-layer-comparator-loss-bound-strengthened}
Let Assumptions~\ref{assumpt: separable} and~\ref{assump: activation} hold, and let
$U = W_{t_{1}} + U_1$ with
$U_1 \;=\; \frac{\ln(\tilde\gamma^{2}\,\eta(t - t_1)) + 2\kappa}{\tilde\gamma}\,\overline{W}_*$.
Then it holds that
\begin{equation}
  L(U) \;\le\; F(U) \;\le\; \frac{F(W_{t_1})}{\tilde\gamma^{2}\,\eta\,(t - t_1)}. \label{eq: two-layer-comparator-loss-bound-strengthened}
\end{equation}
\end{lemma}

\begin{proof}
Let $C \;:=\; \frac{\ln(\tilde\gamma^{2}\,\eta(t - t_1)) + 2\kappa}{\tilde\gamma}$, so that $U_1 = C\,\overline{W}_*$ and $C\,\tilde\gamma = \ln(\tilde\gamma^{2}\,\eta(t - t_1)) + 2\kappa$. Since
$\overline{W}_* = (a_1 W_*, \ldots, a_m W_*)$, the $r$-th block of $U$ satisfies
\begin{equation}
  U^{(r)}_k \;=\; (W_{t_1})^{(r)}_k + a_r\, C\, (W_*)_k,
  \qquad \forall r \in [m], k \in [K]. \label{eq:nn-comp-Ublock}
\end{equation}

For every sample $i \in [n]$ and class $j \neq y_i$, let $M_{ij}(W) := z_i(W)_{y_i} - z_i(W)_j$.
We compute the margin increment
\begin{equation}
  \Delta M_{ij} \;:=\; M_{ij}(U) - M_{ij}(W_{t_1})
   \;=\; \frac{1}{m}\sum_{r=1}^{m} a_r\,\bigl[D^{r}_{i, y_i} - D^{r}_{i, j}\bigr], \label{eq:nn-comp-margin-shift}
\end{equation}
where for brevity we let $s^{r}_{i,k} := x_i^{\!\top}(W_{t_1})^{(r)}_k, \Delta_i(k) := (W_*\, x_i)_k$ and
\[
  D^{r}_{i,k} \;:=\; \phi\bigl(s^{r}_{i,k} + a_r\, C\,\Delta_i(k)\bigr) - \phi\bigl(s^{r}_{i,k}\bigr).
\]

\textbf{Mean-value theorem.}\
Since $\phi$ is continuously differentiable (Assumption~\ref{assump: activation}), by the mean-value theorem,
for each $r \in [m]$ and $k \in \{y_i, j\}$, there exists a point $\xi^{r}_{i,k}\in \left(s^{r}_{i,k}, s^{r}_{i,k} + a_r\, C\,\Delta_i(k)\right)$ such that
\begin{equation}
  D^{r}_{i,k} \;=\; \phi'\bigl(\xi^{r}_{i,k}\bigr) \cdot a_r\, C\,\Delta_i(k). \label{eq:nn-comp-MVT}
\end{equation}
Substituting~\eqref{eq:nn-comp-MVT} into~\eqref{eq:nn-comp-margin-shift} and using $a_r^{2} = 1$, we have that
\begin{equation}
  \Delta M_{ij}
   \;=\; \frac{C}{m}\sum_{r=1}^{m}\bigl[\phi'(\xi^{r}_{i,y_i})\,\Delta_i(y_i) - \phi'(\xi^{r}_{i,j})\,\Delta_i(j)\bigr]
   \;=\; C\bigl[\bar\phi'_{i,y_i}\,\Delta_i(y_i) - \bar\phi'_{i,j}\,\Delta_i(j)\bigr], \label{eq:nn-comp-margin-shift-2}
\end{equation}
where $\bar\phi'_{i,k} := \tfrac{1}{m}\sum_{r=1}^{m}\phi'(\xi^{r}_{i,k}) \in [\alpha, 1]$ by
Assumption~\ref{assump: activation}~(Item~\ref{assump: activation:grad}). Adding and subtracting $\bar\phi'_{i,y_i}\,\Delta_i(j)$, we have that 
\begin{eqnarray}
  \bar\phi'_{i,y_i}\,\Delta_i(y_i) - \bar\phi'_{i,j}\,\Delta_i(j)
  &\;=\;& \bar\phi'_{i,y_i}\,\bigl(\Delta_i(y_i) - \Delta_i(j)\bigr)
       + \bigl(\bar\phi'_{i,y_i} - \bar\phi'_{i,j}\bigr)\,\Delta_i(j) \nonumber \\
   &\geq& \alpha \gamma + \bigl(\bar\phi'_{i,y_i} - \bar\phi'_{i,j}\bigr)\,\Delta_i(j),    \label{eq:nn-comp-decomp}
\end{eqnarray}
where at the last step we have used that $\bar\phi'_{i,y_i} \ge \alpha > 0$ and $\Delta_i(y_i) - \Delta_i(j) \ge \gamma > 0$ from
Assumption~\ref{assumpt: separable}.
Since $\bar\phi'_{i,y_i},\, \bar\phi'_{i,j} \in [\alpha, 1]$, we have
$|\bar\phi'_{i,y_i} - \bar\phi'_{i,j}| \le 1 - \alpha$. By the Cauchy--Schwarz inequality,
$\|W_*\|_F = 1$, and $\|x_i\|_2 \le 1$ from Assumption~\ref{assumpt: separable}, it holds that
\[
  |\Delta_i(j)| \;=\; |(W_*)_j^{\!\top} x_i| \;\le\; \|(W_*)_j\|_2\,\|x_i\|_2 \;\le\; \|W_*\|_F\,\|x_i\|_2 \;\le\; 1.
\]
Combining these two bounds, we get
\begin{equation}
  \bigl|\bigl(\bar\phi'_{i,y_i} - \bar\phi'_{i,j}\bigr)\,\Delta_i(j)\bigr| \;\le\; 1 - \alpha. \label{eq:nn-comp-cross}
\end{equation}
Substituting~\eqref{eq:nn-comp-cross} into~\eqref{eq:nn-comp-decomp}, we have that
\begin{equation}
  \bar\phi'_{i,y_i}\,\Delta_i(y_i) - \bar\phi'_{i,j}\,\Delta_i(j)
   \;\ge\; \alpha\gamma - (1 - \alpha) \;=\; \tilde\gamma. \label{eq:nn-comp-piece}
\end{equation}
Using~\eqref{eq:nn-comp-piece} with~\eqref{eq:nn-comp-margin-shift-2}, we obtain
\begin{equation}
  \Delta M_{ij} \;\ge\; C\,\tilde\gamma \;=\; \ln\!\bigl(\tilde\gamma^{2}\,\eta(t - t_1)\bigr) + 2\kappa,
   \label{eq:nn-comp-shift-final}
\end{equation}
Since $M_{ij}(U) = M_{ij}(W_{t_1}) + \Delta M_{ij}$, we have that
\begin{align}
  \exp\!\bigl(-M_{ij}(U)\bigr)
  &\;=\; \exp\!\bigl(-M_{ij}(W_{t_1})\bigr) \cdot \exp\!\bigl(-\Delta M_{ij}\bigr) \notag\\
  &\;\stackrel{\eqref{eq:nn-comp-shift-final}}{\le}\;
   \exp\!\bigl(-M_{ij}(W_{t_1})\bigr) \cdot \exp\!\Bigl(-\ln\!\bigl(\tilde\gamma^{2}\,\eta(t - t_1)\bigr) - 2\kappa\Bigr) \notag\\
  &\;=\; \frac{e^{-2\kappa}\,\exp(-M_{ij}(W_{t_1}))}{\tilde\gamma^{2}\,\eta(t - t_1)}\nonumber \\
   &\;\le\; \frac{\exp(-M_{ij}(W_{t_1}))}{\tilde\gamma^{2}\,\eta(t - t_1)}, \label{eq:nn-comp-exp-bound}
\end{align}
where the last inequality uses $e^{-2\kappa} \le 1$ since $\kappa \ge 0$.
Summing~\eqref{eq:nn-comp-exp-bound} over $i \in [n],$ multiplying by $\frac{1}{n}$ and summing again over $j \neq y_i$, we get
\begin{eqnarray}
   F(U) &=& \frac{1}{n}\sum_{i=1}^{n}\sum_{j \neq y_i}\exp\!\bigl(-M_{ij}(U)\bigr)\nonumber \\
   &\le& \frac{1}{\tilde\gamma^{2}\,\eta(t - t_1)} \cdot \frac{1}{n}\sum_{i=1}^{n}\sum_{j \neq y_i}\exp\!\bigl(-M_{ij}(W_{t_1})\bigr)\nonumber \\
   &=& \frac{F(W_{t_1})}{\tilde\gamma^{2}\,\eta(t - t_1)}\nonumber. 
\end{eqnarray}
\end{proof}

\begin{lemma}[Two-layer squared-gradient comparison]
\label{lem:two-layer-squared-grad-comparison}
Let Assumptions~\ref{assumpt: separable}, \ref{assump: activation} hold. Then, for every \(W\in \R^{md\times K}\), it holds that
\begin{eqnarray}
\sum_{i=1}^n \|\nabla \ell_i(W)\|_F^2
&\leq&
\frac{2n^2}{\tilde\gamma^2}\,\|\nabla L(W)\|_F^2.
\label{eq:two-layer-squared-grad-comparison}
\end{eqnarray}
\end{lemma}

\begin{proof}
We have that
\begin{eqnarray}
\nabla \ell_i(W)
&=&
J_i(W)^\top\bigl(p_i(W)-e_{y_i}\bigr),
\end{eqnarray}
where \(J_i(W)\) denotes the Jacobian of the logits \(z_i(W)\) with respect to \(W\).
From Lemma~\ref{lem:nn-jacobian}, we have
\begin{eqnarray}
\|J_i(W)\|_{op}
&\leq&
\frac{1}{\sqrt m}.
\end{eqnarray}
Thus, it holds that
\begin{eqnarray}
\|\nabla \ell_i(W)\|_F
&\leq&
\|J_i(W)\|_{op}\,\|p_i(W)-e_{y_i}\|_2
\nonumber\\
&\leq&
\frac{1}{\sqrt m}\,\|p_i(W)-e_{y_i}\|_2.
\end{eqnarray}
Taking the square on both sides, we obtain
\begin{eqnarray}
\|\nabla \ell_i(W)\|_F^2
&\leq&
\frac{1}{m}\,\|p_i(W)-e_{y_i}\|_2^2.
\label{eq:sample-grad-two-layer-step1}
\end{eqnarray}
We have that
\begin{eqnarray}
\|p_i(W)-e_{y_i}\|_2^2
&=&
\bigl(1-p_i(W)_{y_i}\bigr)^2+\sum_{j\neq y_i} p_i(W)_j^2
\nonumber\\
&\leq&
2\bigl(1-p_i(W)_{y_i}\bigr)^2,\label{eq:nn1}
\end{eqnarray}
since \(\sum_{j\neq y_i}p_i(W)_j = 1-p_i(W)_{y_i}\). Substituting \eqref{eq:nn1} into
\eqref{eq:sample-grad-two-layer-step1}, we obtain
\begin{eqnarray}
\|\nabla \ell_i(W)\|_F^2
&\leq&
\frac{2}{m}\bigl(1-p_i(W)_{y_i}\bigr)^2.
\end{eqnarray}
Summing over \(i=1,\dots,n\), it follows that
\begin{eqnarray}
\sum_{i=1}^n \|\nabla \ell_i(W)\|_F^2
&\leq&
\frac{2}{m}\sum_{i=1}^n \bigl(1-p_i(W)_{y_i}\bigr)^2.
\label{eq:sample-grad-two-layer-step2}
\end{eqnarray}
Since \(0\leq 1-p_i(W)_{y_i}\leq 1\), we have
\begin{eqnarray}
\sum_{i=1}^n \bigl(1-p_i(W)_{y_i}\bigr)^2
&\leq&
\left(\sum_{i=1}^n \bigl(1-p_i(W)_{y_i}\bigr)\right)^2
= n^2 G(W)^2.
\label{eq:sample-grad-two-layer-step3}
\end{eqnarray}
Combining \eqref{eq:sample-grad-two-layer-step2} and
\eqref{eq:sample-grad-two-layer-step3}, we get
\begin{eqnarray}
\sum_{i=1}^n \|\nabla \ell_i(W)\|_F^2
&\leq&
\frac{2n^2}{m}\,G(W)^2.
\label{eq:sample-grad-two-layer-step4}
\end{eqnarray}
On the other hand, by Lemma~\ref{lem:nn-grad-bound}, we have that
\begin{eqnarray}
\sqrt m\,\|\nabla L(W)\|_F
&\geq&
\tilde\gamma\,G(W),
\end{eqnarray}
and therefore
\begin{eqnarray}
G(W)^2
&\leq&
\frac{m}{\tilde\gamma^2}\,\|\nabla L(W)\|_F^2 .
\label{eq:sample-grad-two-layer-step5}
\end{eqnarray}
Substituting \eqref{eq:sample-grad-two-layer-step5} into
\eqref{eq:sample-grad-two-layer-step4}, we conclude that
\begin{eqnarray}
\sum_{i=1}^n \|\nabla \ell_i(W)\|_F^2
&\leq&
\frac{2n^2}{\tilde\gamma^2}\,\|\nabla L(W)\|_F^2.
\end{eqnarray}
\end{proof}

\newpage
\subsection*{Variance Bound for Two-layer NN}\label{app:twolayer:variance}
\begin{lemma}[Variance bound]\label{lemma:variance_two_layer}
Let \(g(W) = \frac{1}{b}\sum_{i\in B_t}\nabla\ell(f(x_i;W),y_i)\) denote the minibatch gradient for the two-layer model $f$ of \eqref{eq: 2nn}. Then, it holds that
\begin{equation}
    \mathbb E\bigl[\|g_t - \nabla L(W_t)\|_F^{2} \mid \mathcal F_t\bigr] \;\le\; \frac{2}{m b}\, G(W_t),\label{eq:variance-bound}
\end{equation}
where $G(W) = \frac{1}{n}\sum_{i=1}^{n}(1 - p_i(W)_{y_i})$.
\end{lemma}
\begin{proof}
We begin by decomposing the variance of the stochastic oracle. For uniform sampling of a minibatch $B_t$ of size $b \ge 1$, conditionally on the filtration $\mathcal F_t$, we have that
\begin{equation}
  \mathbb E\bigl[\|g_t - \nabla L(W_t)\|_F^{2} \mid \mathcal F_t\bigr]
   \;=\; \frac{1}{b}\,\mathrm{Var}_i\bigl(\nabla \ell_i(W_t)\bigr)
   \;\le\; \frac{1}{b}\,\mathbb E_i\bigl[\|\nabla \ell_i(W_t)\|_F^{2}\bigr], \label{eq:nn-var-step}
\end{equation}
where we used $\mathrm{Var}(Z) \le \mathbb E[\|Z\|^{2}]$ for any random variable $Z$. By Lemma~\ref{lem:nn-grad-bound}, we have that
\[
  \|\nabla \ell_i(W_t)\|_F^{2}
   \;\le\; \frac{2}{m}\bigl(1 - p_i(W_t)_{y_i}\bigr)^{2}
   \;\le\; \frac{2}{m}\bigl(1 - p_i(W_t)_{y_i}\bigr).
\]
Summing for $i=1, ..., n$ and dividing by $n$, we get
\[
  \mathbb E_i\bigl[\|\nabla \ell_i(W_t)\|_F^{2}\bigr] \;\le\; \frac{2}{m}\, G(W_t),
\]
and substituting into~\eqref{eq:nn-var-step}, we obtain
\[
  \mathbb E\bigl[\|g_t - \nabla L(W_t)\|_F^{2} \mid \mathcal F_t\bigr] \;\le\; \frac{2}{m b}\, G(W_t).
\]
\end{proof}
{
For an arbitrary
$U\in\mathbb R^{md\times K}$, define the linearized logits
\begin{equation}
    \widehat z_i(W;U)
    :=
    z_i(W)+J_i(W)(U-W),
    \qquad i\in[n],
    \label{eq:nn-linearized-logits}
\end{equation}
and the corresponding linearized loss
\begin{equation}
    \widehat L_W(U)
    :=
    \frac1n\sum_{i=1}^n
    \left[
    \log\left(
    \sum_{k=1}^K
    \exp\bigl(\widehat z_i(W;U)_k\bigr)
    \right)
    -
    \widehat z_i(W;U)_{y_i}
    \right].
    \label{eq:nn-linearized-loss}
\end{equation}
\begin{lemma}
[Linearized-comparator bound for the cross-entropy loss]
\label{lemma: conv_proxy_NN}
Let Assumptions~\ref{assumpt: separable} and
\ref{assump: activation} hold. 
Then, it holds that
\begin{equation}
    \left\langle
    \nabla L(W),U-W
    \right\rangle_F
    \leq
    \widehat L_W(U)-L(W).
    \label{eq:nn-I1-arbitrary}
\end{equation}
Moreover, if $U=\frac{
    \ln(\widetilde\gamma^2\eta T)+2\kappa
    }{\widetilde\gamma}\overline W_*$, then for every
$W\in\mathbb R^{md\times K}$, it holds that $\widehat L_W(U) \leq \frac{K-1}{\widetilde\gamma^2\eta T}$ and 
\begin{equation}
    \left\langle
    \nabla L(W),U-W
    \right\rangle_F
    \leq
    \frac{K-1}{\widetilde\gamma^2\eta T}
    -
    L(W).
    \label{eq:nn-I1-final}
\end{equation}
\end{lemma}

\begin{proof}
We first establish the comparison with the linearized loss for an
arbitrary $U$. We then show that the comparator $U:=\frac{
    \ln(\widetilde\gamma^2\eta T)+2\kappa
    }{\widetilde\gamma}\overline W_*$ gives a uniform upper bound on this loss.

For every $i\in[n]$, the chain rule and~\eqref{eq:nn-grad}
give
\begin{eqnarray}
    \left\langle
    \nabla\ell_i(W),U-W
    \right\rangle_F
    &=&
    \left\langle
    J_i(W)^\top
    \bigl(p_i(W)-e_{y_i}\bigr),
    U-W
    \right\rangle_F
    \nonumber\\
    &=&
    \left\langle
    p_i(W)-e_{y_i},
    J_i(W)(U-W)
    \right\rangle
    \nonumber\\
    &=&
    \left\langle
    p_i(W)-e_{y_i},
    \widehat z_i(W;U)-z_i(W)
    \right\rangle.
    \label{eq:nn-I1-chain-rule}
\end{eqnarray}
The cross-entropy loss is convex with respect to its logit
argument. Therefore, it holds
\begin{eqnarray}
    \left\langle
    p_i(W)-e_{y_i},
    \widehat z_i(W;U)-z_i(W)
    \right\rangle
    &\leq&
    \ell\bigl(\widehat z_i(W;U),y_i\bigr)
    -
    \ell\bigl(z_i(W),y_i\bigr).
    \label{eq:nn-I1-logit-convexity}
\end{eqnarray}
Averaging~\eqref{eq:nn-I1-logit-convexity} over
$i=1,\ldots,n$, we obtain
\begin{eqnarray}
    \left\langle
    \nabla L(W),U-W
    \right\rangle_F
    &\leq&
    \widehat L_W(U)-L(W),
    \nonumber
\end{eqnarray}
which proves~\eqref{eq:nn-I1-arbitrary}.

We next specialize the result to $U=\frac{
    \ln(\widetilde\gamma^2\eta T)+2\kappa
    }{\widetilde\gamma}\overline W_*$.
For brevity, define $s_{i,r,k}:=x_i^\top W_k^{(r)}$ and $\Delta_i:=W_*^\top x_i$.
By the same argument used in the proof of
Lemma~\ref{lem:nn-perceptron}, for every $i\in[n]$,
$j\neq y_i$, and $r\in[m]$, we have
\begin{equation}
    \phi'(s_{i,r,y_i})\Delta_i(y_i)
    -
    \phi'(s_{i,r,j})\Delta_i(j)
    \geq
    \widetilde\gamma.
    \label{eq:nn-linearized-margin-direction}
\end{equation}

Define $\theta_{i,k}(W)
    :=
    z_i(W)_k-\bigl(J_i(W)W\bigr)_k$.
Using the definition of the logits and the Jacobian, we have
\begin{eqnarray}
    \theta_{i,k}(W)
    &=&
    \frac1m\sum_{r=1}^m
    a_r
    \left[
    \phi(s_{i,r,k})
    -
    \phi'(s_{i,r,k})s_{i,r,k}
    \right].
    \nonumber
\end{eqnarray}
Since $|a_r|=1$ and
$|\phi(s)-s\phi'(s)|\leq\kappa$, it follows that
\begin{equation}
    |\theta_{i,k}(W)|\leq\kappa.
    \label{eq:nn-theta-I1-bound}
\end{equation}

From~\eqref{eq:nn-linearized-logits}, we have
\begin{eqnarray}
    \widehat z_i(W;U)_k
    &=&
    \bigl(J_i(W)U\bigr)_k
    +
    \theta_{i,k}(W).
    \label{eq:nn-linearized-logit-decomp}
\end{eqnarray}
Moreover, since
$(U)^{(r)}=a_r\frac{
    \ln(\widetilde\gamma^2\eta T)+2\kappa
    }{\widetilde\gamma}W_*$, we obtain
\begin{eqnarray}
    \bigl(J_i(W)U\bigr)_k
    &=&
    \frac1m
    \sum_{r=1}^m
    a_r\phi'(s_{i,r,k})
    x_i^\top(U)_k^{(r)}
    \nonumber\\
    &=&
    \frac{
    \ln(\widetilde\gamma^2\eta T)+2\kappa
    }{m \widetilde\gamma}
    \sum_{r=1}^m
    a_r^2\phi'(s_{i,r,k})
    \Delta_i(k)
    \nonumber\\
    &=&
    \frac{
    \ln(\widetilde\gamma^2\eta T)+2\kappa
    }{m \widetilde\gamma}
    \sum_{r=1}^m
    \phi'(s_{i,r,k})
    \Delta_i(k).
    \label{eq:nn-JU1-I1}
\end{eqnarray}
Combining~\eqref{eq:nn-linearized-logit-decomp},
\eqref{eq:nn-JU1-I1}, and
\eqref{eq:nn-linearized-margin-direction}, we have that, for
every $j\neq y_i$,
\begin{eqnarray}
    \widehat z_i(W;U)_{y_i}
    -
    \widehat z_i(W;U)_j
    &=&
    \frac{
    \ln(\widetilde\gamma^2\eta T)+2\kappa
    }{m \widetilde\gamma}
    \sum_{r=1}^m
    \left[
    \phi'(s_{i,r,y_i})\Delta_i(y_i)
    -
    \phi'(s_{i,r,j})\Delta_i(j)
    \right]
    \nonumber\\
    &&
    +
    \theta_{i,y_i}(W)-\theta_{i,j}(W)
    \nonumber\\
    &\geq&
    \ln(\widetilde\gamma^2\eta T)+2\kappa -2\kappa
    \nonumber\\
    &\geq&
    \ln(\widetilde\gamma^2\eta T).
    \label{eq:nn-linearized-margin-final}
\end{eqnarray}

Using~\eqref{eq:nn-linearized-margin-final}, we obtain
\begin{eqnarray}
    \ell\bigl(
    \widehat z_i(W;U),y_i
    \bigr)
    &=&
    \log\left(
    1+
    \sum_{j\neq y_i}
    \exp\left(
    -\left[
    \widehat z_i(W;U)_{y_i}
    -
    \widehat z_i(W;U)_j
    \right]
    \right)
    \right)
    \nonumber\\
    &\leq&
    \log\left(
    1+\frac{K-1}{\widetilde\gamma^2\eta T}
    \right)
    \nonumber\\
    &\leq&
    \frac{K-1}{\widetilde\gamma^2\eta T},
    \label{eq:nn-linearized-loss-final}
\end{eqnarray}
where the last inequality follows from
$\log(1+x)\leq x$ for every $x\geq0$.
Averaging~\eqref{eq:nn-linearized-loss-final} over
$i=1,\ldots,n$ and combining this bound with~\eqref{eq:nn-I1-arbitrary} proves
\eqref{eq:nn-I1-final}.
\end{proof}}
\subsection{Proof of Lemma~\ref{lem:leaky-activations} for Leaky Activations}
\label{app: lemma:leaky_activations}
\begin{proof}
We verify Assumption~\ref{assump: activation} for the activation
functions stated in the lemma.

{
For $\phi\in\{\mathrm{GELU},\mathrm{Softplus},\mathrm{SiLU}\}$,
define
\begin{eqnarray}
    d
    :=
    \frac{1+2c}{3},
    \qquad
    \widetilde\phi(x)
    :=
    d x+\frac{1-d}{4}\phi(x).
    \label{eq:corrected-leaky-template}
\end{eqnarray}
For
$\phi\in\{\mathrm{Huberized\ ReLU},\tanh,\sigma\}$,
we retain the original definition
\begin{eqnarray}
    \widetilde\phi(x)
    :=
    c x+(1-c)\phi(x).
    \label{eq:standard-leaky-template}
\end{eqnarray}
}

\textbf{Continuous differentiability.}
By Example~2.1 of \citet{cai2024large}, each base activation
\[
    \phi
    \in
    \{
    \mathrm{GELU},
    \mathrm{Softplus},
    \mathrm{SiLU},
    \tanh,
    \sigma
    \}
\]
is continuously differentiable on $\mathbb R$.
The Huberized ReLU is continuously differentiable at the
breakpoints $0$ and $h$ by direct verification.
Therefore, every activation $\widetilde\phi$ defined above is
continuously differentiable on $\mathbb R$.

\textbf{Derivative condition.}
{
We first consider
$\phi\in\{\mathrm{GELU},\mathrm{Softplus},\mathrm{SiLU}\}$.
Differentiating~\eqref{eq:corrected-leaky-template}, we obtain
\begin{eqnarray}
    \widetilde\phi'(x)
    =
    d+\frac{1-d}{4}\phi'(x).
    \label{eq:corrected-leaky-derivative}
\end{eqnarray}
By Example~2.1 and Appendix~F.2 of
\citet{cai2024large}, these base activations satisfy
\begin{eqnarray}
    |\phi'(x)|
    \leq
    2,
    \qquad
    x\in\mathbb R.
    \label{eq:base-derivative-bound}
\end{eqnarray}
Consequently,
\begin{eqnarray}
    d-\frac{1-d}{2}
    \leq
    \widetilde\phi'(x)
    \leq
    d+\frac{1-d}{2}.
    \nonumber
\end{eqnarray}
Since $d=(1+2c)/3$, we have
\begin{eqnarray}
    d-\frac{1-d}{2}
    &=&
    c,
    \nonumber\\
    d+\frac{1-d}{2}
    &=&
    \frac{2+c}{3}
    \leq
    1.
    \nonumber
\end{eqnarray}
Therefore,
\begin{eqnarray}
    c
    \leq
    \widetilde\phi'(x)
    \leq
    1,
    \qquad
    x\in\mathbb R.
    \label{eq:corrected-first-group-derivative}
\end{eqnarray}

We next consider the Huberized ReLU, $\tanh$, and sigmoid.
For these activations,
\begin{eqnarray}
    0
    \leq
    \phi'(x)
    \leq
    1,
    \qquad
    x\in\mathbb R.
    \nonumber
\end{eqnarray}
Differentiating~\eqref{eq:standard-leaky-template}, we obtain
\begin{eqnarray}
    \widetilde\phi'(x)
    =
    c+(1-c)\phi'(x),
    \nonumber
\end{eqnarray}
and hence
\begin{eqnarray}
    c
    \leq
    \widetilde\phi'(x)
    \leq
    1.
    \label{eq:corrected-second-group-derivative}
\end{eqnarray}
}

Setting $\alpha:=c$, inequalities
\eqref{eq:corrected-first-group-derivative} and
\eqref{eq:corrected-second-group-derivative} give
\begin{eqnarray}
    \alpha
    \leq
    \widetilde\phi'(x)
    \leq
    1,
    \qquad
    x\in\mathbb R.
    \nonumber
\end{eqnarray}
Moreover,
$\alpha=c>1/(1+\gamma)$ by assumption.
This verifies the derivative condition.

\textbf{Smoothness.}
{
Let $\lambda_\phi$ denote the coefficient multiplying the base
activation:
\begin{eqnarray}
    \lambda_\phi
    :=
    \begin{cases}
    \dfrac{1-d}{4}
    =
    \dfrac{1-c}{6},
    &
    \phi\in
    \{\mathrm{GELU},\mathrm{Softplus},\mathrm{SiLU}\},
    \\[2mm]
    1-c,
    &
    \phi\in
    \{\mathrm{Huberized\ ReLU},\tanh,\sigma\}.
    \end{cases}
    \nonumber
\end{eqnarray}
For every $x,y\in\mathbb R$, we have
\begin{eqnarray}
    |\widetilde\phi'(x)-\widetilde\phi'(y)|
    &=&
    \lambda_\phi
    |\phi'(x)-\phi'(y)|
    \nonumber\\
    &\leq&
    \lambda_\phi L_\phi |x-y|,
    \label{eq:corrected-leaky-smoothness}
\end{eqnarray}
where $L_\phi$ is a Lipschitz constant of $\phi'$. For GELU, Softplus, and SiLU, we use respectively
\[
    L_{\mathrm{GELU}}\leq2,
    \qquad
    L_{\mathrm{Softplus}}\leq1,
    \qquad
    L_{\mathrm{SiLU}}\leq4.
\]
Thus, we have that
\begin{eqnarray}
    \lambda_\phi L_\phi
    \leq
    \frac{1-c}{6}\,4
    =
    \frac{2(1-c)}{3}
    \leq
    4(1-c).
    \nonumber
\end{eqnarray}
Thus, $\widetilde\beta=4(1-c)$ is a valid common smoothness
constant for this group.

For the Huberized ReLU,
$L_\phi=1/h$, so
$\widetilde\beta=(1-c)/h$.
For $\tanh$ and sigmoid, $L_\phi\leq1$, so
$\widetilde\beta=1-c$ is valid.
}

\textbf{Near-homogeneity.}
{
Let $\widetilde\phi(z)
    =
    a_\phi z+\lambda_\phi\phi(z)$ with
\[
    \widetilde\phi'(z)
    =
    a_\phi+\lambda_\phi\phi'(z),
\]
where $a_\phi=d$ for GELU, Softplus, and SiLU, and
$a_\phi=c$ for the remaining activations. Then, it holds
\begin{eqnarray}
    \widetilde\phi(z)-\widetilde\phi'(z)z
    &=&
    a_\phi z+\lambda_\phi\phi(z)
    -
    \left(
    a_\phi+\lambda_\phi\phi'(z)
    \right)z
    \nonumber\\
    &=&
    \lambda_\phi
    \left[
    \phi(z)-\phi'(z)z
    \right].
    \label{eq:corrected-near-homogeneity}
\end{eqnarray}
Hence, we have
\begin{eqnarray}
    |
    \widetilde\phi(z)-\widetilde\phi'(z)z
    |
    &\leq&
    \lambda_\phi\kappa_\phi,
    \label{eq:corrected-kappa-bound}
\end{eqnarray}
where $\kappa_\phi
    :=
    \sup_{z\in\mathbb R}
    |\phi(z)-z\phi'(z)|.$ For the first group, the base activations satisfy
\[
    \kappa_{\mathrm{GELU}}
    =
    \sqrt{\frac{2}{\pi}}e^{-1}
    <1,
    \qquad
    \kappa_{\mathrm{Softplus}}
    =
    \ln 2
    <1,
    \qquad
    \kappa_{\mathrm{SiLU}}
    \leq1.
\]
Therefore, it holds
\begin{eqnarray}
    \lambda_\phi\kappa_\phi
    \leq
    \frac{1-c}{6}
    \leq
    1-c,
    \nonumber
\end{eqnarray}
so $\kappa=1-c$ is valid.

For the Huberized ReLU,
$\kappa_\phi=h/2$, and hence
\[
    \lambda_\phi\kappa_\phi
    =
    \frac{(1-c)h}{2}.
\]
Finally, for $\tanh$ and sigmoid,
$\kappa_\phi\leq1$, and therefore
$\lambda_\phi\kappa_\phi\leq1-c$.
}

Thus, all three conditions of
Assumption~\ref{assump: activation} are satisfied with the
constants stated in the lemma.
\end{proof}

\subsection{Proofs for the EoS Regime}
\label{app:twolayer:eos}

\begin{lemma}\label{lem:nn-distance-comparator}
Let $U = U_1 + U_2$ with $U_1 \in \mathbb R^{md \times K}$ and
$U_2 = \frac{\eta(1 + \frac{1}{b})}{\tilde\gamma}\,\overline{W}_{*}$. Then, for all $t \ge 1$, it holds that
\begin{equation}
  \frac{\mathbb E\bigl[\|W_t - U\|_F^{2}\bigr]}{2\tilde\eta\, t}
  + \frac{1}{t}\sum_{k=0}^{t-1}\mathbb E[L(W_k)]
  \;\le\; {\widehat L_W(U_1)} + \frac{\|W_0 - U\|_F^{2}}{2\tilde\eta\, t}. \label{eq:nn-distance-comparator}
\end{equation}
\end{lemma}
\begin{proof}
From the update rule of \ref{SGD} with $\tilde\eta = m\eta$, we have that
\begin{align*}
  \|W_{t+1} - U\|_F^{2}
   &\;=\; \|W_t - U\|_F^{2} + 2\tilde\eta\,\langle g_t, U - W_t\rangle_F + \tilde\eta^{2}\|g_t\|_F^{2}\\
   &\;=\; \|W_t - U\|_F^{2} + 2\tilde\eta\,\langle g_t, U_1 - W_t\rangle_F
      + \tilde\eta^{2}\!\left(\frac{2}{\tilde\eta}\langle g_t, U_2\rangle_F + \|g_t\|_F^{2}\right).
\end{align*}
Taking expectation conditional on $\mathcal F_t$ and using the unbiasedness $\mathbb E[g_t \mid \mathcal F_t] = \nabla L(W_t)$, we have that
\begin{eqnarray}
  \mathbb E\!\left[\|W_{t+1} - U\|_F^{2}\,\big|\, \mathcal F_t\right]
   &=& \|W_t - U\|_F^{2} + 2\tilde\eta\,\langle \nabla L(W_t), U_1 - W_t\rangle_F\nonumber \\
      &&+ \tilde\eta^{2}\!\left(\frac{2}{\tilde\eta}\langle \nabla L(W_t), U_2\rangle_F 
     + \mathbb E\bigl[\|g_t\|_F^{2} \mid \mathcal F_t\bigr]\right). \label{eq:nn-recursion}
\end{eqnarray}
We, next, show that with the choice $U_2 = \frac{\eta(1 + \frac{1}{b})}{\tilde\gamma}\,\overline{W}_{*}$, the last term in~\eqref{eq:nn-recursion} is non-positive. We have that
\begin{eqnarray*}
  \frac{2}{\tilde\eta}\langle \nabla L(W_t), U_2\rangle_F
  &\;=\;& \frac{2(1 + \frac{1}{b})}{m\tilde\gamma}\,\langle \nabla L(W_t), \overline{W}_{*}\rangle_F\\
  &\stackrel{\text{Lemma~\ref{lem:nn-perceptron}}}{\le}\;&
   -\frac{2(1 + \frac{1}{b})}{m}\, G(W_t) .
\end{eqnarray*}
For the second moment $\mathbb E\bigl[\|g_t\|_F^{2} \mid \mathcal F_t\bigr]$, the bias-variance decomposition together with Lemmas~\ref{lem:nn-grad-bound} and~\ref{lemma:variance_two_layer} give
\begin{eqnarray*}
  \mathbb E\bigl[\|g_t\|_F^{2} \mid \mathcal F_t\bigr]
  &\;=\;& \|\nabla L(W_t)\|_F^{2} + \mathbb E\bigl[\|g_t - \nabla L(W_t)\|_F^{2} \mid \mathcal F_t\bigr]\\
  &\stackrel{\text{Lemma~\ref{lem:nn-grad-bound}}}{\le}\;&
   \frac{2}{m}\, G(W_t)^{2} + \frac{2}{m b}\, G(W_t)\\
  &\;\le\;& \frac{2}{m}\, G(W_t) + \frac{2}{m b}\, G(W_t)\\
   &\;=\;& \frac{2(1 + \frac{1}{b})}{m}\, G(W_t),
\end{eqnarray*}
where we used $G(W_t)^{2} \le G(W_t)$ since $G(W_t) \in [0, 1]$.
Summing the two bounds, we obtain
\begin{eqnarray}
    \frac{2}{\tilde\eta}\langle \nabla L(W_t), U_2\rangle_F + \mathbb E\bigl[\|g_t\|_F^{2} \mid \mathcal F_t\bigr]
   \;\le\; -\frac{2(1 + \frac{1}{b})}{m}\, G(W_t) + \frac{2(1 + \frac{1}{b})}{m}\, G(W_t)
   \;=\; 0 \label{eq:pre_7}.
\end{eqnarray}
Substituting \eqref{eq:pre_7} into~\eqref{eq:nn-recursion}, we obtain
\[
  \mathbb E\!\left[\|W_{t+1} - U\|_F^{2}\,\big|\,\mathcal F_t\right]
  \;\le\; \|W_t - U\|_F^{2} + 2\tilde\eta\,\langle \nabla L(W_t), U_1 - W_t\rangle_F.
\]
{Using Lemma~\ref{lemma: conv_proxy_NN}
, we have that $\langle \nabla L(W_t), U_1 - W_t\rangle_F \le \widehat L_W(U_1) - L(W_t)$,} and therefore we obtain
\begin{eqnarray}
    \mathbb E\!\left[\|W_{t+1} - U\|_F^{2}\,\big|\,\mathcal F_t\right]
  \;\le\; \|W_t - U\|_F^{2} + 2\tilde\eta\bigl[{\widehat L_W(U_1)} - L(W_t)\bigr]. \label{eq:before_sum_lemma_c10}
\end{eqnarray}
Taking expectation again and using the tower law of expectation, we have that
\begin{eqnarray}
    \mathbb E\!\left[\|W_{t+1} - U\|_F^{2}\,\right]
  \;\le\; \Expe{\|W_t - U\|_F^{2}} + 2\tilde\eta \Expe{{\widehat L_W(U_1)} - L(W_t)} \nonumber
\end{eqnarray}
Summing for $k = 0, \ldots, t-1$ and dividing by $2\tilde\eta\, t$, we obtain
\begin{eqnarray}
    \frac{\mathbb E\bigl[\|W_t - U\|_F^{2}\bigr]}{2\tilde\eta\, t}
  + \frac{1}{t}\sum_{k=0}^{t-1}\mathbb E[L(W_k)]
  \;\le\; {\widehat L_W(U_1)} + \frac{\|W_0 - U\|_F^{2}}{2\tilde\eta\, t}\nonumber
\end{eqnarray}
\end{proof}

\begin{theorem}[Restatement of Theorem~\ref{thm: sgd_2_layer_nn}]\label{thm:nn-eos-rate}
Let Assumptions~\ref{assumpt: separable}, \ref{assump: activation}, hold. The iterates of \ref{SGD} with any step size $\tilde\eta = m\eta > 0${, batch size $b \ge 1$ and any initialization $W_0 \in \R^{md\times K}$} applied to the two-layer network~\eqref{eq: 2nn} satisfy for any $T \ge 1$ that
{\[
  \min_{0 \le k \le T-1}\mathbb E[L(W_k)]
  \;\le\;
  \frac{K - 1 + 4\ln^{2}(\tilde\gamma^{2}\eta T) + 16\kappa^{2}
                + 2\eta^{2}(1 + \frac{1}{b})^{2} + \frac{\tilde\gamma^{2}}{m}\|W_0\|_F^2}{\tilde\gamma^{2}\eta\, T}.
\]
Moreover, if $W_0 = \mathbf{0}$, it holds that}
\[
  \min_{0 \le k \le T-1}\mathbb E[L(W_k)]
  \;\le\;
  \frac{K - 1 + 2\ln^{2}(\tilde\gamma^{2}\eta T) + 8\kappa^{2}
                + \eta^{2}(1 + \frac{1}{b})^{2}}{\tilde\gamma^{2}\eta\, T}.
\]
\end{theorem}
\begin{proof}
From Lemma~\ref{lem:nn-distance-comparator}, we have that for all $T \ge 1$ it holds
\begin{eqnarray}
    \frac{1}{T}\sum_{k=0}^{T-1}\mathbb E[L(W_k)] &\le& {\widehat L_W(U_1)} + {\frac{\|W_0 - U\|_F^{2}}{2\tilde\eta\, T}} - \frac{\mathbb E\bigl[\|W_T - U\|_F^{2}\bigr]}{2\tilde\eta\, T} \nonumber \\
    &\leq& {\widehat L_W(U_1)} + {\frac{\|W_0 - U\|_F^{2}}{2\tilde\eta\, T}}, \label{eq:pre_80}
\end{eqnarray}
{where we used the non-negativity of $\|W_T - U\|_F^2$.} For $U_1 = \frac{\ln(\tilde\gamma^{2}\eta T) +2\kappa}{\tilde\gamma}\,\overline{W}_{*}, U_2 = \frac{\eta(1 + \frac{1}{b})}{\tilde\gamma}\,\overline{W}_{*}$ we have from Lemma~\ref{lemma: conv_proxy_NN} that
\[
  {\widehat L_W(U_1)} \;\le\; \frac{K - 1}{\tilde\gamma^{2}\eta T}.
\]
From inequality~\eqref{eq:nn-U-norm-sq}, we have that
\begin{eqnarray}
    \|U\|_F^{2} \;\le\; \frac{4 m\,\ln^{2}(\tilde\gamma^{2}\eta T) + 16 m\,\kappa^{2} + 2 m\eta^{2}(1 + \frac{1}{b})^{2}}{\tilde\gamma^{2}} \label{eq:pre_8}
\end{eqnarray}
{For a general initialization, we apply $\|W_0 - U\|_F^2 \leq 2\|W_0\|_F^2 + 2\|U\|_F^2$. Substituting \eqref{eq:pre_8} into \eqref{eq:pre_80} and using $\tilde\eta\, T = m\eta T$, we obtain
\begin{eqnarray}
    \frac{1}{T}\sum_{k=0}^{T-1}\mathbb E[L(W_k)]
  &\leq& \frac{K - 1}{\tilde\gamma^{2}\eta T} + \frac{\|W_0\|_F^2 + \|U\|_F^2}{m\eta T} \nonumber \\
  &\leq& \frac{K - 1 + 4\ln^{2}(\tilde\gamma^{2}\eta T) + 16\kappa^{2}
                + 2\eta^{2}(1 + \frac{1}{b})^{2} + \frac{\tilde\gamma^{2}}{m}\|W_0\|_F^2}{\tilde\gamma^{2}\eta\, T}.\nonumber 
\end{eqnarray}
If $W_0 = \mathbf{0}$, then $\|W_0 - U\|_F^2 = \|U\|_F^2$, and substituting \eqref{eq:pre_8} into \eqref{eq:pre_80} directly gives
\begin{eqnarray}
    \frac{1}{T}\sum_{k=0}^{T-1}\mathbb E[L(W_k)]
  &\leq& \frac{K - 1 + 2\ln^{2}(\tilde\gamma^{2}\eta T) + 8\kappa^{2}
                + \eta^{2}(1 + \frac{1}{b})^{2}}{\tilde\gamma^{2}\eta\, T}.\nonumber 
\end{eqnarray}
In both cases, the bound on $\min_{0 \le k \le T-1}\mathbb E[L(W_k)]$ follows since the minimum is at most the average.}
\end{proof}

\subsection{Proofs for the Stable Regime}\label{app:twolayer:stable}

Throughout this section, we let
\begin{eqnarray}
    \widetilde L_{\mathrm{NN}}
    &:=&
    \min\left\{
    \frac{1}{
    16\eta(1+\widetilde\beta)
    \left(
    1+\frac{2n}{
    b\min\{\widetilde\gamma^2,1\}}
    \right)},
    \frac{1}{\sqrt{2}n\eta},
    \frac{1}{2ne^{\kappa+2}}
    \right\},
    \nonumber\\
    \mathcal S_{\mathrm{NN}}
    &:=&
    \left\{
    W\in\mathbb R^{md\times K}
    :
    L(W)\leq\widetilde L_{\mathrm{NN}}
    \right\}.
    \nonumber
\end{eqnarray}

\begin{lemma}\label{lem:nn-descent}
If for some $t \ge 0$ the iterates of \ref{SGD} for the two-layer network~\eqref{eq: 2nn} satisfy
$L(W_t) \leq \tilde{L}_{NN}$,
then it holds that
\begin{equation}
  \Expep{L(W_{t+1}) - L(W_t)} \;\le\; 0, \label{eq:nn-descent}
\end{equation}
and thus $W_t \in \mathcal S_{NN}$.
\end{lemma}
\begin{proof}
Conditionally on $\mathcal F_t$, let
$I_{t,1},\ldots,I_{t,b}$ be independent random variables
sampled uniformly from $[n]$. Recall that the minibatch gradient
is
\begin{eqnarray}
    g_t
    :=
    \frac1b
    \sum_{s=1}^b
    \nabla\ell_{I_{t,s}}(W_t)
    \in\mathbb R^{md\times K},
    \nonumber
\end{eqnarray}
and decomposes block-wise as
$g_t=(g_t^{(1)},\ldots,g_t^{(m)})$, where
$g_t^{(r)}\in\mathbb R^{d\times K}$.

Fix $i\in[n]$ and define the logit increment
\begin{eqnarray}
    \Delta z_i
    &:=&
    z_i(W_{t+1})-z_i(W_t)
    \nonumber\\
    &=&
    \frac1m
    \sum_{r=1}^m
    a_r\,
    \phi\left(x_i^\top W_{t+1}^{(r)}\right)
    -
    \frac1m
    \sum_{r=1}^m
    a_r\,
    \phi\left(x_i^\top W_t^{(r)}\right).
    \nonumber
\end{eqnarray}
We argue component-wise. For every class $k\in[K]$, we have
\begin{eqnarray}
    |(\Delta z_i)_k|
    &\leq&
    \frac1m
    \sum_{r=1}^m
    |a_r|
    \left|
    \phi\left(x_i^\top W_{t+1,k}^{(r)}\right)
    -
    \phi\left(x_i^\top W_{t,k}^{(r)}\right)
    \right|,
    \label{eq:nn-stab-taylor-1}
\end{eqnarray}
where $W_{t,k}^{(r)}\in\mathbb R^d$ denotes the $k$-th column
of $W_t^{(r)}\in\mathbb R^{d\times K}$.

Since $|\phi'|\leq1$ by
Assumption~\ref{assump: activation}, the activation $\phi$ is
$1$-Lipschitz. Hence,
\begin{eqnarray}
    &&
    \left|
    \phi\left(x_i^\top W_{t+1,k}^{(r)}\right)
    -
    \phi\left(x_i^\top W_{t,k}^{(r)}\right)
    \right|
    \nonumber\\
    &\leq&
    \left|
    x_i^\top
    \left(
    W_{t+1,k}^{(r)}-W_{t,k}^{(r)}
    \right)
    \right|
    \nonumber\\
    &\leq&
    \|x_i\|_2
    \left\|
    W_{t+1,k}^{(r)}-W_{t,k}^{(r)}
    \right\|_2
    \nonumber\\
    &\leq&
    \left\|
    W_{t+1,k}^{(r)}-W_{t,k}^{(r)}
    \right\|_2
    \nonumber\\
    &\stackrel{\eqref{SGD}}{=}&
    \widetilde\eta
    \left\|
    (g_t^{(r)})_k
    \right\|_2,
    \label{eq:nn-stab-taylor-2}
\end{eqnarray}
where we used $\|x_i\|_2\leq1$ and the update
\[
    W_{t+1}^{(r)}
    =
    W_t^{(r)}-\widetilde\eta g_t^{(r)}.
\]
Substituting~\eqref{eq:nn-stab-taylor-2} into
\eqref{eq:nn-stab-taylor-1} and applying the Cauchy--Schwarz
inequality across the $m$ summands, we obtain
\begin{eqnarray}
    |(\Delta z_i)_k|
    &\leq&
    \frac{\widetilde\eta}{m}
    \sum_{r=1}^m
    |a_r|
    \left\|
    (g_t^{(r)})_k
    \right\|_2
    \nonumber\\
    &\leq&
    \frac{\widetilde\eta}{m}
    \left(
    \sum_{r=1}^m a_r^2
    \right)^{1/2}
    \left(
    \sum_{r=1}^m
    \left\|
    (g_t^{(r)})_k
    \right\|_2^2
    \right)^{1/2}
    \nonumber\\
    &=&
    \frac{\widetilde\eta}{\sqrt m}
    \left(
    \sum_{r=1}^m
    \left\|
    (g_t^{(r)})_k
    \right\|_2^2
    \right)^{1/2},
    \label{eq:nn-stab-taylor-3}
\end{eqnarray}
where we used
$\sum_{r=1}^m a_r^2=m$ since $|a_r|=1$. Moreover,
\begin{eqnarray}
    \sum_{r=1}^m
    \left\|
    (g_t^{(r)})_k
    \right\|_2^2
    \leq
    \sum_{r=1}^m
    \sum_{k'=1}^K
    \left\|
    (g_t^{(r)})_{k'}
    \right\|_2^2
    =
    \|g_t\|_F^2.
    \nonumber
\end{eqnarray}
Therefore,
\begin{eqnarray}
    |(\Delta z_i)_k|
    \leq
    \frac{\widetilde\eta}{\sqrt m}
    \|g_t\|_F.
    \label{eq:nn-stab-taylor-4}
\end{eqnarray}
Taking the maximum over $k\in[K]$, we obtain
\begin{eqnarray}
    \|\Delta z_i\|_\infty
    &\leq&
    \frac{\widetilde\eta}{\sqrt m}
    \|g_t\|_F
    \nonumber\\
    &\leq&
    \frac{\widetilde\eta}{b\sqrt m}
    \sum_{s=1}^b
    \left\|
    \nabla\ell_{I_{t,s}}(W_t)
    \right\|_F
    \nonumber\\
    &\stackrel{\eqref{eq:nn-grad-i}}{\leq}&
    \frac{\sqrt2\eta}{b}
    \sum_{s=1}^b
    \ell_{I_{t,s}}(W_t)
    \nonumber\\
    &\leq&
    \sqrt2\eta
    \max_{j\in[n]}
    \ell_j(W_t)
    \nonumber\\
    &\leq&
    \sqrt2n\eta L(W_t),
    \label{eq:nn-logit-incr}
\end{eqnarray}
where we used $\eta=\widetilde\eta/m$ and
\[
    \max_{j\in[n]}\ell_j(W_t)
    \leq
    \sum_{j=1}^n\ell_j(W_t)
    =
    nL(W_t).
\]

For every $\theta\in[0,1]$, define
\[
    W_t^\theta
    :=
    W_t-\theta\widetilde\eta g_t.
\]
Repeating the preceding argument with
$\theta\widetilde\eta$ in place of $\widetilde\eta$, we get
\begin{eqnarray}
    \|z_i(W_t^\theta)-z_i(W_t)\|_\infty
    &\leq&
    \theta\sqrt2n\eta L(W_t)
    \nonumber\\
    &\leq&
    \sqrt2n\eta L(W_t).
    \label{eq:nn-segment-logit-incr}
\end{eqnarray}
Since
\[
    L(W_t)
    \leq
    \widetilde L_{\mathrm{NN}}
    \leq
    \frac{1}{\sqrt2n\eta},
\]
it follows that
\begin{eqnarray}
    \|z_i(W_t^\theta)-z_i(W_t)\|_\infty
    \leq1
    \qquad
    \text{for every }\theta\in[0,1].
    \label{eq:nn-segment-logit-small}
\end{eqnarray}
Consequently, for every $k\neq y_i$, we have
\begin{eqnarray}
    z_i(W_t^\theta)_k
    -
    z_i(W_t^\theta)_{y_i}
    &\leq&
    z_i(W_t)_k-z_i(W_t)_{y_i}+2.
    \nonumber
\end{eqnarray}
Using the definition of the cross-entropy loss, we obtain
\begin{eqnarray}
    \ell_i(W_t^\theta)
    &=&
    \ln\left(
    1+
    \sum_{k\neq y_i}
    \exp\left(
    z_i(W_t^\theta)_k
    -
    z_i(W_t^\theta)_{y_i}
    \right)
    \right)
    \nonumber\\
    &\leq&
    \ln\left(
    1+
    e^2
    \sum_{k\neq y_i}
    \exp\left(
    z_i(W_t)_k-z_i(W_t)_{y_i}
    \right)
    \right)
    \nonumber\\
    &\leq&
    e^2\ell_i(W_t),
    \label{eq:nn-loss-ratio}
\end{eqnarray}
where the last inequality follows from
$\ln(1+cu)\leq c\ln(1+u)$ for $c\geq1$ and $u\geq0$.

Let
\[
    \psi_i(\theta)
    :=
    \ell_i(W_t-\theta\widetilde\eta g_t),
    \qquad
    \theta\in[0,1].
\]
By Taylor's theorem, there exists $\theta_i\in(0,1)$ such that
\begin{eqnarray}
    \ell_i(W_{t+1})
    &=&
    \ell_i(W_t)
    -
    \widetilde\eta
    \left\langle
    \nabla\ell_i(W_t),g_t
    \right\rangle
    \nonumber\\
    &&\quad+
    \frac{\widetilde\eta^2}{2}
    \left\langle
    g_t,
    \nabla^2\ell_i
    \left(
    W_t-\theta_i\widetilde\eta g_t
    \right)
    g_t
    \right\rangle.
    \label{eq:nn-taylor}
\end{eqnarray}
By Lemma~\ref{lem:nn-hessian}, applied at the shifted point
$W_t-\theta_i\widetilde\eta g_t$, we have
\begin{eqnarray}
    \left\|
    \nabla^2\ell_i
    \left(
    W_t-\theta_i\widetilde\eta g_t
    \right)
    \right\|_{\mathrm{op}}
    &\leq&
    \frac{2(1+\widetilde\beta)}{m}
    \ell_i
    \left(
    W_t-\theta_i\widetilde\eta g_t
    \right).
    \label{eq:nn-Hess-shifted}
\end{eqnarray}
Combining this inequality with~\eqref{eq:nn-loss-ratio} and using
$e^2<8$, we obtain
\begin{eqnarray}
    \left\|
    \nabla^2\ell_i
    \left(
    W_t-\theta_i\widetilde\eta g_t
    \right)
    \right\|_{\mathrm{op}}
    &\leq&
    \frac{2e^2(1+\widetilde\beta)}{m}
    \ell_i(W_t)
    \nonumber\\
    &\leq&
    \frac{16(1+\widetilde\beta)}{m}
    \ell_i(W_t).
    \label{eq:nn-Hess-final}
\end{eqnarray}
Substituting~\eqref{eq:nn-Hess-final} into
\eqref{eq:nn-taylor}, we get
\begin{eqnarray}
    \ell_i(W_{t+1})
    &\leq&
    \ell_i(W_t)
    -
    \widetilde\eta
    \left\langle
    \nabla\ell_i(W_t),g_t
    \right\rangle
    \nonumber\\
    &&\quad+
    \frac{
    8(1+\widetilde\beta)\widetilde\eta^2
    }{m}
    \ell_i(W_t)\|g_t\|_F^2.
    \label{eq:nn-sample-descent}
\end{eqnarray}
Summing~\eqref{eq:nn-sample-descent} over $i=1,\ldots,n$ and
dividing by $n$, we have
\begin{eqnarray}
    L(W_{t+1})-L(W_t)
    &\leq&
    -\widetilde\eta
    \left\langle
    \nabla L(W_t),g_t
    \right\rangle
    \nonumber\\
    &&\quad+
    \frac{
    8(1+\widetilde\beta)\widetilde\eta^2
    }{m}
    L(W_t)\|g_t\|_F^2.
    \label{eq:nn-empir-descent}
\end{eqnarray}

Since $I_{t,1},\ldots,I_{t,b}$ are conditionally independent
and uniformly distributed over $[n]$, we have
\begin{eqnarray}
    \mathbb E[g_t\mid\mathcal F_t]
    =
    \nabla L(W_t).
    \nonumber
\end{eqnarray}
Taking conditional expectation in
\eqref{eq:nn-empir-descent}, we obtain
\begin{eqnarray}
    \mathbb E\left[
    L(W_{t+1})-L(W_t)
    \mid\mathcal F_t
    \right]
    &\leq&
    -\widetilde\eta
    \|\nabla L(W_t)\|_F^2
    \nonumber\\
    &&\quad+
    \frac{
    8(1+\widetilde\beta)\widetilde\eta^2
    }{m}
    L(W_t)
    \mathbb E\left[
    \|g_t\|_F^2
    \mid\mathcal F_t
    \right].
    \label{eq:nn-cond-desc}
\end{eqnarray}

Using the conditional independence of the sampled indices, we
have the exact second-moment identity
\begin{eqnarray}
    \mathbb E\left[
    \|g_t\|_F^2
    \mid\mathcal F_t
    \right]
    &=&
    \left(
    1-\frac1b
    \right)
    \|\nabla L(W_t)\|_F^2
    +
    \frac1{nb}
    \sum_{i=1}^n
    \|\nabla\ell_i(W_t)\|_F^2
    \nonumber\\
    &\leq&
    \|\nabla L(W_t)\|_F^2
    +
    \frac1{nb}
    \sum_{i=1}^n
    \|\nabla\ell_i(W_t)\|_F^2
    \nonumber\\
    &\leq&
    \left(
    1+\frac{2n}{b\widetilde\gamma^2}
    \right)
    \|\nabla L(W_t)\|_F^2,
    \label{eq:nn-bv-2}
\end{eqnarray}
where the last inequality follows from
Lemma~\ref{lem:two-layer-squared-grad-comparison}.

Substituting~\eqref{eq:nn-bv-2} into
\eqref{eq:nn-cond-desc} and using
$\widetilde\eta^2/m=\widetilde\eta\eta$, we get
\begin{eqnarray}
    &&
    \mathbb E\left[
    L(W_{t+1})-L(W_t)
    \mid\mathcal F_t
    \right]
    \nonumber\\
    &\leq&
    -\widetilde\eta
    \|\nabla L(W_t)\|_F^2
    \nonumber\\
    &&\quad+
    8(1+\widetilde\beta)
    \widetilde\eta\eta
    \left(
    1+\frac{2n}{b\widetilde\gamma^2}
    \right)
    L(W_t)
    \|\nabla L(W_t)\|_F^2
    \nonumber\\
    &=&
    -\widetilde\eta
    \left[
    1-
    8\eta(1+\widetilde\beta)
    \left(
    1+\frac{2n}{b\widetilde\gamma^2}
    \right)
    L(W_t)
    \right]
    \|\nabla L(W_t)\|_F^2.
    \label{eq:nn-desc-final}
\end{eqnarray}

Finally, suppose that
$L(W_t)\leq\widetilde L_{\mathrm{NN}}$. Since
$\min\{\widetilde\gamma^2,1\}\leq\widetilde\gamma^2$, we have
\begin{eqnarray}
    &&
    8\eta(1+\widetilde\beta)
    \left(
    1+\frac{2n}{b\widetilde\gamma^2}
    \right)
    L(W_t)
    \nonumber\\
    &\leq&
    8\eta(1+\widetilde\beta)
    \left(
    1+\frac{2n}{
    b\min\{\widetilde\gamma^2,1\}}
    \right)
    \widetilde L_{\mathrm{NN}}
    \nonumber\\
    &\leq&
    {\frac{1}{2}
    \;\leq\;}
    1,
    \nonumber
\end{eqnarray}
{where we used the factor $16$ in the definition of $\widetilde L_{\mathrm{NN}}$.}
Therefore,~\eqref{eq:nn-desc-final} implies
\begin{eqnarray}
    \mathbb E\left[
    L(W_{t+1})-L(W_t)
    \mid\mathcal F_t
    \right]
    \leq0
    \qquad
    \text{on }
    \left\{
    L(W_t)\leq\widetilde L_{\mathrm{NN}}
    \right\}.
    \nonumber
\end{eqnarray}
If $L(W_t)\leq\widetilde L_{\mathrm{NN}}$ holds almost surely,
then taking expectation on both sides and applying the tower
property gives
\begin{eqnarray}
    \mathbb E\left[
    L(W_{t+1})-L(W_t)
    \right]
    \leq0.
    \nonumber
\end{eqnarray}
\end{proof}

\begin{lemma}\label{lem:nn-stable-risk}
If the iterates of \ref{SGD} satisfy $W_k \in \mathcal S_{NN}, \forall k \in [t_{1}, t-1]$ for a deterministic $t_1 > 0$, then it holds that
\begin{equation}
  \mathbb E[L(W_t)] \;\le\; \frac{8}{\tilde\gamma^{2}\, \eta\,(t - t_{1})}.
\end{equation}
\end{lemma}
{
\begin{proof}
We first establish a quadratic drift inequality for the loss
inside the stable set. We then take conditional expectations
with respect to $\mathcal F_{t_1}$ and solve the resulting
recursion.

Fix $k\in\{t_1,\ldots,t-1\}$. Since
$W_k\in\mathcal S_{NN}$, we have $L(W_k)\leq\widetilde L_{\mathrm{NN}}$. From~\eqref{eq:nn-desc-final}, it follows that
\begin{eqnarray}
    \mathbb E\left[
    L(W_{k+1})-L(W_k)
    \mid\mathcal F_k
    \right]
    &\leq&
    -\widetilde\eta
    \left[
    1-
    8\eta(1+\widetilde\beta)
    \left(
    1+\frac{2n}{b\widetilde\gamma^2}
    \right)
    L(W_k)
    \right]
    \|\nabla L(W_k)\|_F^2.
    \label{eq:nn-stable-last-drift-1}
\end{eqnarray}
By the definition of $\widetilde L_{\mathrm{NN}}$, we have
\begin{equation}
    \widetilde L_{\mathrm{NN}}
    \leq
    \frac{1}{
    16\eta(1+\widetilde\beta)
    \left(
    1+\frac{2n}{
    b\min\{\widetilde\gamma^2,1\}}
    \right)}
    \leq
    \frac{1}{
    16\eta(1+\widetilde\beta)
    \left(
    1+\frac{2n}{b\widetilde\gamma^2}
    \right)}.
    \label{eq:nn-stable-threshold-consequence}
\end{equation}
Thus, we have that
\begin{eqnarray}
    1-
    8\eta(1+\widetilde\beta)
    \left(
    1+\frac{2n}{b\widetilde\gamma^2}
    \right)
    L(W_k)
    &\geq&
    \frac12.
\end{eqnarray}
Substituting this inequality into
\eqref{eq:nn-stable-last-drift-1}, we obtain
\begin{equation}
    \mathbb E\left[
    L(W_{k+1})-L(W_k)
    \mid\mathcal F_k
    \right]
    \leq
    -\frac{\widetilde\eta}{2}
    \|\nabla L(W_k)\|_F^2.
    \label{eq:nn-stable-last-drift-2}
\end{equation}

We next lower bound the gradient norm in terms of the loss.
Again, by the definition of $\widetilde L_{\mathrm{NN}}$,
\[
    L(W_k)
    \leq
    \widetilde L_{\mathrm{NN}}
    \leq
    \frac{1}{2ne^{\kappa+2}}
    \leq
    \frac1n.
\]
Since all individual losses are nonnegative, it follows that
\[
    0\leq\ell_i(W_k)\leq nL(W_k)\leq1,
    \qquad i\in[n].
\]
Moreover, since
$\ell_i(W_k)=-\log p_i(W_k)_{y_i}$, we have
\[
    1-p_i(W_k)_{y_i}
    =
    1-e^{-\ell_i(W_k)}.
\]
Using $1-e^{-u}\geq u/2$ for every $u\in[0,1]$, we obtain
\begin{eqnarray}
    G(W_k)
    &=&
    \frac1n\sum_{i=1}^n
    \left(
    1-e^{-\ell_i(W_k)}
    \right)
    \nonumber\\
    &\geq&
    \frac1{2n}
    \sum_{i=1}^n\ell_i(W_k)
    \nonumber\\
    &=&
    \frac12L(W_k).
    \label{eq:nn-stable-G-loss-comparison}
\end{eqnarray}
Applying~\eqref{eq:nn-grad-lower}, we have
\begin{eqnarray}
    \|\nabla L(W_k)\|_F^2
    &\geq&
    \frac{\widetilde\gamma^2}{m}
    G(W_k)^2
    \nonumber\\
    &\stackrel{\eqref{eq:nn-stable-G-loss-comparison}}{\geq}&
    \frac{\widetilde\gamma^2}{4m}
    L(W_k)^2.
    \label{eq:nn-stable-gradient-loss-comparison}
\end{eqnarray}
Combining~\eqref{eq:nn-stable-last-drift-2} and
\eqref{eq:nn-stable-gradient-loss-comparison}, and using
$\widetilde\eta=m\eta$, we get
\begin{eqnarray}
    \mathbb E\left[
    L(W_{k+1})-L(W_k)
    \mid\mathcal F_k
    \right]
    &\leq&
    -\frac{\widetilde\eta}{2}
    \frac{\widetilde\gamma^2}{4m}
    L(W_k)^2
    \nonumber\\
    &=&
    -\frac{\eta\widetilde\gamma^2}{8}
    L(W_k)^2.
    \label{eq:nn-stable-quadratic-drift}
\end{eqnarray}

For every $k\geq t_1$, define $a_k
    :=
    \mathbb E\left[
    L(W_k)\mid\mathcal F_{t_1}
    \right].$
Taking conditional expectation of
\eqref{eq:nn-stable-quadratic-drift} with respect to
$\mathcal F_{t_1}$ and applying the tower property, we have
\begin{eqnarray}
    a_{k+1}
    &\leq&
    a_k
    -
    \frac{\eta\widetilde\gamma^2}{8}
    \mathbb E\left[
    L(W_k)^2\mid\mathcal F_{t_1}
    \right]
    \nonumber\\
    &\leq&
    a_k
    -
    \frac{\eta\widetilde\gamma^2}{8}
    a_k^2,
    \label{eq:nn-stable-ak-recursion}
\end{eqnarray}
where the second inequality follows from conditional Jensen's
inequality.

Let $c:=\frac{\eta\widetilde\gamma^2}{8}.$
Since the cross-entropy loss is positive for finite logits,
$a_k>0$. Moreover,~\eqref{eq:nn-stable-ak-recursion} implies
\[
    a_k-a_{k+1}\geq ca_k^2
    \qquad\text{and}\qquad
    a_{k+1}\leq a_k.
\]
Therefore, we have
\begin{eqnarray}
    \frac1{a_{k+1}}-\frac1{a_k}
    &=&
    \frac{a_k-a_{k+1}}{a_ka_{k+1}}
    \nonumber\\
    &\geq&
    \frac{a_k-a_{k+1}}{a_k^2}
    \nonumber\\
    &\geq&
    c.
    \label{eq:nn-stable-reciprocal}
\end{eqnarray}
Summing~\eqref{eq:nn-stable-reciprocal} over
$k=t_1,\ldots,t-1$, we obtain
\begin{equation}
    \frac1{a_t}
    \geq
    \frac1{a_{t_1}}
    +
    \frac{\eta\widetilde\gamma^2}{8}(t-t_1)
    \geq \frac{\eta\widetilde\gamma^2}{8}(t-t_1).
    \label{eq:nn-stable-reciprocal-sum}
\end{equation}
Since $L(W_{t_1})$ is $\mathcal F_{t_1}$-measurable, we have
$a_{t_1}=L(W_{t_1})$. Hence,
\begin{eqnarray}
    \mathbb E\left[
    L(W_t)\mid\mathcal F_{t_1}
    \right]
    =a_t
    &\leq&
    \frac{8}{
    \eta\widetilde\gamma^2(t-t_1)
    }.
    \label{eq:nn-stable-conditional-rate}
\end{eqnarray}
Finally, taking expectation on both sides of
\eqref{eq:nn-stable-conditional-rate} and applying the tower
property gives
\[
    \mathbb E[L(W_t)]
    \leq
    \frac{8}{
    \eta\widetilde\gamma^2(t-t_1)
    },
\]
which proves the result.
\end{proof}
}
\begin{lemma}\label{lem:nn-entry}
Let $W_0 = \textbf{0}$. There exists $t_{\mathrm{in}} \le t_{\max}(\delta)$ such that with probability at least $1 - \delta$, we have that
$L(W_{t_{in}}) \leq \tilde{L}_{NN}$ with 
\[
  t_{\max}(\delta)
:=
\frac{1}{\tilde\gamma^2}
\max\left\{
\frac{8\Big(K-1+4\kappa+2\eta(2+\frac1b)\Big)}
{\eta\delta \tilde{L}_{NN}},
\;
\frac{32}{\eta\delta \tilde{L}_{NN}}\log\!\Big(\frac{32}{\eta\delta \tilde{L}_{NN}}\Big)
\right\}.
\]
\end{lemma}

\begin{proof}
From Lemma~\ref{lem:nn-G-cesaro} applied for $t = t_{\max}(\delta)$, we have that
\[
  \frac{1}{t_{\max}}\sum_{k=0}^{t_{\max} - 1}\mathbb E[G(W_k)]
   \;\le\;\frac{2\bigl(K - 1 + 2\ln(\tilde\gamma^{2}\eta t_{\max}) + 4\kappa + 2\eta(1 + \frac{1}{b})\bigr)}{\eta\tilde\gamma^2 t_{\max}}.
\]
Using $\ln(\tilde\gamma^{2}\eta t_{\max}) \le \ln(\tilde\gamma^{2} t_{\max}) + \eta$ and absorbing constants, we get
\begin{equation}
  \frac{1}{t_{\max}}\sum_{k=0}^{t_{\max} - 1}\mathbb E[G(W_k)]
   \;\le\; \frac{2\bigl(K - 1 + 2\ln(\tilde\gamma^{2} t_{\max}) + 4\kappa + 2\eta(2 + \frac{1}{b})\bigr)}{\eta\tilde\gamma^2 t_{\max}}. \label{eq:nn-Gbar-1}
\end{equation}
For $\delta \in (0, 1)$, we select $t_{\max}$ so that the right-hand side of~\eqref{eq:nn-Gbar-1} is at most $\delta\,\tilde{L}_{NN}/2$, {where $\tilde{L}_{NN}$ is the stable-set threshold defined at the beginning of Appendix~\ref{app:twolayer:stable}.}
Specifically, we next verify the formula for the $t_{\max}$:
\begin{itemize}
    \item For $t\geq \frac{8\Big(K-1+4\kappa+2\eta(2+\frac1b)\Big)}{\eta\delta \tilde{L}_{NN}\tilde\gamma^2}$, it holds that $\frac{2\Big(K-1+4\kappa+2\eta(2+\frac1b)\Big)}{\eta\tilde\gamma^2 t}\leq
\frac{\delta \tilde{L}_{NN}}{4}$.
    \item For $t\geq\frac{32}{\eta\delta \tilde{L}_{NN}\tilde\gamma^2}\log\!\Big(\frac{32}{\eta\delta \tilde{L}_{NN}}\Big)$, it holds that $\tilde\gamma^2 t\geq\frac{32}{\eta\delta \tilde{L}_{NN}}\log\!\Big(\frac{3}{\eta\delta \tilde{L}_{NN}}\Big)$, which is a sufficient condition (see Lemma G.5 in \citet{cai2024large}) for $\frac{4\log(\tilde\gamma^2 t)}{\eta\tilde\gamma^2 t}\leq
\frac{\delta \tilde{L}_{NN}}{4}$.
\end{itemize}
Thus, it suffices to select \(t_{\max}(\delta)\) such that
\begin{eqnarray}
t_{\max}(\delta)
&\geq&
\frac{1}{\tilde\gamma^2}
\max\left\{
\frac{8\Big(K-1+4\kappa+2\eta(2+\frac1b)\Big)}
{\eta\delta \tilde{L}_{NN}},
\;
\frac{32}{\eta\delta \tilde{L}_{NN}}\log\!\Big(\frac{32}{\eta\delta \tilde{L}_{NN}}\Big)
\right\}.
\end{eqnarray}
Thus, there exists $t_{\mathrm{in}} \le t_{\max}(\delta)$ such that
\[
  \mathbb E[G(W_{t_{\mathrm{in}}})] \;\le\; \frac{\delta \,\tilde{L}_{NN}}{2} .
\]
Define the event $\mathcal{E}_{\mathrm{in}} := \{G(W_{t_{\mathrm{in}}}) \le \frac{\tilde{L}_{NN}}{2}\}$. Since $G(W_{t_{\mathrm{in}}}) \ge 0$, by Markov's inequality we have that
\[
  \mathbb P(\mathcal{E}_{\mathrm{in}}^{c}) \;=\; \mathbb P\bigl(G(W_{t_{\mathrm{in}}}) > \frac{\tilde{L}_{NN}}{2}\bigr)
   \;\le\; \frac{2 \mathbb E[G(W_{t_{\mathrm{in}}})]}{\tilde{L}_{NN}} \;\le\; \delta,
\]
and hence $\mathbb P(\mathcal{E}_{\mathrm{in}}) \ge 1 - \delta$. Conditioning on the event $\mathcal{E}_{\mathrm{in}}$ and using the fact $G(W_{t_{\mathrm{in}}}) \le \frac{1}{2n e^{\kappa + 2}} \le \frac{1}{2n}$, every term in $G(W_{t_{\mathrm{in}}}) = \frac{1}{n} \sum_i (1 - p_i(W_{t_{\mathrm{in}}})_{y_i})$ satisfies $1 - p_i(W_{t_{\mathrm{in}}})_{y_i} \le 1/(2 e^{\kappa+2}) \le 1/2$ and thus $p_i(W_{t_{\mathrm{in}}})_{y_i} \ge 1/2$. Hence, for every $j \neq y_i$, it holds
\[
  z_i(W_{t_{\mathrm{in}}})_{y_i} - z_i(W_{t_{\mathrm{in}}})_j \;=\; \ln\!\frac{p_i(W_{t_{\mathrm{in}}})_{y_i}}{p_i(W_{t_{\mathrm{in}}})_j} \;\ge\; 0,
\]
which gives
\begin{eqnarray}
    F(W_{t_{\mathrm{in}}}) &=& \frac{1}{n}\sum_{i=1}^{n}\sum_{j \neq y_i}\frac{p_i(W_{t_{\mathrm{in}}})_j}{p_i(W_{t_{\mathrm{in}}})_{y_i}}\nonumber\\
   &\leq& \frac{1}{n}\sum_{i=1}^{n} 2\,\bigl(1 - p_i(W_{t_{\mathrm{in}}})_{y_i}\bigr)\nonumber\\
   &=& 2\, G(W_{t_{\mathrm{in}}})\nonumber\\
   &\leq&  \tilde{L}_{NN},\nonumber
\end{eqnarray}
where we used $p_i(W)_{y_i} \ge 1/2$ and $\sum_{j\neq y_i} p_i(W)_j = 1 - p_i(W)_{y_i}$. Since $L(W) \le F(W), \forall W \in \R^{md\times K}$, on the event $\mathcal{E}_{\mathrm{in}}$, we have that
\[
  L(W_{t_{\mathrm{in}}}) \;\le\; F(W_{t_{\mathrm{in}}}) \;\le\; \tilde{L}_{NN}.
\]
Combining this with $\mathbb P(\mathcal{E}_{\mathrm{in}}) \ge 1 - \delta$, we conclude that there exists $t_{\mathrm{in}} \le t_{\max}(\delta)$ such that with probability at least $1 - \delta$ it holds that $L(W_{t_{\mathrm{in}}}) \le \tilde{L}_{NN}$, and the dynamics enter the stable regime $\mathcal S_{NN}$ at iteration $t_{\mathrm{in}}$.
\end{proof}

\subsection*{Proof of Theorem~\ref{thm: sgd_2_layer_nn_stable}}
\label{app: thm: sgd_2_layer_nn_stable}
\begin{proof}
    The theorem is proved by combining Lemma~\ref{lem:nn-stable-risk} and Lemma~\ref{lem:nn-entry}.
\end{proof}

\newpage
\subsection{Proofs for the Stochastic Stabilization Mechanism}\label{app:twolayer:selfstab}

\begin{lemma}\label{lem:nn-exit}
Let $\mathcal{E}_t = \bigl\{L(W_t) \le \tilde{L}_{NN} ,\, L(W_{t+1}) > \tilde{L}_{NN}\bigr\}$. If at iteration $t > 0$ the dynamics are in the stable regime, then the exit probability satisfies
\[
  \mathbb P(\mathcal E_t \mid \mathcal F_t)
  \le2\exp\!\left(
  -\,\frac{bD_t}{4\eta L(W_t)^{\frac{3}{2}}}  \right),
\]
{where $ D_t = \frac{3\,(R_t^{NN})^2}
  {8(1+\tilde\beta)\left(48\eta (1+\tilde\beta)\,L(W_t)^{1/2}
  + (1+n)R_t^{NN}\right)}$, $R_t^{NN}=\sqrt{L(W_t)+16(1+\tilde\beta)\Delta_t^{NN}}-\sqrt{L(W_t)}$ and $\Delta_t^{NN}= \tilde L_{NN} - L(W_t)\geq 0$.}
\end{lemma}
\begin{proof}
From inequality~\eqref{eq:nn-empir-descent} in the proof of Lemma~\ref{lem:nn-descent}, we have that
\[
  L(W_{t+1})
   \;\le\; L(W_t) - \tilde\eta\,\langle \nabla L(W_t), g_t\rangle
   + \frac{8(1 + \tilde\beta)\,\tilde\eta^{2}}{m}\,L(W_t)\,\|g_t\|_F^{2}.
\]
Let $\xi_t := g_t - \nabla L(W_t)$ denote the noise at iteration $t$. Substituting
$g_t = \nabla L(W_t) + \xi_t$ and expanding, we get
\begin{eqnarray}
   L(W_{t+1}) &\le& L(W_t) - \tilde\eta\,\bigl(1 - 8\eta\, c\, L(W_t)\bigr)\,\|\nabla L(W_t)\|_F^{2}\nonumber \\
   && - \tilde\eta\,\bigl(1 - 16\eta\, c\, L(W_t)\bigr)\,\langle \nabla L(W_t), \xi_t\rangle
   + 8\eta\,\tilde\eta\, c\, L(W_t)\,\|\xi_t\|_F^{2}, \label{eq:nn-exit-start}
\end{eqnarray}
where we have let $c := 1+\tilde\beta$ for brevity.

On the event $\mathcal{E}_t = \{L(W_t) \le \tilde L_{NN},\, L(W_{t+1}) > \tilde L_{NN}\}$, it holds that
\begin{eqnarray}
  \tilde L_{NN} - L(W_t) + \tilde\eta\bigl(1 - 8\eta c L(W_t)\bigr)\|\nabla L(W_t)\|_F^{2}
  &<& -\tilde\eta\bigl(1 - 16\eta c L(W_t)\bigr)\langle\nabla L(W_t), \xi_t\rangle \nonumber \\
  && + 8\eta\tilde\eta c L(W_t)\|\xi_t\|_F^{2}. \label{eq:nn-exit-event}
\end{eqnarray}
{Since we seek an upper bound on $\mathbb P(\mathcal E_t\mid\mathcal F_t)$, we need a condition on the noise $\xi_t$ that is \emph{necessary} for the event $\mathcal E_t$. To this end, we lower bound the left-hand side and upper bound the right-hand side of~\eqref{eq:nn-exit-event}.

In the stable regime, we have that $L(W_t)\le\tilde L_{NN}\le\frac{1}{16\eta c}$, and hence
\begin{eqnarray}
0\;\le\;1-16\eta cL(W_t)\;\le\;1
\qquad\text{and}\qquad
1-8\eta cL(W_t)\;\ge\;0. \label{eq:nn-coeff-bounds}
\end{eqnarray}
Thus, the left-hand side of~\eqref{eq:nn-exit-event} is lower bounded by
\begin{eqnarray}
\tilde L_{NN} - L(W_t) + \tilde\eta\bigl(1 - 8\eta c L(W_t)\bigr)\|\nabla L(W_t)\|_F^{2}
\;\ge\;
\tilde L_{NN} - L(W_t)
\;=\;\Delta_t^{NN}. \label{eq:nn-lhs-lower}
\end{eqnarray}
For the right-hand side of~\eqref{eq:nn-exit-event}, using the Cauchy--Schwarz inequality, \eqref{eq:nn-coeff-bounds}, Lemma~\ref{lem:nn-grad-bound}~\eqref{eq:nn-grad-full} and the fact that $\tilde\eta=m\eta$, we have that
\begin{eqnarray}
&& -\tilde\eta\bigl(1 - 16\eta c L(W_t)\bigr)\langle\nabla L(W_t), \xi_t\rangle
   + 8\eta\tilde\eta c L(W_t)\|\xi_t\|_F^{2} \nonumber\\
&\le&
\tilde\eta\,\|\nabla L(W_t)\|_F\,\|\xi_t\|_F
+ 8\eta\tilde\eta c L(W_t)\|\xi_t\|_F^{2} \nonumber\\
&\le&
\sqrt{2m}\,\eta L(W_t)\|\xi_t\|_F
+ 8m\eta^2 c L(W_t)\|\xi_t\|_F^{2}. \label{eq:nn-rhs-upper}
\end{eqnarray}
Combining~\eqref{eq:nn-exit-event}, \eqref{eq:nn-lhs-lower} and~\eqref{eq:nn-rhs-upper}, on the event $\mathcal E_t$ it necessarily holds that
\begin{eqnarray}
\Delta_t^{NN}
\;<\;
\sqrt{2m}\,\eta L(W_t)\|\xi_t\|_F
+ 8m\eta^2 c L(W_t)\|\xi_t\|_F^{2}. \label{eq:nn-necessary}
\end{eqnarray}
The right-hand side of~\eqref{eq:nn-necessary} is increasing in $\|\xi_t\|_F\ge0$. Hence, solving the quadratic inequality, \eqref{eq:nn-necessary} is equivalent to
\begin{eqnarray}
\|\xi_t\|_F
\;>\;
r_t
&:=&
\frac{\sqrt{2m\eta^2L(W_t)^2+32m\eta^2cL(W_t)\Delta_t^{NN}}-\sqrt{2m}\,\eta L(W_t)}{16m\eta^2cL(W_t)} \nonumber\\
&=&
\frac{R_t^{NN}}{8\sqrt{2m}\,\eta c\,L(W_t)^{1/2}}, \label{eq:nn-radius}
\end{eqnarray}
where $R_t^{NN}=\sqrt{L(W_t)+16c\Delta_t^{NN}}-\sqrt{L(W_t)}\ge0$. Therefore, we have that $\mathcal E_t\subseteq\{\|\xi_t\|_F>r_t\}$ and thus
\begin{eqnarray}
\mathbb P(\mathcal E_t \mid \mathcal F_t)
\;\le\;
\mathbb P\left(\|\xi_t\|_F > r_t\mid \mathcal F_t\right). \label{eq:nn-exit-noise}
\end{eqnarray}}

Conditionally on the filtration $\mathcal F_t$, the noise $\xi_t = \tfrac{1}{b}\sum_{j\in B_t}\zeta_{t,j}$
is the average of $b$ centered i.i.d.\ vectors
$\zeta_{t,j} := \nabla \ell_{j}(W_t) - \nabla L(W_t), \forall j \in B_t$. From
Lemma~\ref{lem:nn-grad-bound}, each term in the sum satisfies the bound
\begin{equation}
  \|\zeta_{t,j}\|_F \;\le\; \|\nabla \ell_{j}(W_t)\|_F + \|\nabla L(W_t)\|_F
   \;\le\; \sqrt{\tfrac{2}{m}}\,(1 + n)\, L(W_t), \label{eq:nn-zeta-bound}
\end{equation}
and the variance bound
\begin{equation}
  \mathbb E\bigl[\|\zeta_{t,j}\|_F^{2} \mid \mathcal F_t\bigr]
   \;\le\; \mathbb E\bigl[\|\nabla \ell_{i_j}(W_t)\|_F^{2} \mid \mathcal F_t\bigr]
   \;\le\; \frac{2}{m}\, L(W_t). \label{eq:nn-zeta-var}
\end{equation}
{By the vector Bernstein inequality applied to $\xi_t$ with radius $r_t$, we have
\begin{eqnarray}
  \mathbb P(\mathcal E_t \mid \mathcal F_t)
  &\le&
  2\exp\!\left(
  -\,\frac{b\,r_t^2}
  {\tfrac{4}{m}L(W_t)
  + \tfrac{2\sqrt{2}(1+n)}{3\sqrt m}L(W_t)\,r_t}
  \right) \nonumber \\
  &=&
  2\exp\!\left(
  -\,\frac{b\,(R_t^{NN})^2/\bigl(128m\eta^2c^2L(W_t)\bigr)}
  {\tfrac{4}{m}L(W_t)
  + \tfrac{(1+n)}{12m\eta c}L(W_t)^{1/2}R_t^{NN}}
  \right) \nonumber \\
  &=& 2\exp\!\left(
  -\,\frac{b}{4\eta L(W_t)^{\frac{3}{2}}} \cdot \frac{3\,(R_t^{NN})^2}
  {8c\left(48\eta c\,L(W_t)^{1/2}
  + (1+n)R_t^{NN}\right)}
  \right), \nonumber
\end{eqnarray}
which proves the result after substituting $c=1+\tilde\beta$. Finally, note that as $L(W_t)\to0$ we have $R_t^{NN}\to4\sqrt{(1+\tilde\beta)\tilde L_{NN}}$ and $D_t\to\frac{3\sqrt{\tilde L_{NN}}}{2(1+n)\sqrt{1+\tilde\beta}}>0$, so the exit probability is of order $\exp\bigl(-\Omega\bigl(b/(\eta L(W_t)^{3/2})\bigr)\bigr)$.}
\end{proof}

\begin{lemma}\label{lem:return-nn}
Assume that the dynamics exit the stable regime $\mathcal{S}_{NN}$ at $t_{\mathrm{out}} > 0$ and fix $\delta \in (0, 1)$. Then, with probability at least $1 - \delta$ the iterates of \eqref{SGD} for the two-layer network in \eqref{eq: 2nn} return inside the stable regime $\mathcal{S}_{NN}$ in at most
\begin{equation*}
t_{re} = \left\lceil\frac{4}{\tilde{\gamma}^2 \eta \delta \tilde{L}_{NN}} \max\!\left\{ A_{NN}, 16 \ln\!\left( \frac{64}{\tilde{\gamma}^2 \eta \delta \tilde{L}_{NN}} \right) \right\}\right\rceil
\end{equation*}
number of steps, where
\begin{equation*}
A_{NN} := {9(K - 1) + 16 \ln^2(\tilde{\gamma}^2 \eta t_{\mathrm{out}}) + 80 \kappa^2 + 10 \eta^2 \Big(1 + \tfrac{1}{b}\Big)^{\!2}}.
\end{equation*}
\end{lemma}
\begin{proof}
Let $U = U_1 + U_2$ with $U_1 = \alpha \mathbf{W}_*$, $U_2 = \frac{\eta(1+\tfrac{1}{b})}{\tilde{\gamma}} \mathbf{W}_*$ and $\alpha = \frac{\ln(\tilde{\gamma}^2 \eta (t - t_{\mathrm{out}})) + 2\kappa}{\tilde{\gamma}}$. From the update rule of \eqref{SGD} with $\tilde{\eta} = m\eta$, we have that
\begin{align*}
\|W_{t+1} - U\|_F^2 &= \|W_t - U\|_F^2 + 2\tilde{\eta} \langle g_t, U - W_t \rangle_F + \tilde{\eta}^2 \|g_t\|_F^2 \\
&= \|W_t - U\|_F^2 + 2\tilde{\eta} \langle g_t, U_1 - W_t \rangle_F + \tilde{\eta}^2 \!\left( \frac{2}{\tilde{\eta}} \langle g_t, U_2 \rangle_F + \|g_t\|_F^2 \right).
\end{align*}
Taking expectation conditional on the filtration $\mathcal{F}_t$ and using the unbiasedness property of the stochastic oracles, we get
\begin{equation}\label{eq:return-nn-1}
\begin{aligned}
\mathbb{E}\!\left[ \|W_{t+1} - U\|_F^2 \,\big|\, \mathcal{F}_t \right] &= \|W_t - U\|_F^2 + 2\tilde{\eta} \langle \nabla L(W_t), U_1 - W_t \rangle_F \\
&\quad + \tilde{\eta}^2 \!\left( \frac{2}{\tilde{\eta}} \langle \nabla L(W_t), U_2 \rangle_F + \mathbb{E}\!\left[ \|g_t\|_F^2 \,\big|\, \mathcal{F}_t \right] \right).
\end{aligned}
\end{equation}
We, next, show that with the choice $U_2 = \frac{\eta(1+\tfrac{1}{b})}{\tilde{\gamma}} \mathbf{W}_*$, the last term in \eqref{eq:return-nn-1} is non-positive. We have that
\begin{equation*}
\frac{2}{\tilde{\eta}} \langle \nabla L(W_t), U_2 \rangle_F = \frac{2(1 + \tfrac{1}{b})}{m \tilde{\gamma}} \langle \nabla L(W_t), \mathbf{W}_* \rangle_F \stackrel{\text{Lemma~\ref{lem:nn-perceptron}}}{\leq} -\frac{2(1 + \tfrac{1}{b})}{m} G(W_t).
\end{equation*}
For the second moment $\mathbb{E}\!\left[ \|g_t\|_F^2 \,\big|\, \mathcal{F}_t \right]$, the bias-variance decomposition together with Lemmas~\ref{lem:nn-grad-bound} and \ref{lemma:variance_two_layer} give
\begin{equation*}
\begin{aligned}
\mathbb{E}\!\left[ \|g_t\|_F^2 \,\big|\, \mathcal{F}_t \right] &= \|\nabla L(W_t)\|_F^2 + \mathbb{E}\!\left[ \|g_t - \nabla L(W_t)\|_F^2 \,\big|\, \mathcal{F}_t \right] \\
&\stackrel{\text{Lemma~\ref{lem:nn-grad-bound}}}{\leq} \frac{2}{m} G(W_t)^2 + \frac{2}{m b} G(W_t) \\
&\leq \frac{2}{m} G(W_t) + \frac{2}{m b} G(W_t) \\
&= \frac{2(1 + \tfrac{1}{b})}{m} G(W_t),
\end{aligned}
\end{equation*}
where we used $G(W_t)^2 \leq G(W_t)$ since $G(W_t) \in [0, 1]$. Summing the two bounds, we obtain
\begin{equation}\label{eq:return-nn-2}
\frac{2}{\tilde{\eta}} \langle \nabla L(W_t), U_2 \rangle_F + \mathbb{E}\!\left[ \|g_t\|_F^2 \,\big|\, \mathcal{F}_t \right] \leq -\frac{2(1 + \tfrac{1}{b})}{m} G(W_t) + \frac{2(1 + \tfrac{1}{b})}{m} G(W_t) = 0.
\end{equation}
Substituting \eqref{eq:return-nn-2} into \eqref{eq:return-nn-1}, we obtain
\begin{equation*}
\mathbb{E}\!\left[ \|W_{t+1} - U\|_F^2 \,\big|\, \mathcal{F}_t \right] \leq \|W_t - U\|_F^2 + 2\tilde{\eta} \langle \nabla L(W_t), U_1 - W_t \rangle_F.
\end{equation*}
{Using Lemma~\ref{lemma: conv_proxy_NN}}, we have that $\langle \nabla L(W_t), U_1 - W_t \rangle_F \leq {\widehat L_W(U_1)} - L(W_t)$, and therefore we obtain
\begin{equation*}
\mathbb{E}\!\left[ \|W_{t+1} - U\|_F^2 \,\big|\, \mathcal{F}_t \right] \leq \|W_t - U\|_F^2 + 2\tilde{\eta} \big[ {\widehat L_W(U_1)} - L(W_t) \big].
\end{equation*}
Taking expectation again and using the tower law of expectation, we have that
\begin{equation*}
\mathbb{E}\!\left[ \|W_{t+1} - U\|_F^2 \right] \leq \mathbb{E}\!\left[ \|W_t - U\|_F^2 \right] + 2\tilde{\eta}\, \mathbb{E}\!\left[ {\widehat L_W(U_1)} - L(W_t) \right].
\end{equation*}
Summing for $k = t_{\mathrm{out}}, \ldots, t - 1$ and dividing by $2\tilde{\eta}(t - t_{\mathrm{out}})$, we obtain
\begin{equation*}
\frac{\mathbb{E}\!\left[ \|W_t - U\|_F^2 \right]}{2\tilde{\eta} (t - t_{\mathrm{out}})} + \frac{1}{t - t_{\mathrm{out}}} \sum_{k = t_{\mathrm{out}}}^{t-1} \mathbb{E}\!\left[ L(W_k) \right] \leq {\widehat L_W(U_1)} + \frac{\mathbb{E}\!\left[\|W_{t_{\mathrm{out}}} - U\|_F^2\right]}{2\tilde{\eta}(t - t_{\mathrm{out}})}.
\end{equation*}
Using the non-negativity of $\|W_t - U\|_F^2$, we get
\begin{equation}\label{eq:return-nn-avg}
\frac{1}{t - t_{\mathrm{out}}} \sum_{k = t_{\mathrm{out}}}^{t-1} \mathbb{E}\!\left[ L(W_k) \right] \leq {\widehat L_W(U_1)} + \frac{\mathbb{E}\!\left[\|W_{t_{\mathrm{out}}} - U\|_F^2\right]}{2\tilde{\eta}(t - t_{\mathrm{out}})}.
\end{equation}
For $U_1 = \alpha \mathbf{W}_*$ with $\alpha = \frac{\ln(\tilde{\gamma}^2 \eta (t - t_{\mathrm{out}})) + 2\kappa}{\tilde{\gamma}}$, from Lemma~\ref{lemma: conv_proxy_NN} we have that
\begin{equation}\label{eq:return-nn-LU1}
{\widehat L_W(U_1)} \leq \frac{K - 1}{\tilde{\gamma}^2 \eta (t - t_{\mathrm{out}})}.
\end{equation}
We, next, bound the term $\mathbb{E}\!\left[\|W_{t_{\mathrm{out}}} - U\|_F^2\right]$. Applying $(a + b)^2 \leq 2a^2 + 2b^2$, we obtain
\begin{equation}\label{eq:return-nn-split}
\mathbb{E}\!\left[\|W_{t_{\mathrm{out}}} - U\|_F^2\right] \leq 2\, \mathbb{E}\!\left[\|W_{t_{\mathrm{out}}}\|_F^2\right] + 2 \|U\|_F^2.
\end{equation}
For the term $\|U\|_F^2$, similarly to Theorem~\ref{thm:nn-eos-rate} we have that
\begin{equation}\label{eq:return-nn-U}
\|U\|_F^2 \leq 2 \|U_1\|_F^2 + 2 \|U_2\|_F^2 \leq \frac{4 m \ln^2(\tilde{\gamma}^2 \eta(t - t_{\mathrm{out}})) + 16 m \kappa^2 + 2 m \eta^2 \big(1 + \tfrac{1}{b}\big)^2}{\tilde{\gamma}^2},
\end{equation}
where we have used $\|\mathbf{W}_*\|_F^2 = m$ and $(a + b)^2 \leq 2a^2 + 2b^2$. For the term $\mathbb{E}\!\left[\|W_{t_{\mathrm{out}}}\|_F^2\right]$, we apply {Lemma~\ref{lem:nn-distance-comparator}} from the original time origin with comparator $U^{(0)} = U_1^{(0)} + U_2$ where $U_1^{(0)} = \alpha_0 \mathbf{W}_*$ and $\alpha_0 = \frac{\ln(\tilde{\gamma}^2 \eta t_{\mathrm{out}}) + 2\kappa}{\tilde{\gamma}}$, which gives
\begin{eqnarray}
    \mathbb{E}\!\left[\|W_{t_{\mathrm{out}}} - U^{(0)}\|_F^2\right] &\leq& 2 \tilde{\eta}\, t_{\mathrm{out}}\, {\widehat L_W(U_1^{(0)})} + \|U^{(0)}\|_F^2 \nonumber \\
&\leq& \frac{2 m (K - 1) + 4 m \ln^2(\tilde{\gamma}^2 \eta t_{\mathrm{out}}) + 16 m \kappa^2 + 2 m \eta^2 \big(1 + \tfrac{1}{b}\big)^2}{\tilde{\gamma}^2},\quad \label{eq:return-nn-W-tout-diff}
\end{eqnarray}
where we have used $W_0 = 0$, $\tilde{\eta} = m\eta$, and the bound {of Lemma~\ref{lemma: conv_proxy_NN}} applied at time $t_{\mathrm{out}}$. Applying $(a + b)^2 \leq 2a^2 + 2b^2$ again, we obtain
\begin{eqnarray}
    \mathbb{E}\!\left[\|W_{t_{\mathrm{out}}}\|_F^2\right] &\leq&2\, \mathbb{E}\!\left[\|W_{t_{\mathrm{out}}} - U^{(0)}\|_F^2\right] + 2 \|U^{(0)}\|_F^2 \nonumber \\
    &\leq& \frac{8 m (K - 1) + 16 m \ln^2(\tilde{\gamma}^2 \eta t_{\mathrm{out}}) + 64 m \kappa^2 + 8 m \eta^2 \big(1 + \tfrac{1}{b}\big)^2}{\tilde{\gamma}^2}.\label{eq:return-nn-W-tout}
\end{eqnarray}
Substituting \eqref{eq:return-nn-LU1}, \eqref{eq:return-nn-U}, and \eqref{eq:return-nn-W-tout} into \eqref{eq:return-nn-avg} and using $\tilde{\eta} = m\eta$, we obtain
\begin{equation}\label{eq:return-nn-master}
\frac{1}{t - t_{\mathrm{out}}} \sum_{k = t_{\mathrm{out}}}^{t-1} \mathbb{E}\!\left[ L(W_k) \right] \leq \frac{C_{NN}(t, t_{\mathrm{out}})}{\tilde{\gamma}^2 \eta (t - t_{\mathrm{out}})},
\end{equation}
where
\begin{equation*}
C_{NN}(t, t_{\mathrm{out}}) := 9 (K - 1) + 16 \ln^2(\tilde{\gamma}^2 \eta t_{\mathrm{out}}) + 80 \kappa^2 + 10 \eta^2 \big(1 + \tfrac{1}{b}\big)^2 + 4 \ln^2(\tilde{\gamma}^2 \eta(t - t_{\mathrm{out}})).
\end{equation*}
We now split the bound \eqref{eq:return-nn-master} into two terms. Let $A_{NN} := 9 (K - 1) + 16 \ln^2(\tilde{\gamma}^2 \eta t_{\mathrm{out}}) + 80 \kappa^2 + 10 \eta^2 \Big(1 + \tfrac{1}{b}\Big)^{\!2}.$
Then, \eqref{eq:return-nn-master} becomes
\begin{equation}\label{eq:return-nn-master-split}
\frac{1}{t - t_{\mathrm{out}}} \sum_{k = t_{\mathrm{out}}}^{t-1} \mathbb{E}\!\left[ L(W_k) \right] \leq \frac{A_{NN}}{\tilde{\gamma}^2 \eta (t - t_{\mathrm{out}})} + \frac{4\ln^2(\tilde{\gamma}^2 \eta (t - t_{\mathrm{out}}))}{\tilde{\gamma}^2 \eta (t - t_{\mathrm{out}})}.
\end{equation}
We select $t_{\mathrm{re}}$ so that each of the two terms on the right-hand side of \eqref{eq:return-nn-master-split} is at most $\delta \tilde{L}_{NN}/4$. Specifically, we next verify the formula for $t_{\mathrm{re}}$:
\begin{itemize}
    \item For $t - t_{\mathrm{out}} \geq \frac{4 A_{NN}}{\tilde{\gamma}^2 \eta \delta \tilde{L}_{NN}}$, it holds that
    \begin{equation*}
    \frac{A_{NN}}{\tilde{\gamma}^2 \eta (t - t_{\mathrm{out}})} \leq \frac{\delta \tilde{L}_{NN}}{4}.
    \end{equation*}
    
    \item For $t - t_{\mathrm{out}} \geq \frac{64}{\tilde{\gamma}^2 \eta \delta \tilde{L}_{NN}} \ln\!\left( \frac{64}{\tilde{\gamma}^2 \eta \delta \tilde{L}_{NN}} \right)$, it holds that $\tilde{\gamma}^2 \eta (t - t_{\mathrm{out}}) \geq \frac{64}{\delta \tilde{L}_{NN}} \ln\!\left( \frac{64}{\tilde{\gamma}^2 \eta \delta \tilde{L}_{NN}} \right)$, which is a sufficient condition (see Lemma G.5 in \cite{cai2024large}) for
    \begin{equation*}
    \frac{4 \ln^2(\tilde{\gamma}^2 \eta (t - t_{\mathrm{out}}))}{\tilde{\gamma}^2 \eta (t - t_{\mathrm{out}})} \leq \frac{\delta \tilde{L}_{NN}}{4}.
    \end{equation*}
\end{itemize}
Thus, it suffices to select $t_{\mathrm{re}}$ such that
\begin{equation}\label{eq:return-nn-tre}
t_{\mathrm{re}} = \left\lceil\frac{4}{\tilde{\gamma}^2 \eta \delta \tilde{L}_{NN}} \max\!\left\{ A_{NN}, 16 \ln\!\left( \frac{64}{\tilde{\gamma}^2 \eta \delta \tilde{L}_{NN}} \right) \right\}\right\rceil.
\end{equation}
For this choice of $t_{\mathrm{re}}$ and selecting $t = t_{\mathrm{out}} + t_{\mathrm{re}}$, we have that
\begin{equation}\label{eq:return-nn-final}
\frac{1}{t_{\mathrm{re}}} \sum_{k = t_{\mathrm{out}}}^{t_{\mathrm{out}} + t_{\mathrm{re}} - 1} \mathbb{E}\!\left[ L(W_k) \right] \leq \frac{\delta \tilde{L}_{NN}}{2}.
\end{equation}
Hence, there exists $t_{\mathrm{in}}^{[2]} \in [t_{\mathrm{out}}, t_{\mathrm{out}} + t_{\mathrm{re}}]$ such that
\begin{equation*}
\mathbb{E}\!\left[ L(W_{t_{\mathrm{in}}^{[2]}}) \right] \leq \frac{\delta \tilde{L}_{NN}}{2}.
\end{equation*}
From Markov's inequality, we have that
\begin{equation*}
\Pr\!\left( L(W_{t_{\mathrm{in}}^{[2]}}) \geq \tilde{L}_{NN} \right) \leq \frac{\mathbb{E}\!\left[ L(W_{t_{\mathrm{in}}^{[2]}}) \right]}{\tilde{L}_{NN}} \leq \frac{\delta}{2} \leq \delta,
\end{equation*}
which implies that
\begin{equation*}
\Pr\!\left( L(W_{t_{\mathrm{in}}^{[2]}}) \leq \tilde{L}_{NN} \right) \geq 1 - \delta.
\end{equation*}
Thus, there exists $t_{\mathrm{in}}^{[2]} \leq t_{\mathrm{out}} + t_{\mathrm{re}}$ such that for $\delta \in (0, 1)$ with probability at least $1 - \delta$ it holds that $L(W_{t_{\mathrm{in}}^{[2]}}) \leq \tilde{L}_{NN}$ and the dynamics return inside the stable set $\mathcal{S}_{NN}$.
\end{proof}
\subsection*{Proof of Theorem~\ref{thm: self_stab_2_layer_nn}}
\label{app: thm: self_stab_2_layer_nn}
\begin{proof}
    The theorem is proved by combining Lemma~\ref{lem:nn-exit} and Lemma~\ref{lem:return-nn}.
\end{proof}

\end{document}